\documentclass{article}

\usepackage[utf8]{inputenc}
\usepackage[english]{babel}

\usepackage[preprint,nonatbib]{neurips_2020}

\usepackage[utf8]{inputenc} 
\usepackage[T1]{fontenc}    
\usepackage{amsmath}
\usepackage{amssymb}
\usepackage{amsthm}
\usepackage{amsfonts}       
\usepackage{url}            
\usepackage{blindtext}
\usepackage{booktabs}       
\usepackage{enumitem}
\usepackage{graphicx}
\graphicspath{{images/}{../images/}}
\usepackage{hyperref}       
\hypersetup{
    colorlinks,
    linkcolor={red!80!black},
    citecolor={blue!80!black},
    urlcolor={blue!80!black}    
}
\usepackage{nicefrac}       
\usepackage{mathtools}
\usepackage{marginnote}
\usepackage{microtype}      
\usepackage{xcolor}         
\usepackage{subcaption}
\usepackage{tikz}
\usepackage{todonotes}

\usepackage{graphicx} 
\usepackage{subcaption} 

\newtheorem{theorem}{Theorem}
\newtheorem{proposition}{Proposition}
\newtheorem{lemma}{Lemma}
\newtheorem{corollary}{Corollary}
\newtheorem{assumption}{Assumption}
\newtheorem*{assumption*}{Assumption}

\newtheorem{definition}{Definition}

\newcommand{\R}{\mathbb{R}}

\newcommand{\abs}[1]{\left|#1\right|}
\newcommand{\norm}[1]{\left\|#1\right\|}

\newcommand{\commentout}[1]{}

\newcommand{\realmatrix}

\usepackage{dsfont} 

\renewcommand{\d}[1]{\ensuremath{\operatorname{d}\!{#1}}}
\newcommand{\e}[1]{ {\mathrm{e}}^{ #1 } }

\newcommand{\expectation}[1]{ \mathbb{E} [ #1 ] }
\newcommand{\expectationWrt}[2]{ \mathbb{E}_{#2} [ #1 ] }
\newcommand{\expectationBig}[1]{ \mathbb{E} \Bigl[ #1 \Bigr] }

\newcommand{\pnorm}[2]{ \| #1 \|{}_{#2} }

\newcommand{\trace}[1]{ \mathrm{Tr}[#1] }

\newcommand{\naturalNumbersPlus}{ \mathbb{N}_{+} }

\newcommand{\realNumbers}{ \mathbb{R} }

\newcommand{\iterand}[2]{ #1^{[#2]} }
\newcommand{\iterandd}[2]{ #1^{\{#2\}} }

\newcommand{\refFigure}[1]{{\textrm{Figure~\ref{#1}}}}

\newcommand{\refDefinition}[1]{{\textrm{Definition~\ref{#1}}}}
\newcommand{\refTheorem}[1]{{\textrm{Theorem~\ref{#1}}}}
\newcommand{\refCorollary}[1]{{\textrm{Corollary~\ref{#1}}}}
\newcommand{\refProposition}[1]{{\textrm{Proposition~\ref{#1}}}}
\newcommand{\refLemma}[1]{{\textrm{Lemma~\ref{#1}}}}
\newcommand{\refAssumption}[1]{{\textrm{Assumption~\ref{#1}}}}

\newcommand{\refSection}[1]{{\textrm{Section~\ref{#1}}}}
\newcommand{\refAppendixSection}[1]{\textrm{Appendix~\ref{#1}}}

\def\eqcom#1{\overset{\textnormal{(#1)}}}

\newcommand{\itr}[2]{ \iterand{#1}{#2} }
\newcommand{\itrd}[2]{ \iterandd{#1}{#2} }

\def\({{\Bigl(}}
\def\){{\Bigr)}}

\newcommand{\ba}{\begin{array}}
\newcommand{\ea}{\end{array}}

\newcommand{\xdeleted}[1]{\deleted{}} 

\definecolor{light}{rgb}{0.5, 0.5, 0.5}

\usepackage[acronym]{glossaries}
\makeglossaries

\newacronym{GF}{GF}{Gradient Flow}
\newacronym{GD}{GD}{Gradient Descent}
\newacronym{NN}{NN}{Neural Network}
\newacronym{ODE}{ODE}{Ordinary Differential Equation}
\newacronym{PL}{PL}{Polyak--\L{}ojasiewicz}
\newacronym{SGD}{SGD}{Stochastic Gradient Descent}
\newacronym{SVD}{SVD}{Singular Value Decomposition}

\title{Asymptotic convergence rate of Dropout\\ on shallow linear neural networks}

\author{
  Albert Senen--Cerda\\
  \small Department of Mathematics \& Computer Science\\
  \small Eindhoven University of Technology\\
  \small \texttt{a.senen.cerda@tue.nl}
  \And
  Jaron Sanders\\
  \small Department of Mathematics \& Computer Science\\
  \small Eindhoven University of Technology\\
  \small \texttt{jaron.sanders@tue.nl}
}

\date{\today}

\begin{document}

\maketitle

\begin{abstract}
    We analyze the convergence rate of gradient flows on objective functions induced by \emph{Dropout} and \emph{Dropconnect}, when applying them to shallow linear \glspl{NN}---which can also be viewed as doing matrix factorization using a particular regularizer. Dropout algorithms such as these are thus regularization techniques that use $\{0,1\}$-valued random variables to filter weights during training in order to avoid coadaptation of features. By leveraging a recent result on nonconvex optimization and conducting a careful analysis of the set of minimizers as well as the Hessian of the loss function, we are able to obtain (i) a local convergence proof of the gradient flow and (ii) a bound on the convergence rate that depends on the data, the dropout probability, and the width of the \gls{NN}. Finally, we compare this theoretical bound to numerical simulations, which are in qualitative agreement with the convergence bound and match it when starting sufficiently close to a minimizer.
\end{abstract}
\section{Introduction}
\label{sec:Introduction}

Dropout algorithms are regularization techniques for \glspl{NN} that use $\{ 0, 1 \}$-valued random variables to filter out weights during training in order to avoid coadaptation of features. The first dropout algorithm was proposed by Hinton et al.\ in \cite{hinton2012improving}, and several variants of the algorithm appeared thereafter. These include versions in which edges are dropped \cite{wan2013regularization}, groups of edges are dropped from the input layer \cite{devries2017improved}, the removal probabilities change adaptively \cite{ba2013adaptive,li2016improved}; ones that are suitable for recurrent \glspl{NN} \cite{zaremba2014recurrent,semeniuta2016recurrent}; and variational ones with Gaussian filters \cite{kingma2015variational,molchanov2017variational}. Dropout algorithms have found application in e.g.\ image classification \cite{krizhevsky2012imagenet}, handwriting recognition \cite{pham2014dropout}, heart sound classification \cite{kay2016Dropconnected}, and drug discovery in cancer research \cite{urban2018deep}.  

This paper is about the convergence rate of two dropout algorithms: the original \emph{Dropout} \cite{hinton2012improving}, and the variant \emph{Dropconnect} \cite{wan2013regularization}. These dropout algorithms roughly work as follows. During the training procedure of a \gls{NN}, we iteratively present either algorithm with new (possibly random) input and output samples. Either algorithm then only updates a random set of weights of the \gls{NN}, leaving all other weights unchanged. For $p \in (0,1]$, entire nodes of the \gls{NN} are dropped with probability $1 - p$ in an independent and identically distributed manner in the case of \emph{Dropout}; and edges are dropped with probability $1 - p$ in an independent and identically distributed manner in the case of \emph{Dropconnect}. For each algorithm we display a sample, random subgraph of the base graph of a \gls{NN} in \refFigure{fig:Random_two_layer_NNs_with_Dropout_and_Dropconnect}. 

\begin{figure}[htbp]
    \centering
    \begin{subfigure}{0.3125\textwidth}
        \centering
        \def\layersep{3.25em}
        \begin{tikzpicture}[shorten >=1pt,->,draw=black!50, node distance=\layersep]
            \tikzstyle{every pin edge}=[<-,shorten <=1pt]
            \tikzstyle{neuron}=[circle,draw=black,fill=white!100,minimum size=10pt,inner sep=0pt]
            \tikzstyle{input neuron}=[neuron];
            \tikzstyle{output neuron}=[neuron];
            \tikzstyle{hidden neuron}=[neuron];
            \tikzstyle{annot} = [text width=4em, text centered]

            \foreach \name / \y in {1,...,4}
            \node[input neuron, pin=left:$x_{\y}$] (I-\name) at (0,-\y) {};

            \foreach \name / \y in {1,...,5}
            \path[yshift=+0.5cm]
            node[hidden neuron] (Ha-\name) at (\layersep,-\y cm) {};

            \foreach \name / \y in {1,...,3}
            \path[yshift=-0.5cm]
            node[output neuron,pin={[pin edge={->}]right:$y_{\y}$}] (O-\name) at (2*\layersep, -\y cm) {};

            \pgfmathsetseed{1987116}

            \foreach \source in {1,...,4}
            \foreach \dest in {1,...,5}
                {
                    \pgfmathparse{int(1000*rand+1000)} 
                    \ifnum \pgfmathresult>0
                        \path (I-\source) edge[black, opacity=1] (Ha-\dest);
                    \else
                        \path (I-\source) edge[black, opacity=0] (Ha-\dest);
                    \fi
                }

            \foreach \source in {1,...,5}
            \foreach \dest in {1,...,3}
                {
                    \pgfmathparse{int(1000*rand+1000)} 
                    \ifnum \pgfmathresult>0
                        \path (Ha-\source) edge[black, opacity=1] (O-\dest);
                    \else
                        \path (Ha-\source) edge[black, opacity=0] (O-\dest);
                    \fi
                }

        \end{tikzpicture}
        \caption{Full \gls{NN}.}
    \end{subfigure}
    ~
    \begin{subfigure}{0.3125\textwidth}
        \centering
        \def\layersep{3.25em}
        \begin{tikzpicture}[shorten >=1pt,->,draw=black!50, node distance=\layersep]
            \tikzstyle{every pin edge}=[<-,shorten <=1pt]
            \tikzstyle{neuron}=[circle,draw=black,fill=white!100,minimum size=10pt,inner sep=0pt]
            \tikzstyle{input neuron}=[neuron];
            \tikzstyle{output neuron}=[neuron];
            \tikzstyle{hidden neuron}=[neuron];
            \tikzstyle{annot} = [text width=4em, text centered]

            \foreach \name / \y in {1,...,4}
            \node[input neuron, pin=left:$x_{\y}$] (I-\name) at (0,-\y) {};

            \foreach \name / \y in {1,...,5}
            \path[yshift=+0.5cm]
            node[hidden neuron] (Ha-\name) at (\layersep,-\y cm) {};

            \foreach \name / \y in {1,...,3}
            \path[yshift=-0.5cm]
            node[output neuron,pin={[pin edge={->}]right:$y_{\y}$}] (O-\name) at (2*\layersep, -\y cm) {};

            \pgfmathsetseed{1987111}

            \foreach \source in {1,...,4}
                {
                    \pgfmathparse{int(1000*rand+1000)} 
                    \foreach \dest in {1,...,5}
                        {
                            \ifnum \pgfmathresult>1000
                                \path (I-\source) edge[black, opacity=1] (Ha-\dest);
                            \else
                                \path (I-\source) edge[black, opacity=0] (Ha-\dest);
                            \fi
                        }
                }
            \foreach \source in {1,...,5}
                {
                    \pgfmathparse{int(1000*rand+1000)} 
                    \foreach \dest in {1,...,3}
                        {

                            \ifnum \pgfmathresult>1000
                                \path (Ha-\source) edge[black, opacity=1] (O-\dest);
                            \else
                                \path (Ha-\source) edge[black, opacity=0] (O-\dest);
                            \fi
                        }
                }
        \end{tikzpicture}
        \caption{\emph{Dropout}. Case $p=0.5$}
    \end{subfigure}
    ~
    \begin{subfigure}{0.3125\textwidth}
        \centering
        \def\layersep{3.25em}
        \begin{tikzpicture}[shorten >=1pt,->,draw=black!50, node distance=\layersep]
            \tikzstyle{every pin edge}=[<-,shorten <=1pt]
            \tikzstyle{neuron}=[circle,draw=black,fill=white!100,minimum size=10pt,inner sep=0pt]
            \tikzstyle{input neuron}=[neuron];
            \tikzstyle{output neuron}=[neuron];
            \tikzstyle{hidden neuron}=[neuron];
            \tikzstyle{annot} = [text width=4em, text centered]

            \foreach \name / \y in {1,...,4}
            \node[input neuron, pin=left:$x_{\y}$] (I-\name) at (0,-\y) {};

            \foreach \name / \y in {1,...,5}
            \path[yshift=+0.5cm]
            node[hidden neuron] (Ha-\name) at (\layersep,-\y cm) {};

            \foreach \name / \y in {1,...,3}
            \path[yshift=-0.5cm]
            node[output neuron,pin={[pin edge={->}]right:$y_{\y}$}] (O-\name) at (2*\layersep, -\y cm) {};

            \pgfmathsetseed{1987117}

            \foreach \source in {1,...,4}
            \foreach \dest in {1,...,5}
                {
                    \pgfmathparse{int(1000*rand+1000)} 
                    \ifnum \pgfmathresult>1000
                        \path (I-\source) edge[black, opacity=1] (Ha-\dest);
                    \else
                        \path (I-\source) edge[black, opacity=0] (Ha-\dest);
                    \fi
                }

            \foreach \source in {1,...,5}
            \foreach \dest in {1,...,3}
                {
                    \pgfmathparse{int(1000*rand+1000)} 
                    \ifnum \pgfmathresult>1000
                        \path (Ha-\source) edge[black, opacity=1] (O-\dest);
                    \else
                        \path (Ha-\source) edge[black, opacity=0] (O-\dest);
                    \fi
                }

        \end{tikzpicture}
        \caption{\emph{Dropconnect}. Case $p=0.5$}
    \end{subfigure}
    \caption{(a) The base graph of a \gls{NN} consisting of two layers. Here, the number of input, hidden, and output nodes are $h = 4, f = 5, e = 3$ respectively. (b) A random subgraph being trained by \emph{Dropout}. When applying canonical \emph{Dropout}, we drop every node of the graph with probability $1-p$ in an independent, identically distributed fashion. (c) A random subgraph being trained by \emph{Dropconnect}. When applying \emph{Dropconnect}, we drop edges with probability $1-p$ in an independent, identically distributed fashion.}
    \label{fig:Random_two_layer_NNs_with_Dropout_and_Dropconnect}
\end{figure}
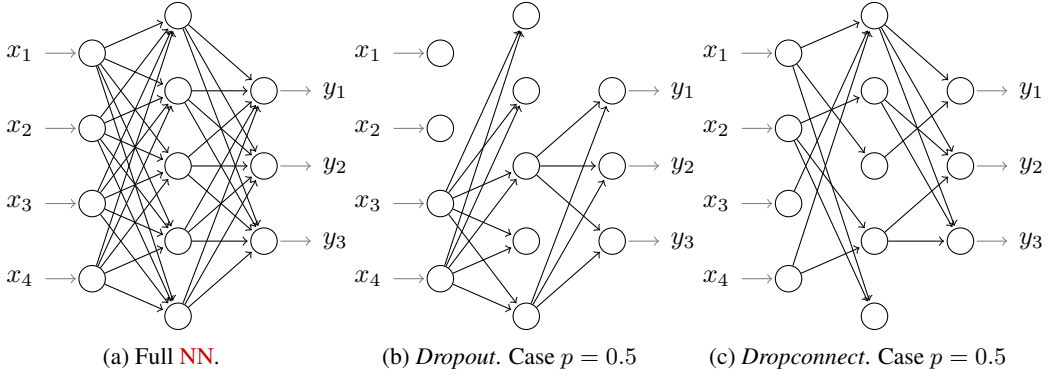

Dropout algorithms have been subject to extensive study in the literature. 
Most theoretical focus has been on the regularization effects of dropout algorithms \cite{hinton2012improving,baldi2013understanding,wager2013dropout,srivastava2014dropout,baldi2014dropout,cavazza2017dropout,mianjy2018implicit,mianjy2019dropout,pal2019regularization,wei2020implicit}, which are indeed a most welcome feature in practice. 
At the same time, it is known that a dropout algorithm executes stochastic gradient descent on an objective function that intricately depends on the probability distribution prescribing which of the weights are being trained each iteration \cite{mianjy2018implicit,mianjy2019dropout,pal2019regularization,senen2020almost}. 
Thus, one can anticipate that the additional variance introduced by a dropout algorithm will come at the cost of a lower convergence rate.

There are relatively few papers on the theoretical convergence properties of dropout algorithms. In \cite{senen2020almost}, it is proven that the iterates of a relatively large class of dropout algorithms converge asymptotically to a stationary set of a set of ordinary differential equations, when applying them to \glspl{NN} with polynomially bounded, smooth activation functions. The authors also show that when one applies \emph{Dropout} or \emph{Dropconnect} to deep linear \glspl{NN} structured as a tree consisting of $L$ layers, the convergence rate decreases by a factor $p^L$. In \cite{Mianjy2020OnCA}, a classification setting with shallow $\mathrm{ReLU}$ \glspl{NN} utilizing \emph{Dropout} is studied, and a nonasymptotic convergence rate of gradient descent for achieving $\epsilon$-suboptimality in the test error is derived. The result pertains to \glspl{NN} that are overparametrized, and relies on an assumption that the data distribution is separable with a margin in a particular reproducing kernel Hilbert space. Their derived convergence rate is independent of $1-p$, which as explained in \cite{Mianjy2020OnCA} is because of the separability assumption. 

The goal of the present paper is to further investigate the convergence rate of \emph{Dropout} and \emph{Dropconnect}, and in particular shed light on its dependency on properties of the data, the dropout probability $1-p$, and the structure of the \gls{NN}. To do so we investigate the convergence rate of the gradient flow on an objective function induced by \emph{Dropout} and \emph{Dropconnect} on shallow linear \glspl{NN}. The fact that there are relatively few convergence results regarding \glspl{NN} with dropout algorithms to build on \cite{senen2020almost,Mianjy2020OnCA}, combined with the additional challenges of a stochastic algorithm, means that at least for now, linear \glspl{NN} give the right balance of complexity and feasibility in order to obtain sharp rates. Our results in \refSection{sec:Results_shallow_linear_NN_local_convergence} show that these \glspl{NN} are sufficiently rich for new insights into the dependencies of the convergence rate on the data, the dropout probability, and the \glspl{NN}'s structure.

\subsection{Summary of results}

This paper investigates the convergence rate of the gradient flow of an ordinary differential equation that approximates the behavior of \emph{Dropout} and \emph{Dropconnect} when applying them to shallow linear \glspl{NN}. To see the relation, consider that both algorithms update the weights matrices $W = (W_2, W_1)$ of this \gls{NN} iteratively by setting
\begin{equation}
	\itr{W}{n+1}
	= 
	\itr{W}{n} - \itrd{\eta}{n} \itr{\Delta}{n+1}
	\label{eqn:Dropout_algorithm__Unprojected}
\end{equation}
for $n = 0, 1, 2$, \emph{et cetera}. Here, $\itr{\eta}{n} > 0$ denote the step sizes of the algorithm, and the $\itr{\Delta}{n+1}$ represent the random directions that result from the act of dropping edges or nodes and calculating the gradient of the resulting random \glspl{NN}. These random directions satisfy in shallow linear \glspl{NN} with whitened data \cite{mianjy2018implicit,senen2020almost}
\begin{align}
    \expectation{ \itr{\Delta}{n+1} \mid \itr{W}{0}, \ldots, \itr{W}{n} } 
    &
    = \nabla \mathcal{J}( \itr{W}{n} )
    ,
    \quad
    \textnormal{where}
    \nonumber \\     
    \mathcal{J}(W)
    &    
    = 
    \pnorm{Y - a W_2 W_1}{\mathrm{F}}^2 
    + b \trace{ \mathrm{Diag}(W_1W_1^T) \mathrm{Diag}(W_2W_2^T) }
    .
    \label{eqn:dropout_regular_intro}
\end{align}
Here, $Y$ is a matrix encoding the whitened data, and $\pnorm{\cdot}{\mathrm{F}}^2$ denotes the Frobenius norm. The constants $a, b$ have one closed-form expression in terms the probability $1-p$ of dropping nodes when using \emph{Dropout}, and another closed-form expression in terms of the probability $1-p$ of dropping edges when using \emph{Dropconnect} (see \refSection{sec:preliminaries_dropout_risk}). 
For diminishing step sizes $\itr{\eta}{n}$, we may therefore view both \emph{Dropout} and \emph{Dropconnect} schemes described in \eqref{eqn:Dropout_algorithm__Unprojected} as being noisy discretizations\footnote{
    Observe that the algorithm in \eqref{eqn:Dropout_algorithm__Unprojected} satisfies
$
    \itr{W}{n+1}
	= 
    \itr{W}{n} 
    + 
    \itrd{\eta}{n} 
    ( 
        - \nabla \mathcal{J}( \itr{W}{n} ) + \itr{M}{n+1} 
    )
$
where
$
    \itr{M}{n+1} 
    = 
    \expectation{ \itr{\Delta}{n+1} \mid \itr{W}{0}, \ldots, \itr{W}{n} } - \itr{\Delta}{n+1}
$
describes a \emph{martingale difference sequence}. This martingale difference sequence's expectation with respect to the past $\itr{W}{0}, \ldots, \itr{W}{n}$ is zero.
}
of the ordinary differential equation
\begin{equation}
    \frac{\d{W}}{\d{t}}
    = - \nabla \mathcal{J}( W(t) ).
    \label{eqn:Gradient_flow_introduction}
\end{equation}
This argument is not complete: to formally establish that the random iterates $\{ \itr{W}{n} \}$ indeed follow the trajectories of the gradient flow in \eqref{eqn:Gradient_flow_introduction}, one may employ the so-called \emph{ordinary differential equation method} \cite{bertsekas1996neuro,kushner2003stochastic,borkar2009stochastic,senen2020almost}. The present paper takes the relation however for granted and is about estimating the convergence rate of the gradient flow in \eqref{eqn:Gradient_flow_introduction}.

For shallow linear \glspl{NN}, \refTheorem{thm:main_proposition_maintext} in \refSection{sec:Results_shallow_linear_NN_local_convergence} gives an upper bound to the convergence rate for the gradient flow of $\mathcal{J}(W)$ when starting close to a minimizer. Informally stated, we prove that 
\begin{equation}
    \textnormal{if}
    \quad
    \frac{\d{W}}{\d{t}}
    = - \nabla \mathcal{J}( W(t) ),
    \quad
    \textnormal{then}
    \quad
    d( W(t), M)
    \leq \e{ - \omega t } d( W(0), M )
    \quad
    \textnormal{for some}
    \quad
    \omega > 0.
    \label{eqn:Simplified_main_result_of_the_paper}
\end{equation}
Here, $M$ denotes the closed set of global minimizers of \eqref{eqn:dropout_regular_intro}, and $d(x,M) = \inf_{y \in  M} \{ \abs{x-y} \}$ denotes the Euclidean distance between the point $x$ and the set $M$. If $W(t)$ converges to a so-called \emph{balanced} minimizer (see \refSection{sec:set_minima_assumptions}), then \refTheorem{thm:main_proposition_maintext} also gives an implicit characterization of $\omega$ that depends on the probability of dropping nodes or edges $1-p$, the number of hidden nodes $f$, and the singular values $\sigma_1, \ldots, \sigma_r$ of the data matrix $Y$. 

For the case of a one-dimensional output ($e=1$), we give a closed-form expression for $\omega$ whenever $W(0)$ is close enough to $M$. Informally stated, \refProposition{prop:convergence_final_case_one} in \refSection{sec:Results_shallow_linear_NN_local_convergence} implies that the upper bound  for $\omega$ satisfies
\begin{equation}
    \omega
    \approx 
    \frac{2p^2(1-p^2)}{p^2 f + 1-p^2} \sigma_1
    \quad
    \textnormal{for \emph{Dropconnect}, and}
    \quad
    \omega
    \approx 
    \frac{2p(1-p)}{p f + 1-p}\sigma_1
    \quad
    \textnormal{for \emph{Dropout}}.
    \label{eqn:Simplified_convergence_rate_Dropconnect_and_Dropout}
\end{equation}

Results \eqref{eqn:Simplified_main_result_of_the_paper} and \eqref{eqn:Simplified_convergence_rate_Dropconnect_and_Dropout} shed light on the convergence rate of \emph{Dropout} and \emph{Dropconnect} for linear \glspl{NN}. For example, in an overparametrized regime ($f \gg e = 1$), the convergence rates $\omega$ in \eqref{eqn:Simplified_convergence_rate_Dropconnect_and_Dropout} decay as $2 \sigma/f$. Furthermore, for every $f$, there is a choice of $p^*$ that maximizes $\omega$; this maximizer satisfies $p^{*} = 1/\sqrt{1+\sqrt{f}}$ for \emph{Dropconnect} and $p^{*} = 1/(1+\sqrt{f})$ for \emph{Dropout}.  
Lastly, it must be remarked that these results and insights also pertain to certain matrix factorization problems. Indeed, the minimization of \eqref{eqn:dropout_regular_intro} is in fact a matrix factorization problem with a regularizer induced by dropout. This was originally observed in \cite{cavazza2017dropout, mianjy2018implicit,mianjy2019dropout}. 

In order to prove \eqref{eqn:Simplified_main_result_of_the_paper} and \eqref{eqn:Simplified_convergence_rate_Dropconnect_and_Dropout}, we use a result in \cite{fehrman2019convergence} on the convergence of gradient flow for nonconvex objective functions. We combine this result with a careful analysis of the set of minimizers of the dropout loss, and of its Hessian. As a set, the set of minimizers $M$ has been characterized in \cite{cavazza2017dropout,mianjy2018implicit,mianjy2019dropout} for \emph{Dropout} and \emph{Dropconnect} and we build on their result. A related but different loss landscape analysis within the context of \glspl{NN} can be found in \cite{nguyen2017loss}. 

Formally, our lower bound to $\omega$ in \eqref{eqn:Simplified_convergence_rate_Dropconnect_and_Dropout} holds only close to $M$. Nonetheless, we expect that the iterates of a gradient descent counterpart should exhibit a similar decay with an exponent similar to our lower bound in \eqref{eqn:Simplified_convergence_rate_Dropconnect_and_Dropout} with enough iterations. To substantiate this claim, we show simulation results in \refSection{sec:Numerical_experiments} that compare numerically measured convergence rates to the rate in \eqref{eqn:Simplified_convergence_rate_Dropconnect_and_Dropout}. The simulations show that indeed, the convergence rate of the gradient descent counterpart exhibits similar qualitative dependencies as our bound in \eqref{eqn:Simplified_convergence_rate_Dropconnect_and_Dropout} for different initializations. Moreover when starting sufficiently close to a minimizer, the dependency of the numerically measured convergence rates on $f$ matches the decay provided by the bound $\omega$: this indicates that our bound in \eqref{eqn:Simplified_convergence_rate_Dropconnect_and_Dropout} bound is sharp.

Up to their conditions, \refTheorem{thm:main_proposition_maintext} and \refProposition{prop:convergence_final_case_one} achieve our goal of shedding light on the dependency of the convergence rate on properties of the data, the dropout probability, and the structure of the \gls{NN}. These new results add to the relatively scarce literature on the convergence properties of dropout algorithms and imply that the convergence rate of the stochastic \emph{Dropout} and \emph{Dropconnect} algorithms will intricately depend on the data, the dropout probability, and the \gls{NN}'s structure. By extension, we expect that this conclusion must also hold for nonlinear \glspl{NN}, though quantitatively establishing such fact requires more research.

\paragraph{Outline of the paper.} 
In \refSection{sec:Introduction}, we have introduced the problem and given an overview of the related literature on convergence results on \glspl{NN}, and for dropout algorithms. In \refSection{sec:Preliminaries}, we lay out notation and describe the dropout setting formally. \refSection{sec:Results_shallow_linear_NN_local_convergence} contains our main results and discussions thereof. In \refSection{sec:outline_proof} we outline the steps of the proof of the main results, while all details are relegated to the appendices. \refSection{sec:Numerical_experiments} compares our theoretical bound on the convergence rate, to numerical measurements of simulations of actual convergence rates. Finally, we conclude in \refSection{sec:Conclusion}. 
\section{Preliminaries}
\label{sec:Preliminaries}

\subsection{Shallow neural networks}

\glsreset{NN}
Consider the problem of finding a function $\Psi_W$ amongst a parametric family $\{ \Psi_{W}: \R^{h} \to \R^{e}, W \in \mathcal{P} \}$ that can best predict every ground truth output $t(x) \in \R^e$ belonging to any input $x \in \R^h$. Here, $\mathcal{P} \subseteq \R^d$ denotes a parameter space of dimension $d$, say. Usually, there is no access to the ground truth, and instead one uses a sample of $n$ pairs of input/output data points $\{(x_i, y_i) \}_{i=1}^{n} \subset \R^{h} \times \R^e$ and then aims to find the best function $\Psi_W$ in the parametric family by 
\begin{equation}
    \textnormal{minimizing}
    \quad
    \mathcal{R}(W) 
    \triangleq \sum_{i=1}^n \pnorm{y_{i} - \Psi_{W}(x_i)}{2}^2
    \quad
    \textnormal{over}
    \quad
    W \in \mathcal{P}.
    \label{eqn:loss_function_general_hypothesis}
\end{equation}
Here, $\mathcal{R}(W)$ is called the \emph{empirical risk}. In this paper, we assume that the data points are fixed, and we make no reference to their underlying distribution. 

Shallow \glspl{NN} constitute one parametric family used to solve \eqref{eqn:loss_function_general_hypothesis}. Concretely, let $e, f, h \in \naturalNumbersPlus$ denote the dimensions of the input, hidden, and output layer, respectively; see \refFigure{fig:Random_two_layer_NNs_with_Dropout_and_Dropconnect}. A shallow \gls{NN} with parameters 
$
W 
= (W_2,W_1) 
\in \R^{e \times f} \times \R^{f \times h} 
\triangleq \mathcal{P}
$
is then given by the function
\begin{align}
\Psi_{W}(x) = W_2\vartheta(W_1 x), 
\quad
\textnormal{where}
\quad
\vartheta : \R \to \R
\quad
\textnormal{is applied component-wise}.
\label{def:shallow_NN_model}
\end{align}
Here, the weights of the second and first layer are collected in the matrices $W_2$ and $W_1$, respectively. Common choices for the function $\vartheta$ include $\textrm{ReLU}(t) = \max \{ 0, t \}$ and $1/(1+\e{-t})$.

In this paper we focus on shallow linear \glspl{NN}, that is, the parametric family of functions spanned by $\Psi_{W}(x) = W_2 W_1 x$ so $\vartheta(t) = t$. For these \glspl{NN}, the optimization problem in \eqref{eqn:loss_function_general_hypothesis} is already quite challenging, as the empirical risk turns out to be nonconvex. Even though the expressiveness of a shallow linear \gls{NN} is low compared to e.g.\ a shallow \gls{NN} with $\vartheta(t) = \mathrm{ReLU}(t)$, their analysis is common in the optimization literature \cite{arora2018convergence,bah2019learning,Bartlett2018GradientDW}. The analyses of linear \gls{NN} are namely thought to give insight also into the optimization of nonlinear \glspl{NN}.

\subsection{Data whitening}

Data whitening is a preprocessing step that rescales the data points  such that their empirical covariance matrix equals the identity. Let $\mathcal{X} = (x_1, \ldots, x_n) \in \R^{h \times n}$ and $\mathcal{Y} = (y_1, \ldots, y_n) \in \R^{e \times n}$ be matrices containing the input and output data points, respectively. In order to be able to whiten the data, one must assume that $\mathcal{X} \mathcal{X}^T \in \R^{h \times h}$ is nonsingular. 

Under said assumption, define now the matrix $Y = \mathcal{Y} \mathcal{X}^T (\mathcal{X} \mathcal{X}^T)^{-1/2} \in \R^{e \times h}$, where $(\mathcal{X}\mathcal{X}^T)^{-1/2}$ is the inverse of the unique positive definite square root of $\mathcal{X}\mathcal{X}^T$. In \refAppendixSection{secappendix:data_whitening}, we derive that
\begin{equation}
\mathcal{R}(W)
= \pnorm{ \mathcal{Y} - W_2 W_1 \mathcal{X} }{\mathrm{F}}^2
= \pnorm{ Y - W_2 W_1 (\mathcal{X}\mathcal{X}^T)^{1/2} }{\mathrm{F}}^2 + c
\label{eqn:Data_whitened_objective_function}
\end{equation}
for some constant $c \in \realNumbers$ independent of $W$. Consequently, after data whitening, which applies a transformation $(W_2,W_1) \to (W_2, W_1 (\mathcal{X}\mathcal{X}^T)^{-1/2} ) $, we may focus on 
\begin{equation}
\textnormal{minimizing}
\quad
R(W) 
\triangleq \pnorm{ Y - W_2 W_1 }{\mathrm{F}}^2 
\quad
\textnormal{over}
\quad
W \in \mathcal{P}
\label{eqn:loss_linear_nowhitening}
\end{equation}
instead of the minimization problem in \eqref{eqn:loss_function_general_hypothesis}. The reader can find the explicit details of this transformation in \refAppendixSection{secappendix:data_whitening}.

\subsection{\emph{Dropout} and \emph{Dropconnect} on shallow, linear \texorpdfstring{\glspl{NN}}{NNs} with whitened data}
\label{sec:preliminaries_dropout_risk}

If the data samples are whitened first and kept fixed, then we can view a dropout algorithm on a shallow linear \gls{NN} as a stochastic algorithm that finds a stationary point of the objective function \cite{senen2020almost}
\begin{equation}
    \mathcal{J}(W) 
    = \expectationWrt{ R(F \odot W) }{}
    = \expectationWrt{ \pnorm{Y - (W_2 \odot F_2) (W_1 \odot F_1)}{\mathrm{F}}^2 }{}.
    \label{eqn:loss_dropout_linear_general}
\end{equation}
Here, $F \odot W$ denotes the component-wise product of each of the elements of the weight matrices $W = (W_2,W_1)$ by the elements of the two random matrices $F = (F_2, F_1) \in \{0,1\}^{e \times f} \times \{0,1\}^{f \times h}$. The expectation here is with respect to the distribution of $F$. \refLemma{lemma:dropout_loss_explicit_with_p} contains explicit expressions for \eqref{eqn:loss_dropout_linear_general} for the cases of \emph{Dropout} and \emph{Dropconnect}. Here, $p$ denotes the probability that a node or edge, respectively, remains. The proof of \refLemma{lemma:dropout_loss_explicit_with_p} is relegated to \refAppendixSection{secappendix:dropout_J}. 

\begin{lemma}
    \label{lemma:dropout_loss_explicit_with_p}
    When using \emph{Dropout},
    \begin{equation}
        \mathcal{J}(W)
        = 
        \norm{Y - p W_2 W_1}_{\mathrm{F}}^2 + (p - p^2) \mathrm{Tr}[ \mathrm{Diag}(W_1W_1^{\mathrm{T}}) \mathrm{Diag}(W_2W_2^{\mathrm{T}}) ].
        \label{eqn:dropout_regular}
    \end{equation}
    When using \emph{Dropconnect},
    \begin{equation}
        \mathcal{J}(W)
        = 
        \norm{Y - p^{2} W_2 W_1}_{\mathrm{F}}^2 + (p^2 - p^4) \mathrm{Tr}[ \mathrm{Diag}(W_1W_1^{\mathrm{T}}) \mathrm{Diag}(W_2W_2^{\mathrm{T}}) ].
        \label{eqn:dropconnect}
    \end{equation}
\end{lemma}

For convenience and without loss of generality, we now choose to scale both weight matrices $W_2, W_1$ by $1/\sqrt{p}$ in the case of \emph{Dropout}, and by a factor $1/p$ in \eqref{eqn:dropout_regular_intro} in the case of \emph{Dropconnect}. Concretely, this means that we will study the \emph{scaled risk function}
\begin{equation}
\mathcal{I}(W) 
\triangleq  \norm{ Y - W_2 W_1}_{\mathrm{F}}^2 + \lambda \textrm{Tr}[ \mathrm{Diag}(W_2^{\mathrm{T}} W_2) \mathrm{Diag}(W_1 W_1^{\mathrm{T}}) ]
\label{eqn:loss_whitened_scaled_dropout}
\end{equation}
with $\lambda = (1-p)/p$ in the case of \emph{Dropout}, and $\lambda = (1-p^2)/p^2$ in the case of \emph{Dropconnect}. The parameter $\lambda$ relates to the relative strength of the regularization term in either dropout algorithm and becomes large whenever the dropout probability $1-p$ increases.

\subsection{Characterization of the set of global minimizers}
\label{sec:Characterization_of_the_set_of_global_minimizers}

The set of global minimizers of \eqref{eqn:loss_whitened_scaled_dropout} have been characterized implicitly in \cite{mianjy2018implicit,mianjy2019dropout}. We build on one of their results, which we repeat here for your convenience. Concretely, let
\begin{equation}
M = \{ W \in \mathcal{P} : \mathcal{I}(W) = \inf_{s \in \mathcal{P}} \mathcal{I}(s) \}
\label{eqn:definition_M_maintext}
\end{equation}
be the set of global minimizers. Let the nonzero singular values of $Y$ be denoted by $\sigma_1 \geq \cdots \geq \sigma_{r}$ with $r \leq \min(e,h)$; and let the compact \gls{SVD} of $Y$ be $U_{\mathrm{c}} \Sigma_{Y} V_{\mathrm{c}}$ where thus $\Sigma_{Y} = \textrm{Diag}(\sigma_1, \ldots, \sigma_r)$ and $U_{\mathrm{c}}^TU_{\mathrm{c}} = V_{\mathrm{c}}V_{\mathrm{c}}^T =\mathrm{I}_r$. Introduce
\begin{equation}
    \kappa_j
    = \frac{1}{j} \sum_{i=1}^j \sigma_{i},
    \quad
    \rho
    = \max \Big \{  j \in [f] ~ : ~ \sigma_j > \frac{j \lambda \kappa_j}{f + j \lambda} \Big \},
    \quad
    \textnormal{and}
    \quad
    \alpha
    = \frac{ \rho \lambda \kappa_{\rho} }{ f + j \lambda }.
    \label{eqn:Definition_rho_and_alpha}
\end{equation}
Define now the \emph{shrinkage thresholding operator with threshold $\alpha$}, which applied to $Y$ is given by 
\begin{equation}
    \mathcal{S}_{\alpha}(Y)
    = U_{\mathrm{c}} (\Sigma_{Y} - \alpha \mathrm{I}_r)_+ V_{\mathrm{c}},
    \quad
    \textnormal{where}
    \quad
    ((\Sigma_{Y} - \alpha \mathrm{I}_r)_+)_{ii} = \max(0, \sigma_i - \alpha).
    \label{eqn:Shrinkage_thresholding_operator}
\end{equation}
By \cite[Theorem 3.4, Theorem 3.6]{mianjy2018implicit}: if $W^* = (W_2^*, W_1^*) \in M$ and $\rho < f$, then 
\begin{equation}
    \mathcal{W}^{*}
    = W_2^* W_1^*
    = \mathcal{S}_{\alpha}[Y]
    \quad
    \text{and}
    \quad
    \mathrm{Diag}((W^{*}_2)^{\mathrm{T}} W^*_2)\mathrm{Diag}(W^*_1 (W^*_1)^{\mathrm{T}})
    = \frac{\norm{\mathcal{W}^{*}}^2_1}{f^2} \mathrm{I}_{f}.
    \label{eqn:optimal_diagonal_main}
\end{equation}
If $f = \rho$, then in \eqref{eqn:optimal_diagonal_main} the conclusion on $\mathcal{W}^*$ must be replaced by the fact that $\mathcal{W}^{*}$ equals the rank-$f$ approximation of $\mathcal{S}_{\alpha}[Y]$.\footnote{This is perhaps not immediately clear in \cite[Theorem 3.6]{mianjy2018implicit}  for the case $\rho = f \leq r$. The fact that the rank-$f$ approximation must be used instead follows from the second-to-last step in the proof of \cite[Theorem 3.6]{mianjy2018implicit}.}

\subsection{Subsets of balanced and diagonally balanced minimizers}
\label{sec:set_minima_assumptions}

The notion of \emph{(approximately) balanced} weights has been found to be a sufficient condition for gradient descent on the objective function of deep linear \glspl{NN} to converge to their minima \cite{arora2018convergence,bah2019learning}. This has also been observed experimentally in \emph{Dropout} for shallow linear \glspl{NN} \cite{mianjy2018implicit}. It may therefore be little surprise that we too will use the notion of balanced weights in our convergence proof. 

\begin{definition}
    \label{definition:balanced_maintext}
    Weights $(W_2, W_1) \in \mathcal{P}$ are \emph{balanced} if $W_2^{\mathrm{T}} W_2 = W_1 W_1^{\mathrm{T}}$. Weights $(W_2, W_1) \in \mathcal{P}$ are \emph{diagonally balanced} if $\mathrm{Diag}(W_2^{\mathrm{T}} W_2) = \mathrm{Diag}(W_1W_1^{\mathrm{T}})$. Let
    \begin{gather}
        M_b
        = 
        \bigl\{ 
            W = (W_2, W_1) \in M : W_2^{\mathrm{T}} W_2 = W_1 W_1^{\mathrm{T}} 
        \bigr\},
        \quad
        \textnormal{and}
        \label{eqn:Definition_Mb}
        \\
        M_{db}
        = 
        \bigl\{ 
            W = (W_2, W_1) \in M : \mathrm{Diag}(W_2^{\mathrm{T}} W_2) = \mathrm{Diag}(W_1 W_1^{\mathrm{T}})
        \bigr\}
        \label{eqn:Definition_Mdb}
    \end{gather}
    be the sets of \emph{balanced minimizers}, and \emph{diagonally balanced minimizers}, respectively.
\end{definition}

We will characterize the sets $M_b, M_{db}$ explicitly as part of our proof. To that end, \refLemma{lemma:first_integral_main} contains a key observation that we will use. It is proven in \refAppendixSection{secappendix:Proof_that_Mdb_equals_Mb}.

\begin{lemma}
The set $M_b = M_{db}$, and is an invariant set for the gradient flow of \eqref{eqn:loss_whitened_scaled_dropout}.
\label{lemma:first_integral_main}
\end{lemma}
\section{Results}
\label{sec:Results_shallow_linear_NN_local_convergence}

\subsection{Assumptions}

We rely on the following assumptions. Both assumptions are mild and expected to hold in most cases as they rely on generic properties of matrices. For a brief discussion on these assumptions, we refer to \refAppendixSection{secappendix:on_the_assumptions}.

First, we limit the degree of symmetry of the set of global minimizers to be able to characterize $M_b$ explicitly. Concretely, we rely on the following assumption on the multiplicity of singular values:

\begin{assumption}
Let $r=\mathrm{rk}(Y) \leq \min(e,h)$ and let the positive singular values $\{ \sigma_i \}_{i=1}^r$ of $Y$  satisfy $\sigma_1 > \cdots > \sigma_r > 0$.
\label{ass:eigenvalues_different_maintext}
\end{assumption}

We also want $M_b$ to be smooth enough to characterize the local behavior of the gradient flow. Concretely, in a neighborhood of $W \in M_{b}$, we would like $M_b$ to be a proper submanifold of $\mathcal{P}$ without singular points. To guarantee that $M_b$ is a manifold `almost everywhere,' the following assumption suffices:

\begin{assumption}
There exists some $W \in M_b$ with full \gls{SVD} $W = (U \Sigma_2 S, S^{T} \Sigma_1 V)$ such that $S$ has no zero entries.
\label{ass:nonvanishing_maintext}
\end{assumption}

\subsection{Convergence rate of \texorpdfstring{gradient flow}{gradient flow} on \emph{Dropout} and \emph{Dropconnect}'s risk functions}

We are now in position to state our main result. Here, for $W \in M$,
\begin{equation}
	V_{R/2, \delta}(W) 
	= 
	\{ 
		x \in M \cap U
		: 
		d(x, M \cap U) 
		= d(x, \bar{B}_{R/2}(W) \cap M \cap U) 
		< \delta 
	\}
\label{eqn:definition_V_set}
\end{equation}
and
$
	\bar{B}_{R/2}(W)
	= \{ x \in \mathcal{P}: \pnorm{x - W}{} \leq R \}
$. 

\begin{theorem}
	\label{thm:main_proposition_maintext}
	Presume Assumptions~\ref{ass:eigenvalues_different_maintext}, \ref{ass:nonvanishing_maintext}. For a generic\footnote{ We understand generic here in an `almost everywhere' sense. $M$ is an algebraic variety defined as the zero locus of a set of polynomials from \eqref{eqn:optimal_diagonal_main}. A point $W \in M$ is smooth in $M$ whenever the rank of a Jacobian is maximal. Only at the points where the rank is not maximal we do not have generic points. This occurs only in an algebraic set of strictly lower dimension than that of $M$. Formally, a generic set of the algebraic variety $M$ consists of all $W \in M$ up to a proper Zariski closed set in $M$. See \cite{smith2004invitation} for reference.} $W \in M$, there exists a neighborhood $U_{W} \subseteq \mathcal{P}$ of $W$, $\delta_0 > 0$, and $R_0 > 0$ such that: for all $\delta \in (0, \delta_0]$, $R \in (0, R_0]$ and $\theta: \mathcal{P} \times [0, \infty) \to \mathcal{P}$ satisfying
	\begin{equation}
		\frac{ \d{\theta_t} }{ \d{t} }
		= - \nabla \mathcal{I}(\theta_t)
		\quad
		\textnormal{and}
		\quad
		\theta_0 \in V_{R/2,\delta}(W),
		\label{eqn:gradient_flow_on_I}
	\end{equation}
	there exists a $\omega_{U} > 0$ such that
	\begin{equation}
		d(\theta_t, U_{W} \cap M) \leq \exp(-\omega_{U} t) d(\theta_0, U_{W} \cap M)
		\quad
		\textnormal{for all}
		\quad
		t \in (0, \infty).
	\end{equation}	

	If moreover $W \in M_b$, then there exists an $\epsilon_{U} \geq 0$ such that $\omega_{U} \in [\omega_{W} - \epsilon_{U}, \omega_W + \epsilon_{U}]$, where 
	\begin{equation}
		\omega_{W}
		= \begin{cases}
		\min 
		\bigl\{
			2 \frac{\lambda \kappa_{\rho}\rho }{f + \lambda \rho} - 2\sigma_{\rho + 1}
			, 
			\zeta_W 
		\bigr\} &\textnormal{ if } \rho < f \\
		\min 
		\bigl\{
			2 (\sigma_{\rho} -\sigma_{\rho + 1})
			, 
			\zeta_W 
		\bigr\} &\textnormal{ if } \rho = f 
		\end{cases}.
		\label{eqn:convergence_rate_omega}
	\end{equation}
	Here, $\sigma_{\rho+1} = 0$ if $r = \rho$, and $\zeta_{W}>0$ depends implicitly on $W$, $p$ and $\sigma_1, \ldots, \sigma_r$.
\end{theorem}

Note that \refTheorem{thm:main_proposition_maintext} gives an upper bound for the convergence rate of gradient flow on $\mathcal{I}(W)$ as long as we start close enough to the set of minima. Moreover, near balanced minimizers $(W_2^*, W_1^*) \in M_b$ a partially explicit bound is given in \eqref{eqn:convergence_rate_omega}. Note that to obtain an upper bound for the convergence rate for gradient flow on $\mathcal{J}(W)$ in \eqref{eqn:loss_dropout_linear_general}, we need to multiply $\omega_{U}$ by $p$ for \emph{Dropout} or $p^2$ for \emph{Dropconnect}.

If the output has dimension one ($e = 1$), then we prove the following special case of \refTheorem{thm:main_proposition_maintext}:

\begin{proposition}
	\label{prop:convergence_final_case_one}
	If the output dimension is one ($e=1$), then:
	(i) \refAssumption{ass:nonvanishing_maintext} holds, and
	(ii) we can replace  \eqref{eqn:convergence_rate_omega} in \refTheorem{thm:main_proposition_maintext} by
	\begin{equation}
		\omega_{U}
		\in [\omega_{W} - \epsilon_{U}, \omega_W + \epsilon_{U}]
		\quad
		\textnormal{where}
		\quad
		\omega_{W}
		= 2\frac{\lambda}{f + \lambda} \sigma_1.
		\label{eqn:convergence_rate_omega__case_e_1}
	\end{equation}
\end{proposition}

While \refTheorem{thm:main_proposition_maintext} already hints at dependencies on the hyperparameters, \refProposition{prop:convergence_final_case_one} provides an upper bound for $\omega$ that explicitly depends on the singular value $\sigma_1$ of the data $Y$, the probability $1-p$ of dropping nodes (or edges) encoded in $\lambda$, and the number of nodes in the hidden layer $f$. In particular, we obtain from \refProposition{prop:convergence_final_case_one} the rates 
\begin{equation}
\omega_{W}^{\textnormal{DC}} 
= \frac{2(1-p^2)}{p^2 f + 1-p^2} \sigma_1,
\quad
\omega_{W}^{\textnormal{DO}} 
= \frac{2(1-p)}{p f + 1-p} \sigma_1,
\label{eqn:convergence_rate_Dropconnect_and_Dropout}
\end{equation}
also in \eqref{eqn:Simplified_convergence_rate_Dropconnect_and_Dropout}, after multiplying by the scaling $p$ and $p^2$ for the cases of \emph{Dropconnect}, \emph{Dropout}, respectively.

\subsection{Discussion}
\label{sec:Discussion}

\refTheorem{thm:main_proposition_maintext} yields a convergence rate that depends on the singular values of the data matrix $Y$, the dropout probability $1-p$, and the structure parameters $e, f, h$ of the \gls{NN}. Observe that the rate $\omega_W$ in \eqref{eqn:convergence_rate_omega} is the minimum of two terms. The first term $\zeta_W$ depends on the point $W$ as well as $p$ and gives a local perspective on the convergence rate's dependency on the initialization and dropout probability (see \refAppendixSection{sec:Proof_of_the_lower_bound_to_the_Hessian_restricted_to_directions_normal_to_the_manifold_of_minima} for its exact dependency)---for our purposes, the fact that it is strictly positive suffices. The second term is namely independent of $W \in M_b$ and provides a more global perspective on the convergence rate's dependency on the data matrix, the dropout probability, and the structure parameters. It is furthermore noteworthy that the dependency on $\zeta_{W}$ disappears in the case $e = 1$, as evidenced from \refProposition{prop:convergence_final_case_one}.

We obtain the rates in \eqref{eqn:Simplified_convergence_rate_Dropconnect_and_Dropout} from \refProposition{prop:convergence_final_case_one} through a multiplication using the scalings discussed above \eqref{eqn:loss_whitened_scaled_dropout}. We observe then that \emph{Dropout} and \emph{Dropconnect} have an impaired convergence rate: the convergence rate in \eqref{eqn:Simplified_convergence_rate_Dropconnect_and_Dropout} is reduced by a factor $p$ in the case of \emph{Dropout}, and $p^2$ in case of \emph{Dropconnect}. This is in agreement with the results in \cite{senen2020almost}. Observe furthermore from \eqref{eqn:Simplified_convergence_rate_Dropconnect_and_Dropout}  that as $p \uparrow 1$, i.e., a regime without dropout, $\omega \downarrow 0$. This tells us that for small dropout rates, convergence is apparently slow for some trajectories of the gradient flow problem. This is explained by the fact that for $p \approx 1$, points $W$ satisfying $Y = W_2 W_1$ are almost minimizers of $\mathcal{J}(W)$. Finding an exact minimizer becomes then less important since there is almost no regularization.

Note also that the rates in \eqref{eqn:Simplified_convergence_rate_Dropconnect_and_Dropout} tell us that in the overparametrized regime $f \gg e = 1$, for every $f$ there is a dropout probability $1-p^*$ that maximizes the rate $\omega$. Solving $\d{\omega}/\d{p} = 0$ shows that,
\begin{align}
	p^{*} 
	&
	= \frac{1}{\sqrt{1+\sqrt{f}}} \sim \frac{1}{f^{1/4}}
	\quad 
	\textnormal{for \emph{Dropconnect}, and} 
	\quad
	p^{*} 
	= \frac{1}{1+\sqrt{f}} \sim \frac{1}{f^{1/2}}
	\quad 
	\textnormal{for \emph{Dropout}}.
	\label{eqn:optimal_p_maximizing_omega}
\end{align}
Setting $p^{*}$ as in \eqref{eqn:optimal_p_maximizing_omega} still implies that the maximizing convergence rate is $\omega^{*} \sim 2\sigma /f$. Hence, the maximizing dropout probability $1-p^*$ will still have limited influence on the convergence rate in this regime. To see this more explicitly, consider that for \emph{Dropout} the best rate $\omega^{*} = \omega(p^{*})$ compared to the rate when choosing a generic dropout probability $1-p \in ( \delta, 1 -\delta)$ satisfies $\omega(p)/\omega^{*} \gtrsim \delta$.

Lastly, let us also consider the matrix factorization problem in which $f \ll e, h$. It follows from \refTheorem{thm:main_proposition_maintext} that when doing matrix factorization with \emph{Dropout} regularization, degeneracies at the minimum are avoided when $\mathrm{rk}(Y) < e, h$, i.e., when the data is of low rank. If $Y$ does not have full rank, then the minima of the usual risk function for matrix factorization $\pnorm{Y - W_2 W_1}{\mathrm{F}}$ are degenerate, and this may impair convergence of gradient descent. In contrast, for the objective function in \eqref{eqn:loss_whitened_scaled_dropout} with a $\lambda > 0$, the set of minima $M$ around a point $W \in M_b$, only becomes degenerate as $p \uparrow 1$. Indeed, we have $2 {\lambda \kappa_{\rho}\rho } / ({f + \lambda \rho}) \downarrow 0$ as $p \uparrow 1$ (recall that $\zeta_{W} > 0$ for any $p \in (0,1)$). When $\rho = f$, we see that up to the term $\zeta_{W}$, there is no dependence on $1-p$ on the convergence rate. Dependence starts appearing when we have $\sigma_{f} \simeq 2 {\lambda \kappa_{\rho}\rho } / ({f + \lambda \rho})$.

\section{Proofs}
\label{sec:outline_proof}

The proofs of \refTheorem{thm:main_proposition_maintext} and \refProposition{prop:convergence_final_case_one} are based on two ideas. The first idea is that the trajectories of a gradient flow, when starting close to a minimizer in $W^* \in M$, should depend to leading order only on the Hessian $\nabla^2 \mathcal{I}(W^*)$. However, when $M$ is a connected set (or a manifold in this case), this may not be true. For it to hold we need the point $W^*$ to be \emph{nondegenerate}, in the sense that directions tangent to the manifold $M$ are included in the kernel of the Hessian and other directions must be orthogonal to $M$ and not in the kernel. The gradient flow to $M$ can then be locally bounded using the eigenvalues of $\nabla^{2} \mathcal{I}$. As it will turn out, `almost every' point in $M$ is nondegenerate. The second idea is that we can give an explicit lower bound to the eigenvalues of the Hessian by restricting to directions orthogonal to $M$. This requires careful computations and is the most involved part of the proof.

\subsection{Overview}

Here is an overview of the steps that will prove \refTheorem{thm:main_proposition_maintext} and \refProposition{prop:convergence_final_case_one}:

\begin{itemize}
	\item[\emph{Step 1.}] We formalize the first idea by relating a lower bound on the Hessian to the convergence rate of gradient flow by using a recent result on nonconvex optimization \cite{fehrman2019convergence}. This result holds whenever the gradient flow is started close to a minimizer in $M$, and requires the minimizer to be nondegenerate. Therefore, to prove \refTheorem{thm:main_proposition_maintext}, we next need to explicitly compute a lower bound to the Hessian in directions orthogonal to $M$ and verify the nondegeneracy condition.

	\item[\emph{Steps 2, 3.}] We reduce the set of minimizers $M$ to the set of balanced minimizers $M_b$ using a group action. The set of balanced minimizers is namely easier to handle: we can prove that $M_b$ is, up to a set of lower dimension than that of $M_b$, a manifold, i.e., $M_b$ is \emph{generic}. We compute the tangent space $\mathrm{T}_W M_b$ explicitly at a generic point $W$.\footnote{By this, we mean `for any $W \in M_b$ up to an algebraic set of lower dimension than $M$' (formally, `up to a proper closed Zariski set in the algebraic variety $M_b$').} Using the group action again, we can then also obtain the tangent space $\mathrm{T}_W M$ by extending the results from the set of balanced minimizers to the set of minimizers.

	\item[\emph{Steps 4, 5.}] We compute the Hessian $\nabla^2 \mathcal{I}$ and calculate a lower bound when $W \in M_b \subset M$. This also implies immediately that $W$ is nondegenerate in $M$. Leveraging the group action again, we can then show that all generic points in $M$ are nondegenerate.

	\item[\emph{Step 6.}] Finally, we combine the result of \emph{Step 1} with the bound and nondegeneracy property in \emph{Steps 4, 5} to prove \refTheorem{thm:main_proposition_maintext} and \refProposition{prop:convergence_final_case_one}.
\end{itemize}

\subsection{Key steps}

We now prove \refTheorem{thm:main_proposition_maintext} and \refProposition{prop:convergence_final_case_one} step by step as listed previously. The detailed proofs of each proposition presented here can be found in the Appendix.

\reversemarginpar

We \marginnote{\emph{Step 1.}} use a recent result on the convergence rate of gradient descent methods for general objective functions \cite{fehrman2019convergence}, in which local convergence in a neighborhood $U_{W}$ of $W \in M$ is guaranteed by local nondegeneracy of the Hessian:

\begin{definition}
	A set $M \subset \mathcal{P}$ of minimizers of $\mathcal{I}(W)$ is \emph{locally nondegenerate at $W$} if there exists a neighborhood $U \subseteq \mathcal{P}$ of $W$, such that: 
	\begin{itemize}[topsep=2pt,itemsep=2pt,partopsep=2pt,parsep=2pt,leftmargin=0pt]
		\item[(i)] $M \cap U$ is a submanifold of $\mathcal{P}$, and 
		\item[(ii)] for any $p \in M \cap U$, $\dim \mathrm{T}_{p} (M \cap U)) = \dim \ker \nabla^2 \mathcal{I}(p)$.
	\end{itemize}
	\label{definition:non_degenerate}
	We also say that the set $M \cap U$ is \emph{nondegenerate} if it is locally nondegenerate at any $W \in M \cap U$.
\end{definition}

Concretely, our first step is to prove the following specification of the bound in \cite[Proposition 3.1]{fehrman2019convergence}. The details are relegated to \refAppendixSection{secappendix:local_convergence_GF}.

\begin{proposition}[Adaptation of Proposition 3.1 in \cite{fehrman2019convergence}]
	Let $U \subseteq \R^{d}$ be an open subset and let $f: U \to \R$ be three times continuously differentiable. Let
	$
		M = \{ w \in \R^{d} : f(w) = \inf_{\mathcal{\theta} \in \R^{d}}f(\mathcal{\theta}) \}
	$
	and suppose that $U \cap M$ is a nonempty differentiable submanifold of $\R^d$ of dimension $\mathfrak{d} < d$. Suppose also that for all $p \in M \cap U$, $d - \mathfrak{d} = \mathrm{rk}(\nabla^2 f(p))$ holds. Then, for any $x_0 \in M \cap U$ there exists $R_0, \delta_0, \lambda \in (0, \infty)$ such that: for all $\delta \in (0, \delta_0]$, $R \in (0, R_0]$ and $\theta : (0, \infty) \to \R^d$ satisfying
	$
		\d{} \theta_t / \d{t}
		= -\nabla f (\theta_t)
	$
	and
	$
		\theta_0 \in V_{R/2, \delta}(x_0),
	$
	it holds that
	\begin{equation}
		d(\theta_t, M \cap U) \leq \exp(-\lambda t) d(\theta_0, M \cap U)
		\quad
		\textnormal{for all}
		\quad
		t \in (0, \infty)
	\end{equation}
	where specifically
	\begin{equation}
		\lambda = \min_{w \in \bar{V}_{R_0 , \delta_0}(W)} \min_{\substack{\norm{v} = 1\\ v \in \ker \nabla^2 f(w)^{\perp}}} \abs{v^T \nabla^2 f(w) v}.
		\label{eqn:lambda_cotangent_space}
	\end{equation}
	\label{prop:lambda_cotangent_space_is_enough}
\end{proposition}

Observe now that \refTheorem{thm:main_proposition_maintext} almost follows from \refProposition{prop:lambda_cotangent_space_is_enough} by identifying the function $f$ with the loss function $\mathcal{I}$---that is, up to \refTheorem{thm:main_proposition_maintext}'s conditions and up to \eqref{eqn:convergence_rate_omega}. Eq.\ \eqref{eqn:convergence_rate_omega} is in fact a lower bound to \eqref{eqn:lambda_cotangent_space}, and the conditions are what allow us to lower bound \eqref{eqn:lambda_cotangent_space} in the first place.

To see where the conditions of \refTheorem{thm:main_proposition_maintext} come from and how the bound in \eqref{eqn:convergence_rate_omega} is obtained, consider the following approach. Suppose for a moment that we were given some open subset $U$ that meets the conditions of \refProposition{prop:lambda_cotangent_space_is_enough} and that $M$ were nondegenerate. If these hypotheses were true, then the convergence rate in \eqref{eqn:lambda_cotangent_space} could be bounded by providing for each $W \in M \cap U$ a lower bound to the Hessian $\nabla^2 \mathcal{I}$ restricted to $T_W^\perp M$. This is because the nondegeneracy of $M$ would imply that for any  $W \in M$, $\ker \nabla^2 \mathcal{I}(w) = \mathrm{T}_{W} M$ and therefore $\ker \nabla^2 \mathcal{I}(w)^{\perp} = \mathrm{T}^{\perp}_{W} M$, and \eqref{eqn:definition_V_set} would then imply that $\overline{V}_{R,\delta}(x_0) \subseteq M \cap U$. 

The two hypotheses used in the approach above have however not been proven. Instead, we will first prove that for a generic $W \in M$ there exists a neighborhood $U$ satisfying the conditions of \refProposition{prop:lambda_cotangent_space_is_enough} (\emph{Steps 2, 3}), and this turns out to be sufficient. After this, we will establish that $\nabla^2 \mathcal{I}(W)|_{T_{W}^{\perp} M}$ is positive definite (\emph{Step 4}), and then we lower bound its minimum eigenvalue (\emph{Step 5}) which allows us to approximately characterize $\omega$ in \refTheorem{thm:main_proposition_maintext}.

We \marginnote{\emph{Step 2.}} start by characterizing $M$ using $M_b$ and a Lie group action on $M$. Let $H \simeq (\R^{*})^{f}$ be the Lie group of invertible diagonal matrices, where $\R^{*} = \R \backslash \{0\}$ is the multiplicative group of invertible elements in $\R$. We embed $H$ in $\R^{f \times f}$ via the diagonal inclusion $(a_1, \ldots, a_f) \to \mathrm{Diag}(a_1, \ldots, a_f) \in \R^{f \times f}$, and define the action $\pi$ of $C \in H$ on $M$ by 
\begin{equation}
\pi(C)(W_2, W_1) = (W_2 C, C^{-1} W_1).
\label{eqn:Action_pi}
\end{equation}

The action $\pi$ can be used to reduce $M$ to $M_b$, as formalized in \refProposition{prop:exists_diagonal_reduction_Mb_to_M}. We refer to \refAppendixSection{secappendix:reduction_M_M_b} for its proof.

\begin{proposition}
	\label{prop:exists_diagonal_reduction_Mb_to_M}
	For every $W \in M$ there exists a unique $C_W \in H$ such that $\pi(C_W)(W) \in M_{b}$.
\end{proposition}

For a subgroup $\mathrm{G}$ of $\mathrm{O}(f)$, we abuse notation and let $L \in \mathrm{O}(f) / \mathrm{G}$ be a representative $L \in \mathrm{O}(f)$ of the equivalence class $[L] \in  \mathrm{O}(f) / \mathrm{G}$ of cosets. Via the group action in \eqref{eqn:Action_pi} we can now characterize the set $M_b$: see  \refProposition{prop:characterization_Mb_main}, which is proven in \refAppendixSection{secappendix:characterazing_Mb}. 

\begin{proposition}
	\label{prop:characterization_Mb_main}
	If \refAssumption{ass:eigenvalues_different_maintext} holds, then
	\begin{equation}
		M_b
		= \Bigl\{
		(U \Sigma_2 L, L^T \Sigma_1 V)
		:
		L \in \frac{ \mathrm{O}(f) }{ \mathrm{I_{\rho} \oplus O(f-\rho)} },
		\mathrm{Diag}
		\bigl( L^{T}
		\begin{psmallmatrix}
			\Sigma^2 & 0 \\
			0 & 0 \\
		\end{psmallmatrix}
		L \bigr)
		= \frac{\norm{\Sigma^2}_1}{f} \mathrm{I_{f}}
		\Bigr\}
		\neq \emptyset.
		\label{eqn:Alternative_representation_of_Mb}
	\end{equation}
	Here, the columns of $U$ and $V$ contain the left- and right-singular vectors of $Y = U \Sigma_Y V$, respectively,
	\begin{equation}
		\Sigma_{2}
		= \begin{pmatrix}
			\Sigma & 0_{\rho \times (f - \rho)}\\
			0_{(e - \rho) \times \rho} & 0_{(e - \rho) \times (f - \rho)}
		\end{pmatrix},
		\quad
		\Sigma_{1}
		= \begin{pmatrix}
			\Sigma & 0_{\rho \times (h - \rho)}\\
			0_{(f - \rho) \times \rho} & 0_{(f - \rho) \times (h - \rho)}
		\end{pmatrix},
		\label{eqn:Definition_Sigma1_and_Sigma2}
	\end{equation}
	where $0_{n \times m}$ denotes the all zero matrix of size $n \times m$, and
	\begin{equation}
		\Sigma^2
		= \mathrm{Diag} \Bigl( \sigma_1 - \rho \frac{\lambda \kappa_{\rho}}{f + \rho \lambda}, \ldots, \sigma_{\rho} - \rho \frac{\lambda \kappa_{\rho}}{f + \rho \lambda} \Bigr) \in \R^{\rho \times \rho}.
		\label{eqn:Definition_Sigma_squared}
	\end{equation}
\end{proposition}

Next, \marginnote{\emph{Step 3.}} we identify $\mathrm{T}_{W} M_b$, the tangent space of $M_b$, whenever it is well defined for a $W \in M_b$. To do so, we find a manifold $\bar{M}_b \simeq \mathrm{O(f) / (I_{\rho}  \oplus O(f-\rho))}  $ such that $M_b \subseteq \bar{M}_b$ and a map $T: \bar{M}_b \to \R^{f}$, whose preimage defines $M_b$ and $\mathrm{T}_{W} M_b$ implicitly up to a set of singular points $\mathrm{Sing}(M_b)$. In particular, since for any $W \in M_b$ we have 
\begin{equation}
	\mathrm{Diag}(W_2^T W_2) = \mathrm{Diag}(W_1W_1^T) = \frac{\norm{\Sigma^2}_1}{f} \mathrm{I}_{f},
	\label{eqn:condition_balancedness_M_b}
\end{equation}
we will define the map $T: \bar{M}_b \to \R^f$ by
\begin{equation}
	T(W_2,W_1) = \mathrm{Diag}(W_2^T W_2) = \mathrm{Diag}(W_1W_1^T).
	\label{eqn:Definition_T_map}
\end{equation}
This map is well defined for each equivalence class in $\mathrm{O(f) / (I_{\rho}  \oplus O(f-\rho))} \simeq \bar{M}_b$ and has at most rank $f - 1$ instead of $f$, since the trace of \eqref{eqn:condition_balancedness_M_b} is fixed in $\bar{M}_b$. We can next use the \emph{implicit function theorem} \cite[Theorem 5.5]{lee2013smooth} to prove in \refAppendixSection{sec:Appendix__Characterization_of_TWMb} that:

\begin{proposition}
	\label{prop:Characterization_of_TWMb}
	Let $W \in M_{b} \backslash \mathrm{Sing}(M_{b})$, where
	\begin{equation}
	\mathrm{Sing}(M_b) 
	= \{ W \in M_b : \mathrm{rk}( \mathrm{D}_{W} T) < f - 1 \}.
	\label{eqn:Definition__Singular_points_of_Mb}
	\end{equation}
	If \refAssumption{ass:eigenvalues_different_maintext} holds, then there exist an open neighborhood $U_W \subset \mathcal{P}$ of $W \in U_{W}$ such that: 
	\begin{itemize}[topsep=2pt,itemsep=2pt,partopsep=2pt,parsep=2pt,leftmargin=0pt]
	\item[(a)] $U_W \cap M_b$ is a submanifold of $\bar{M}_b$ of codimension $f-1$, and 
	\item[(b)] $\mathrm{T}_W M_{b} = \ker{\mathrm{D}}_W T$, where the differential map $\mathrm{D}_W T: \mathrm{T}_W \bar{M}_b  \to \mathrm{T}_{T(W)}\R^{f}$ at $W = (U \Sigma_2 S, S^T \Sigma_1 V )$ given by
	\begin{align}
		\mathrm{D}_W T( V_2, V_1)
		&
		=\mathrm{D}_W T
		\left(
		U \Sigma_2
		\begin{pmatrix}
				X & E \\
				-E^T & 0
			\end{pmatrix}
		S,
		S^T
		\begin{pmatrix}
				X^T & -E \\
				E^T & 0 \\
			\end{pmatrix}
		\Sigma_1 V
		\right)
		\nonumber \\ &
		= 2 \mathrm{Diag}
		\left(
		S^T
		\begin{pmatrix}
				\Sigma^2 X & \Sigma^2 E \\
				0 & 0 \\
			\end{pmatrix}
		S
		\right).
		\label{eqn:Bilinear_form_Mb_implicit_M_Of}
	\end{align}
	\end{itemize}
\end{proposition}

The set $\mathrm{Sing}(M_b)$ contains the singular points of $M_b$, which are points where the usual tangent space cannot be defined in local coordinates. Consequently, $M_b$ cannot be a manifold at these points. Using this set of singular points and \refAssumption{ass:nonvanishing_maintext}, we prove that most points in $M_b$ are regular. The proof is relegated to \refAppendixSection{sec:Mb_is_almost_everywhere_nonsingular}:

\begin{proposition}
	\label{prop:Mb_is_almost_everywhere_nonsingular}
	If Assumptions~\ref{ass:eigenvalues_different_maintext}, \ref{ass:nonvanishing_maintext} hold, then:
	\begin{itemize}[topsep=2pt,itemsep=2pt,partopsep=2pt,parsep=2pt,leftmargin=0pt]
		\item[(a)] $\mathrm{Sing}(M_b)$ is a proper closed set in $M_{b}$;
		\item[(b)] $M_b$ is a manifold up to an algebraic set of lower dimension than that of $M_b$ (i.e., any generic point in $M_b$ is regular).
		\item[(c)] $M_{b}$ has codimension $f-1$ in $\bar{M}_b$.
	\end{itemize}
	Furthermore, if $\rho = 1$, then \refAssumption{ass:nonvanishing_maintext} is satisfied.
\end{proposition}

Now \marginnote{\emph{Step 4.}} that we have identified $\mathrm{T}_{W} M_b$ in \refProposition{prop:Characterization_of_TWMb}, we can use the fact that $M$ can be reduced to $M_b$ via the group action $\pi$ of \emph{Step 2}. This allows us to compute the tangent space $\mathrm{T}_{W} M$ at a nonsingular point $W \in M_b \backslash \mathrm{Sing}(M_b) \subset M$, and to also compute the cotangent space $\mathrm{T}_{W}^{\perp} M$. The latter task is done in \refLemma{lemma:tangent_cotangent_M} in \refAppendixSection{sec:Obtaining_cotangent_space_of_TWM}). 

Having now characterized the cotangent space $\mathrm{T}_{W}^{\perp} M$, wen continue by computing a lower bound to the Hessian. We start by calculating the Hessian in \refAppendixSection{sec:Proof_of_the_bilinear_form_of_the_Hessian}, and identify it as follows:

\begin{proposition}
	\label{prop:hessian_expression}
	For $W=(W_2, W_1) \in \mathcal{P}$, $(V_1, V_2) \in \mathrm{T}_{W} \mathcal{P}$, the Hessian $\nabla^2 \mathcal{I}(W)$ satisfies
	\begin{align}
		&
		\bigl( \mathrm{vec}(V_1), \mathrm{vec}(V_2) \bigr)^T \nabla^2 \mathcal{I}(W) \bigl( \mathrm{vec}(V_1), \mathrm{vec}(V_2) \bigr) 
		= 
		2\norm{W_2 V_1 + V_2 W_1}_{F}^2 
		\nonumber \\ &
		+ 2 \lambda \mathrm{Tr}[ V_1^T \mathrm{Diag}(W_2^T W_2) V_1 ] 
		+ 2 \lambda \mathrm{Tr}[ V_2 \mathrm{Diag}(W_1 W_1^T) V_2^T ] 
		- 4 \mathrm{Tr}[ V_1^T V_2^T( Y - \mathcal{S}_\alpha[Y] ) ] 
		\nonumber \\ &
		+ 2 \lambda 
		\bigl( 
			\pnorm{\mathrm{Diag}(V_2^T W_2) + \mathrm{Diag}(W_1^T V_1)}{\mathrm{F}}^2 
			- \pnorm{\mathrm{Diag}(V_2^T W_2) - \mathrm{Diag}(W_1^T V_1)}{\mathrm{F}}^2 
		\bigr)
	\label{eqn:second_derivative_dropout_loss_2layer}
	\end{align}
	as a bilinear form. Here, for any $A \in \R^{m \times n}$ we consider vectorization notation, that is, 
	\begin{equation}
		\mathrm{vec}(A) 
		= [a_{1,1}, \ldots, a_{m,1}, \ldots, a_{1,n}, \ldots, a_{m,n}]^{T} \in \R^{mn}.
		\label{eqn:Definition_vectorization}
	\end{equation}	
\end{proposition}

Finally, we lower bound the Hessian in the directions normal to the manifold of minima. The proof of the following is relegated to \refAppendixSection{sec:Proof_of_the_lower_bound_to_the_Hessian_restricted_to_directions_normal_to_the_manifold_of_minima}:

\begin{proposition}
	\label{prop:Lower_bound_to_the_Hessian_restricted_to_directions_normal_to_the_manifold_of_minima}
	Suppose Assumptions~\ref{ass:eigenvalues_different_maintext}, \ref{ass:nonvanishing_maintext} hold. 
	For any $W \in M_b \cap M \backslash \mathrm{Sing}(M) \subseteq M$, $\nabla^2 \mathcal{I}(W)$ restricted to $\mathrm{T}^{\perp}_W M$ is a positive definite bilinear form. Furthermore,
	\begin{equation}
		\nabla^2 \mathcal{I}(W)|_{\mathrm{T}^{\perp}_W M} \geq  \omega
	\end{equation}
	where
	\begin{equation}
		\omega
		= \begin{cases}
		\min \Bigl\{ \zeta_{W}, 2 \frac{\lambda \kappa_{\rho}\rho }{f + \lambda \rho} - 2\sigma_{\rho + 1} \Bigr\} &\textnormal{if } \rho < f\\
		\min \Bigl\{ \zeta_{W}, 2(\sigma_{\rho} - \sigma_{\rho + 1}) \Bigr\} &\textnormal{if } \rho = f
		\end{cases}.
		\label{eqn:lower_bound_hessian_minimum}
	\end{equation}
	Here, $\zeta_{W} > 0$ is a positive constant that depends on $W$, $\lambda$ and $\Sigma$. If $\rho = r$ (recall from \eqref{eqn:Definition_Sigma_squared} that we have $\rho \leq r$), then we set $\sigma_{\rho + 1} = \sigma_{r + 1} = 0$.

	In the case that $\rho = 1$, the result holds with \eqref{eqn:lower_bound_hessian_minimum} replaced by
	\begin{equation}
		\omega
		=
		\begin{cases}
			2 \sigma_1 \frac{\lambda  }{f + \lambda}
			&
			\textnormal{if } r = 1,
			\\
			2 \sigma_1 \frac{\lambda  }{f + \lambda} - 2\sigma_{2}
			&
			\textnormal{otherwise}.
		\end{cases}
		\label{eqn:lower_bound_hessian_minimum_case_1}
	\end{equation}
	Additionally, if $\rho = 1$, \refAssumption{ass:nonvanishing_maintext} is satisfied outright.
\end{proposition}

\refProposition{prop:Lower_bound_to_the_Hessian_restricted_to_directions_normal_to_the_manifold_of_minima} \marginnote{\emph{Step 5.}} reveals that for $W \in M_b \cap M \backslash \mathrm{Sing}(M) \subseteq M$, $M$ is locally nondegenerate at $W$---recall \refDefinition{definition:non_degenerate}. 
But in order to apply \refProposition{prop:lambda_cotangent_space_is_enough}, we need to also prove that $M$ is nondegenerate in a large enough neighborhood around such nonsingular point. By continuity, we then obtain a lower bound of the Hessian in a neighborhood of $\omega$: that is, the bound in \eqref{eqn:lambda_cotangent_space} will hold with $\lambda \in [\omega - \epsilon, \omega + \epsilon]$ for some $\epsilon > 0$. The following is proved in \refAppendixSection{sec:Proof_of_proposition_neighborhood_exists}.

\begin{proposition}
	\label{prop:neighborhood_exists}
	Suppose Assumptions~\ref{ass:eigenvalues_different_maintext}, \ref{ass:nonvanishing_maintext} hold.
	If $W \in M_b \cap M \backslash \mathrm{Sing}(M)$, then there exists a neighborhood $U_{W} \subseteq \mathcal{P}$ of $W$ such that:
	\begin{itemize}[topsep=2pt,itemsep=2pt,partopsep=2pt,parsep=2pt,leftmargin=0pt]
		\item[(a)] for any $W^{\prime} \in U_{W} \cap M$, $\ker \nabla^2 \mathcal{I}( W^{\prime}) = \mathrm{T}_{W^{\prime}} M$;
		\item[(b)] $U_{W} \cap M$ is a locally nondegenerate manifold; and
		\item[(c)] for any $W^{\prime} \in U_{W} \cap M$,
			\begin{equation}
				\min_{
					\substack{
						\norm{v} = 1 \\ 
						v \in \mathrm{T}^{\perp}_{W^{\prime}} M
					}
				} 
				v^T \nabla^2 \mathcal{I}(W^{\prime}) v 
				= \omega_{W^{\prime}} 
				> 0.
				\label{eqn:Lower_bound_on_vTHessianv_in_a_neighborhood}
			\end{equation}
	\end{itemize}
\end{proposition}

\refProposition{prop:neighborhood_exists} covers only nonsingular points in $M_b \cap M$. We will now extend the results to $M$ in the generic sense. By using the action $\pi$ from \eqref{eqn:Action_pi} we can show that if a point $W \in M_b$ is nonsingular, so is $\pi(C)(W) \in M$ for any $C \in H$. The action on $M_b$ generates $M$ and provided we have \refAssumption{ass:nonvanishing_maintext}, then by \refProposition{prop:Mb_is_almost_everywhere_nonsingular} $M_b$ is nonsingular for generic points and moreover nondegenerate by \refProposition{prop:neighborhood_exists}. We prove the following in \refAppendixSection{sec:Proof_of_proposition_nondeg_almost_everywhere}.

\begin{proposition}
	If Assumptions~\ref{ass:eigenvalues_different_maintext}, \ref{ass:nonvanishing_maintext} hold, then the set $M$ is a nondegenerate manifold for generic points.
	\label{prop:nondeg_almost_everywhere}
\end{proposition}

We \marginnote{\emph{Step 6.}} are now in position to prove \refTheorem{thm:main_proposition_maintext} and \refProposition{prop:convergence_final_case_one} by applying \refProposition{prop:lambda_cotangent_space_is_enough}. \refProposition{prop:nondeg_almost_everywhere} yields that $M$ is a nondegenerate manifold for generic points. Together with \refProposition{prop:neighborhood_exists}, this implies that up to an algebraic set of lower dimension than the dimension of $M$, for each $W \in M$ there exist a neighborhood $U_W \subseteq \mathcal{P}$ such that for any $W' \in U_W \cap M$ there exists a constant $\omega_{W'}$ so that \eqref{eqn:Lower_bound_on_vTHessianv_in_a_neighborhood} holds. Hence, setting $\lambda = \min_{W^{\prime} \in U_{W} \cap M} \omega_{W^{\prime}}$, we obtain a lower bound to \eqref{eqn:lambda_cotangent_space}. We obtain then a proof of convergence close to $M$. If moreover $W \in M_b$, then \refProposition{prop:Lower_bound_to_the_Hessian_restricted_to_directions_normal_to_the_manifold_of_minima}'s lower bound \eqref{eqn:lower_bound_hessian_minimum} to \eqref{eqn:lambda_cotangent_space} proves that if $U_W$ is small enough for $W \in M_b$, the bound \eqref{eqn:convergence_rate_omega} in \refTheorem{thm:main_proposition_maintext} holds by continuity.

In case $\rho = 1$, \refProposition{prop:Lower_bound_to_the_Hessian_restricted_to_directions_normal_to_the_manifold_of_minima}'s lower bound \eqref{eqn:lower_bound_hessian_minimum_case_1} to \eqref{eqn:lambda_cotangent_space} proves \eqref{eqn:convergence_rate_omega__case_e_1} in \refProposition{prop:convergence_final_case_one}.

This concludes the proof.
\qed

\section{Numerics}
\label{sec:Numerical_experiments}

In this section, we implement the gradient descent algorithm
\begin{equation}
  \itrd{W}{t+1}
  = \itrd{W}{t} - \eta \nabla \mathcal{J}( \itrd{W}{t} )
  \label{eqn:Gradient_descent},
\end{equation}
numerically\footnote{The source code of our implementation is available at \url{https://gitlab.tue.nl/20061069/asymptotic-convergence-rate-of-dropout-on-shallow-linear-neural-networks}.}, and apply it to \emph{Dropout}'s objective function in \eqref{eqn:dropout_regular}. We measure the convergence rate of gradient descent for different widths $f$ and dropout probabilities $1-p$ and conduct a comparison of these measurements to our bound on the convergence rate in \eqref{eqn:convergence_rate_Dropconnect_and_Dropout}. 

Related experimental results for the convergence of \emph{Dropout} can be found in \cite{wei2020implicit}. The number of iterations required for convergence, as well as the dependency of the performance on the initialization, have both been experimentally studied for the linear \gls{NN} case \cite{arora2018convergence} as well as for \emph{Dropout} \cite{mianjy2018implicit}.

\subsection{Setup}

\noindent
\emph{Data set.}
We choose a data set from the \emph{UCI Machine Learning Repository}.\footnote{The repository is located at \url{https://archive.ics.uci.edu/}.} We work with a data set that describes the critical temperature of superconductors \cite{HAMIDIEH2018346} with an input dimension $h = 81$ and the output dimension $e = 1$ (i.e., a value for the critical temperature). After first whitening and then normalizing the data, we obtain a matrix $Y \in \R^{1 \times 80}$ that satisfies $\pnorm{Y}{\mathrm{F}} = 1$. This matrix is used in the risk function in  \eqref{eqn:dropout_regular}. Note that \refAssumption{ass:eigenvalues_different_maintext} holds since $\mathrm{rk}(Y) = 1$, and by \refProposition{prop:convergence_final_case_one} \refAssumption{ass:nonvanishing_maintext} also holds.

\noindent
\emph{Stopping criteria.}
In all experiments, we stop the gradient descent algorithm in \eqref{eqn:Gradient_descent} either when the Frobenius norm of the gradient $\pnorm{ \nabla \mathcal{J}(\itrd{W}{t})}{\mathrm{F}}$ dives below a lower bound, or when it reaches a maximum number of iterations. Concretely, we let $T = \inf_{t} \{ t: \pnorm{ \nabla \mathcal{J}(\itrd{W}{t})}{\mathrm{F}} < 10^{-5} \} \wedge T_{\max}$ with $T_{\max} = 10^{6}/2$ be the random termination time of any one run of the gradient descent algorithm.

\noindent
\emph{Initialization.}
In each experiment we set the initial weights $\itrd{W}{0}$ according to one of two methods. The first method we will call \emph{Gaussian initialization}: we set every weight $W_{ijk} \sim \mathrm{Normal}(0,\sigma^2)$ in an independent, identically distributed manner. The second method we will call \emph{$\epsilon$-initialization}: we compute the set of balanced points $M_b$ in \eqref{eqn:Definition_Mb} explicitly, choose a point $W^* \sim \mathrm{Unif}(M_b)$ from it uniformly at random, and then set every weight $W_{ijk} \sim \mathrm{Normal}( W_{ijk}^*, \epsilon^2 )$ in an independent fashion.

\noindent
\emph{Step size.}
In each experiment, the step size is kept fixed and chosen $\eta =10^{-2}$.

\subsection{Results}

\refFigure{fig:experiment_plots} shows convergence rate fit results for different parameters $p, f$ and the two different initialization methods with different values of $\sigma$ and $\epsilon$. Our fitting procedure was a two-step procedure that worked as follows. 

\noindent
\emph{Step 1.}
For various $f \in \mathcal{F} \subset \naturalNumbersPlus$, $p \in \mathcal{P} \subset [0,1]$, we ran gradient descent as explained above. If the run terminated at a time $T < T_{\max}$, then we fitted the model
\begin{equation}
  G(t; a, \beta_{f,p}) 
  = a \e{- \beta_{f,p} t}
  \quad
  \textnormal{to the points}
  \quad
  \bigl\{
  (
  t,
  \pnorm{ \nabla \mathcal{J}(\itrd{W}{t}) }{\mathrm{F}}
  ) : 
  t = \lfloor \gamma T \rfloor, \ldots, T \bigr\}.
  \label{eqn:Data_points}
\end{equation}
Here, $\gamma \in [0,1)$. In this way, we obtain an estimate $\hat{\beta}_{f,p}$ for the parameter $\beta_{f,p}$ with which the model in \eqref{eqn:Data_points} best describes the measured convergence rate. Note that the estimate $\hat{\beta}_{f,p}$ is random because of our initialization. By conducting independent runs, we obtain a set of sample averages 
$
  \{ \langle \hat{\beta}_{f,p} \rangle \}_{ f \in \mathcal{F}, p \in \mathcal{P} }
$. If the fit did not result in a positive estimate $\hat{\beta}_{f,p} > 0$, then this estimate was discarded. This eliminates runs that pass close to a saddle point.

\noindent
\emph{Step 2.}
To obtain \refFigure{fig:experiment_plots}a we fixed $f \in \naturalNumbersPlus$ and then fitted the model
\begin{equation}
  \beta_f( p; b, \alpha )
  = \frac{b p}{f (\frac{p}{1-p})^{\alpha} + 1}
  \quad
  \textnormal{to the points}
  \quad
  \bigl\{
    (
      p
      ,
      \langle \hat{\beta}_{f,p} \rangle
    )
  \bigr\}_{ p \in \mathcal{P} }.
  \label{eqn:Model_fixed_f_varied_p}  
\end{equation}
This gives estimates $\hat{b}, \hat{\alpha}$ for the parameters $b, \alpha$ with which the model in \eqref{eqn:Model_fixed_f_varied_p} best describes the sample average convergence rate. To obtain Figures~\ref{fig:experiment_plots}b,c, we fixed $p \in [0,1]$ and then fitted the model
\begin{equation}
  \beta_p(f;b,c,\alpha)
  =
  \frac{bp(1-p)}{pf^{\alpha} + 1-p} + c
  \quad
  \textnormal{to the points}
  \quad
  \bigl\{
    (
      f
      ,
      \langle \hat{\beta}_{f,p} \rangle
    )
  \bigr\}_{ f \in \mathcal{F} }.  
  \label{eqn:Model_fixed_p_varied_f}  
\end{equation}
This similarly gives estimates $\hat{b}, \hat{c}, \hat{\alpha}$ for the best model parameters $b, c, \alpha$ in \eqref{eqn:Model_fixed_p_varied_f}. 

Note after substituting \eqref{eqn:Model_fixed_f_varied_p} or \eqref{eqn:Model_fixed_p_varied_f} into \eqref{eqn:Data_points}, that both exponents have an extra factor $p$ when compared to our bound in \refProposition{prop:convergence_final_case_one}. This is because we implemented the objective function $\mathcal{J}(W)$ in \eqref{eqn:dropout_regular} as opposed to $\mathcal{I}(W)$ in \eqref{eqn:loss_whitened_scaled_dropout}. Furthermore, note that if our bound in \refProposition{prop:convergence_final_case_one} turns out to be sufficiently sharp, that we can then expect that $\hat{\alpha} \approx 1$ in either model.

\begin{figure}
  \captionsetup[subfigure]{justification=centering}
  \centering
  \begin{subfigure}{0.45\columnwidth}
    \centering
    \includegraphics[width=\linewidth]{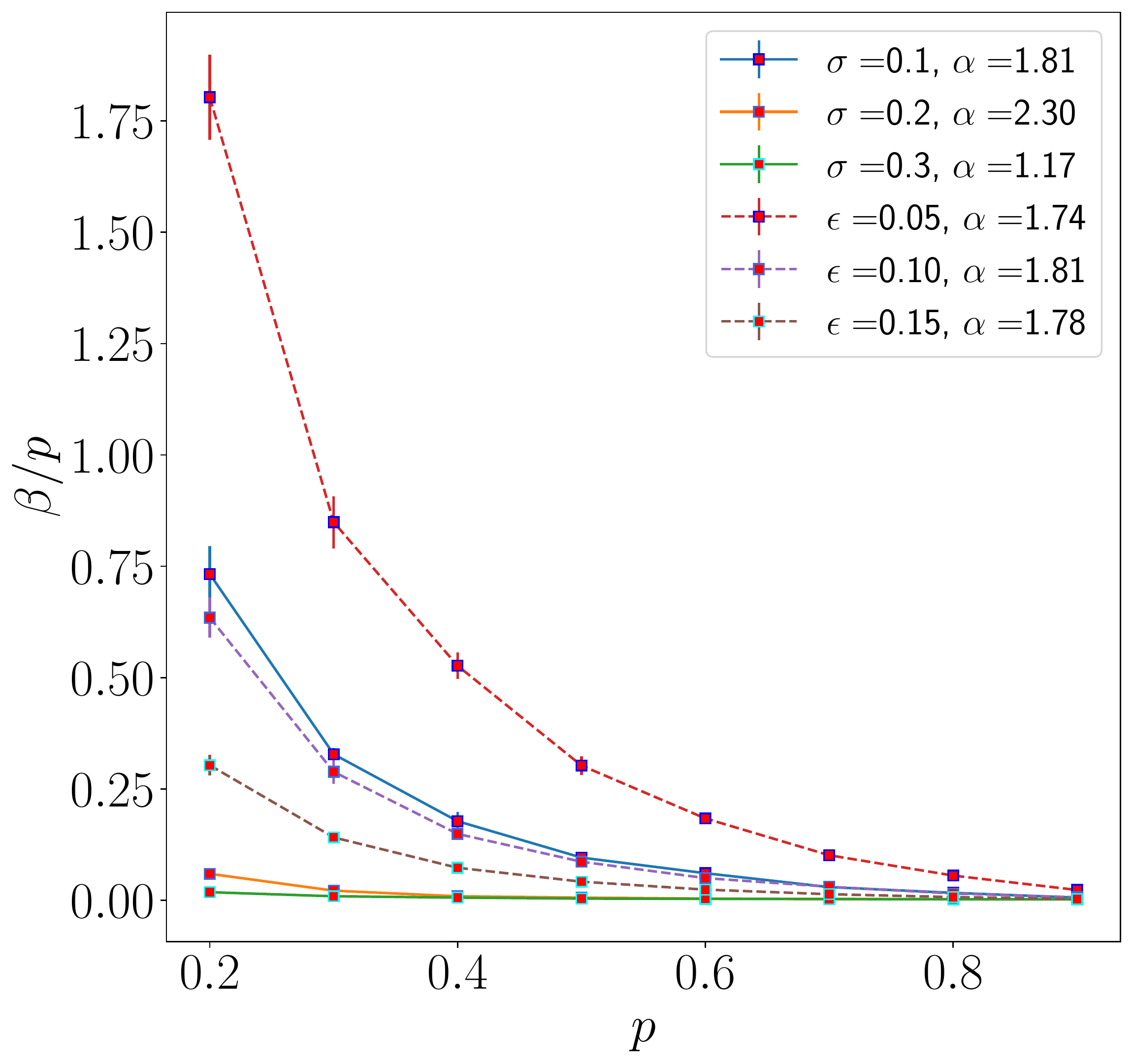}
    \caption{}
    \label{fig:comparison_p}
  \end{subfigure}
  \hfill
  \begin{subfigure}{0.45\columnwidth}
    \centering
    \includegraphics[width=\linewidth]{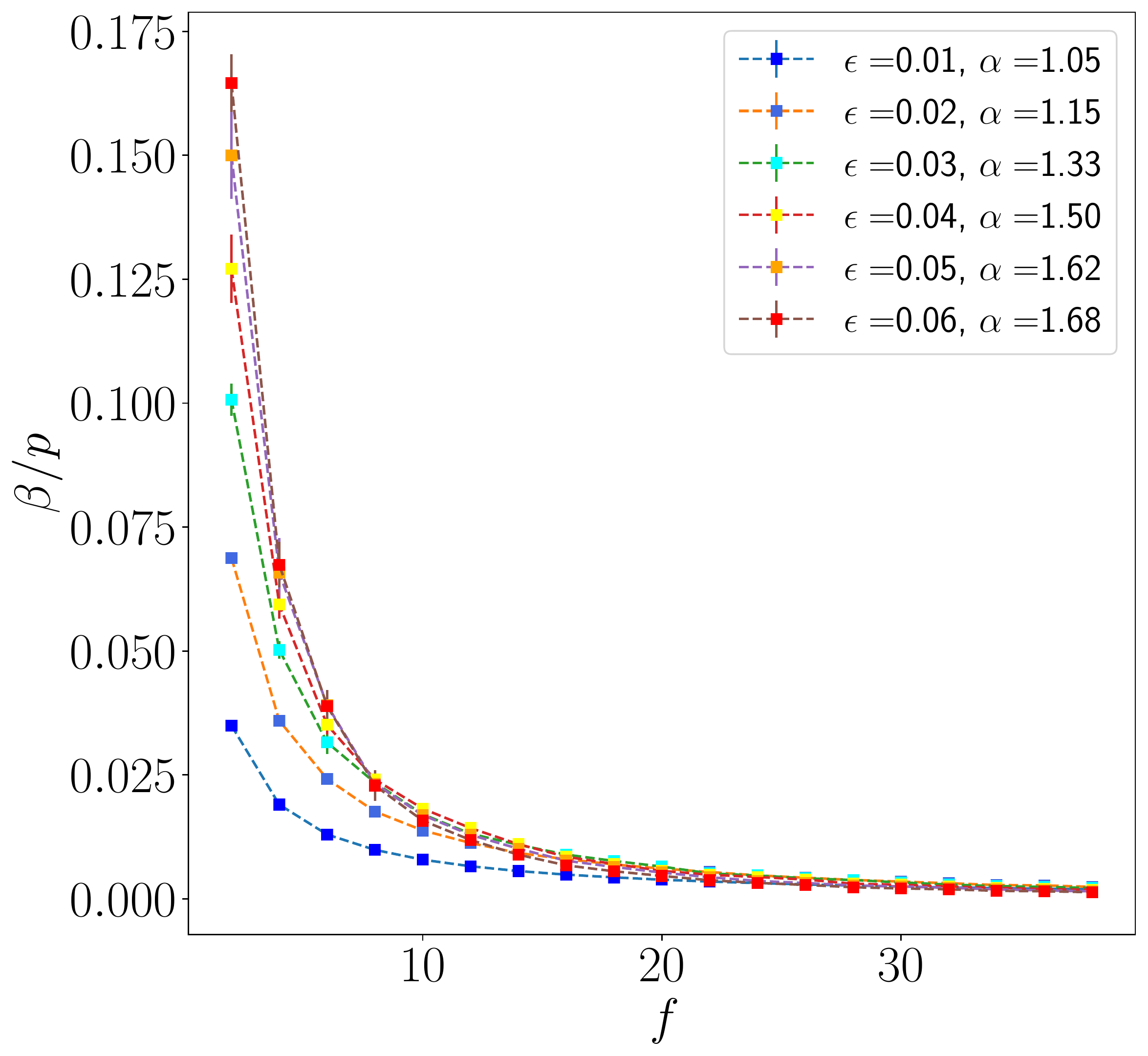}
    \caption{}
    \label{fig:comparison_f}
  \end{subfigure}
  \hfill
  \begin{subfigure}{0.45\columnwidth}
    \centering
    \includegraphics[width=\linewidth]{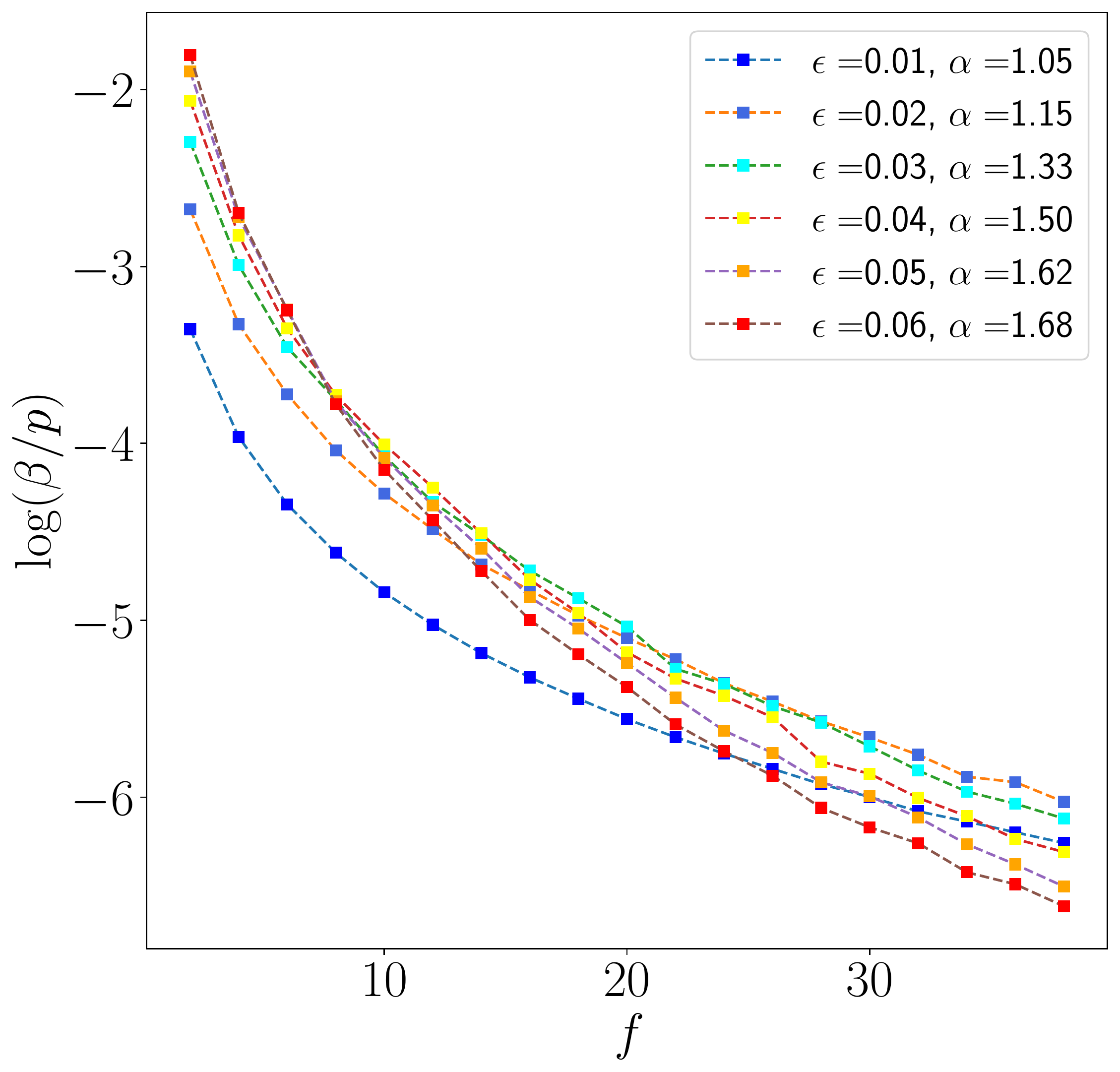}
    \caption{}
    \label{fig:comparison_f2}
  \end{subfigure}
  \hfill
  \begin{subfigure}{0.45\columnwidth}
    \centering
    \includegraphics[width=\linewidth]{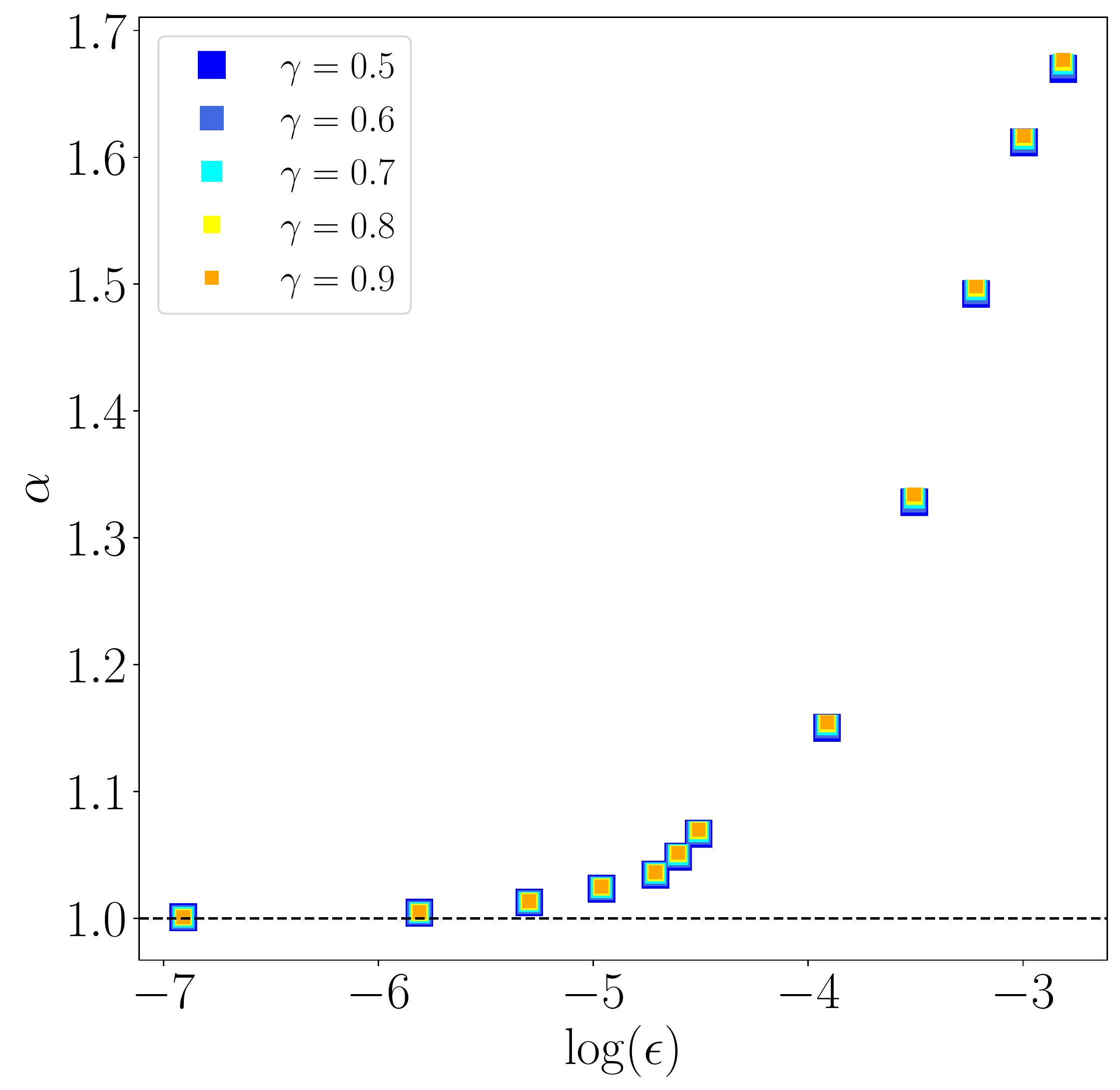}
    \caption{}
    \label{fig:comparison_alpha}
  \end{subfigure}

  \caption{(a) Sample average convergence rate as a function of $p$ for fixed $f = 20$. Here, $\gamma = 0.9$. Our fits of \eqref{eqn:Model_fixed_f_varied_p} are also shown, and the inset shows the chosen $\sigma, \epsilon$ as well as the resulting fit parameters $\hat{\alpha}$. Observe that the sample averages are decreasing in $p$, just as our bound in \eqref{eqn:convergence_rate_Dropconnect_and_Dropout}. (b) Sample average convergence rate as a function of $f$ for fixed $p = 0.7$. The fits of \eqref{eqn:Model_fixed_p_varied_f} are again shown and the inset gives the resulting fit parameters. Recall that by \eqref{eqn:convergence_rate_Dropconnect_and_Dropout}, we may expect for sufficiently small $\epsilon$ that $\beta/p \sim 1/f$ as $f \to \infty$ and consequently $\hat{\alpha} \approx 1$. This is confirmed by the different values of $\hat{\alpha}$ shown in the inset. Observe also that for large $f$, the convergence rate appears to become independent of $f$ but remains positive. (c) A zoomed in variant of (b) obtained by plotting on a logarithmic scale; note in particular that for larger $\epsilon$ we obtain a shift to lower $\beta/p$ as $f$ becomes larger. (d) All resulting fit parameters $\hat{\alpha}$ as a function of $\log \epsilon$ for fixed $p = 0.7$ and various $\gamma$. Observe that indeed $\hat{\alpha} \downarrow 1$ as $\epsilon \downarrow 0$, as predicted by our bound in \eqref{eqn:convergence_rate_Dropconnect_and_Dropout}. This shows that the bound is sharp in this regime.}
  \label{fig:experiment_plots}
\end{figure}

\subsection{Discussion}

\refFigure{fig:experiment_plots} shows that the local bound in \eqref{eqn:convergence_rate_Dropconnect_and_Dropout} characterizes the convergence rate of gradient descent close to convergence qualitatively. The characteristic decay of the convergence rate as $f$ and $p/(1-p)$ increase, as predicted by \eqref{eqn:convergence_rate_Dropconnect_and_Dropout}, is confirmed experimentally with the fits in \refFigure{fig:comparison_p}, \ref{fig:comparison_f}, respectively.
Observe that $\hat{\beta}$ and $\hat{\alpha}$ depend strongly on the initialization, which is parameterized here by $\sigma$ and $\epsilon$. This may be explained as follows. For larger $\sigma, \epsilon$, we initialize farther away from minima. Trajectories of gradient descent that follow valleys of the objective function $\mathcal{J}(W)$---regions where the loss is close to the minima and where slower convergence rates are expected---are then favored. Hence, we can then expect to see a \emph{smaller} sample average convergence rate.

Similarly, resulting fit values for $\hat{\alpha}$ that are greater than one may also explained by a bias induced by the random initialization. Observe in particular for the cases of $\epsilon$-close initialization in \refFigure{fig:comparison_f2} that while for $f \lesssim 24$ a larger $\epsilon$ yields a larger $\hat{\beta}/p$, the contrary occurs for $f \gtrsim 24$. If $f$ is small, then the number of valleys of the objective function $\mathcal{J}(W)$ is also small; consequently, initializations that are further away from minima then favor gradient descent trajectories that have higher convergence rates. As  $f$ increases, the number of valleys also increases; initializations that are further away from minima then favor gradient descent trajectories that tend to get stuck in valleys for longer times. Note that it is not clear whether this explanation also works for the cases of Gaussian initialization, because there we do not guarantee that we initialize close to a minima.

These various decay rates that were measured were used to calculate fit parameters $\hat{\alpha}$. These tend to one as $\epsilon$ becomes smaller as shown in \refFigure{fig:comparison_alpha}, which is in accordance with \refTheorem{thm:main_proposition_maintext}. Note that while the values of $\hat{\alpha}$ tend to $1$ as $\epsilon \to 0$ independently of $\gamma$, there is a shift to larger values of $\hat{\alpha}$ as we increase $\epsilon$. This is because we are seeing an average convergence rate of the trajectories and not the bound in \eqref{eqn:convergence_rate_Dropconnect_and_Dropout}.

Finally, \refFigure{fig:experiment_plots} indicates that independently of the regularization properties of \emph{Dropout} and the inherent scaling of $1/p$ in the number of iterations \cite{senen2020almost}, for $p \to 1$ or $f \to \infty$, the landscape of \emph{Dropout} close to the minimum becomes less rough and suggests that the number of valleys with approximate minima increases, thereby making finding approximate minima easier.
\section{Conclusion}
\label{sec:Conclusion}

In this paper we have analyzed the convergence rate of gradient flow on the objective functions induced by \emph{Dropout} and \emph{Dropconnect} for shallow linear \glspl{NN}. \refTheorem{thm:main_proposition_maintext} gives a lower bound on the convergence rate that depends implicitly on the data matrix $Y$, the probability of dropping nodes or edges $1-p$, and the structure parameters $e, f, h$ of the \gls{NN}. Its proof relied on the application of a state-of-the-art results for the convergence of nonconvex objective functions \cite{fehrman2019convergence}, combined with careful analyses of the set of minimizers as well as lower bounding the Hessian of the scaled dropout risk $\mathcal{I}(W)$. We also provide in \refProposition{prop:convergence_final_case_one} closed-form expression for our lower bound on the convergence rate of gradient flow in the case of a one-dimensional output for \emph{Dropout} and \emph{Dropconnect}. The bounds give insight into the dependencies of the convergence rate in dropout.

\bibliographystyle{plain}
\bibliography{biblio}

\begin{thebibliography}{10}

\bibitem{arora2018convergence}
Sanjeev Arora, Noah Golowich, Nadav Cohen, and Wei Hu.
\newblock A convergence analysis of gradient descent for deep linear neural
  networks.
\newblock In {\em 7th International Conference on Learning Representations,
  ICLR 2019}, 2019.

\bibitem{arvanitogeorgos2003introduction}
Andreas Arvanitoge{\=o}rgos.
\newblock {\em An introduction to Lie groups and the geometry of homogeneous
  spaces}, volume~22.
\newblock American Mathematical Soc., 2003.

\bibitem{ba2013adaptive}
Jimmy Ba and Brendan Frey.
\newblock Adaptive dropout for training deep neural networks.
\newblock In {\em Advances in Neural Information Processing Systems}, pages
  3084--3092, 2013.

\bibitem{bah2019learning}
Bubacarr Bah, Holger Rauhut, Ulrich Terstiege, and Michael Westdickenberg.
\newblock Learning deep linear neural networks: Riemannian gradient flows and
  convergence to global minimizers.
\newblock {\em arXiv preprint arXiv:1910.05505}, 2019.

\bibitem{baldi2013understanding}
Pierre Baldi and Peter~J. Sadowski.
\newblock Understanding {D}ropout.
\newblock In {\em Advances in Neural Information Processing Systems}, pages
  2814--2822, 2013.

\bibitem{baldi2014dropout}
Pierre Baldi and Peter~J. Sadowski.
\newblock The {D}ropout learning algorithm.
\newblock {\em Artificial Intelligence}, 210:78--122, 2014.

\bibitem{Bartlett2018GradientDW}
Peter~L. Bartlett, David~P. Helmbold, and Philip~M. Long.
\newblock Gradient descent with identity initialization efficiently learns
  positive-definite linear transformations by deep residual networks.
\newblock {\em Neural Computation}, 31:477--502, 2018.

\bibitem{bertsekas1996neuro}
Dimitri~P. Bertsekas and John~N. Tsitsiklis.
\newblock {\em Neuro-dynamic programming}.
\newblock Athena Scientific, 1996.

\bibitem{bochnak2013real}
Jacek Bochnak, Michel Coste, and Marie-Fran{\c{c}}oise Roy.
\newblock {\em Real algebraic geometry}, volume~36.
\newblock Springer Science \& Business Media, 2013.

\bibitem{borkar2009stochastic}
Vivek~S. Borkar.
\newblock {\em Stochastic approximation: a dynamical systems viewpoint},
  volume~48.
\newblock Springer, 2009.

\bibitem{cavazza2017dropout}
Jacopo Cavazza, Pietro Morerio, Benjamin Haeffele, Connor Lane, Vittorio
  Murino, and Rene Vidal.
\newblock Dropout as a low-rank regularizer for matrix factorization.
\newblock In {\em International Conference on Artificial Intelligence and
  Statistics}, pages 435--444. PMLR, 2018.

\bibitem{devries2017improved}
Terrance DeVries and Graham~W. Taylor.
\newblock Improved regularization of convolutional neural networks with cutout.
\newblock {\em arXiv preprint arXiv:1708.04552}, 2017.

\bibitem{fehrman2019convergence}
Benjamin Fehrman, Benjamin Gess, and Arnulf Jentzen.
\newblock Convergence rates for the stochastic gradient descent method for
  non-convex objective functions.
\newblock {\em Journal of Machine Learning Research}, 21, 2020.

\bibitem{HAMIDIEH2018346}
Kam Hamidieh.
\newblock A data-driven statistical model for predicting the critical
  temperature of a superconductor.
\newblock {\em Computational Materials Science}, 154:346 -- 354, 2018.

\bibitem{hinton2012improving}
Geoffrey~E. Hinton, Nitish Srivastava, Alex Krizhevsky, Ilya Sutskever, and
  Ruslan~R. Salakhutdinov.
\newblock Improving neural networks by preventing co-adaptation of feature
  detectors.
\newblock {\em arXiv preprint arXiv:1207.0580}, 2012.

\bibitem{kay2016Dropconnected}
Edmund Kay and Anurag Agarwal.
\newblock Dropconnected neural network trained with diverse features for
  classifying heart sounds.
\newblock In {\em 2016 Computing in Cardiology Conference (CinC)}, pages
  617--620. IEEE, 2016.

\bibitem{kingma2015variational}
Durk~P Kingma, Tim Salimans, and Max Welling.
\newblock Variational dropout and the local reparameterization trick.
\newblock In {\em Advances in Neural Information Processing Systems}, pages
  2575--2583, 2015.

\bibitem{krizhevsky2012imagenet}
Alex Krizhevsky, Ilya Sutskever, and Geoffrey~E. Hinton.
\newblock Image{N}et classification with deep convolutional neural networks.
\newblock In {\em Advances in Neural Information Processing Systems}, pages
  1097--1105, 2012.

\bibitem{kushner2003stochastic}
Harold~J. Kushner and George~G. Yin.
\newblock {\em Stochastic approximation and recursive algorithms and
  applications}, volume~35.
\newblock Springer Science \& Business Media, 2003.

\bibitem{lee2013smooth}
John~M. Lee.
\newblock Smooth manifolds.
\newblock In {\em Introduction to Smooth Manifolds}, pages 1--31. Springer,
  2013.

\bibitem{li2016improved}
Zhe Li, Boqing Gong, and Tianbao Yang.
\newblock Improved dropout for shallow and deep learning.
\newblock In {\em Advances in Neural Information Processing Systems}, pages
  2523--2531, 2016.

\bibitem{marshall1979inequalities}
Albert~W. Marshall, Ingram Olkin, and Barry~C. Arnold.
\newblock {\em Inequalities: theory of majorization and its applications},
  volume 143.
\newblock Springer, 1979.

\bibitem{mianjy2019dropout}
Poorya Mianjy and Raman Arora.
\newblock On dropout and nuclear norm regularization.
\newblock In {\em International Conference on Machine Learning}, pages
  4575--4584, 2019.

\bibitem{Mianjy2020OnCA}
Poorya Mianjy and Raman Arora.
\newblock On convergence and generalization of dropout training.
\newblock {\em Advances in Neural Information Processing Systems}, 33, 2020.

\bibitem{mianjy2018implicit}
Poorya Mianjy, Raman Arora, and Rene Vidal.
\newblock On the implicit bias of dropout.
\newblock In {\em International Conference on Machine Learning}, pages
  3540--3548, 2018.

\bibitem{molchanov2017variational}
Dmitry Molchanov, Arsenii Ashukha, and Dmitry Vetrov.
\newblock Variational dropout sparsifies deep neural networks.
\newblock In {\em Proceedings of the 34th International Conference on Machine
  Learning-Volume 70}, pages 2498--2507. JMLR. org, 2017.

\bibitem{nguyen2017loss}
Quynh Nguyen and Matthias Hein.
\newblock The loss surface of deep and wide neural networks.
\newblock In {\em Proceedings of the 34th International Conference on Machine
  Learning-Volume 70}, pages 2603--2612, 2017.

\bibitem{pal2019regularization}
Ambar Pal, Connor Lane, Ren{\'e} Vidal, and Benjamin~D. Haeffele.
\newblock On the regularization properties of structured dropout.
\newblock In {\em Proceedings of the IEEE/CVF Conference on Computer Vision and
  Pattern Recognition}, pages 7671--7679, 2020.

\bibitem{pham2014dropout}
Vu~Pham, Th{\'e}odore Bluche, Christopher Kermorvant, and J{\'e}r{\^o}me
  Louradour.
\newblock Dropout improves recurrent neural networks for handwriting
  recognition.
\newblock In {\em 2014 14th International Conference on Frontiers in
  Handwriting Recognition}, pages 285--290. IEEE, 2014.

\bibitem{semeniuta2016recurrent}
Stanislau Semeniuta, Aliaksei Severyn, and Erhardt Barth.
\newblock Recurrent dropout without memory loss.
\newblock In {\em Proceedings of COLING 2016, the 26th International Conference
  on Computational Linguistics: Technical Papers}, pages 1757--1766, 2016.

\bibitem{senen2020almost}
Albert Senen-Cerda and Jaron Sanders.
\newblock Almost sure convergence of dropout algorithms for neural networks.
\newblock {\em arXiv preprint arXiv:2002.02247}, 2020.

\bibitem{smith2004invitation}
Karen Smith, Lauri Kahanp{\"a}{\"a}, Pekka Kek{\"a}l{\"a}inen, and William
  Traves.
\newblock {\em An invitation to algebraic geometry}.
\newblock Springer Science \& Business Media, 2004.

\bibitem{srivastava2014dropout}
Nitish Srivastava, Geoffrey~E. Hinton, Alex Krizhevsky, Ilya Sutskever, and
  Ruslan Salakhutdinov.
\newblock Dropout: a simple way to prevent neural networks from overfitting.
\newblock {\em The Journal of Machine Learning Research}, 15(1):1929--1958,
  2014.

\bibitem{urban2018deep}
Gregor Urban, Kevin Bache, Duc~T.T. Phan, Agua Sobrino, Alexander~K. Shmakov,
  Stephanie~J. Hachey, Christopher~C.W. Hughes, and Pierre Baldi.
\newblock Deep learning for drug discovery and cancer research: Automated
  analysis of vascularization images.
\newblock {\em IEEE/ACM transactions on computational biology and
  bioinformatics}, 16(3):1029--1035, 2018.

\bibitem{wager2013dropout}
Stefan Wager, Sida Wang, and Percy~S. Liang.
\newblock Dropout training as adaptive regularization.
\newblock In {\em Advances in Neural Information Processing Systems}, pages
  351--359, 2013.

\bibitem{wan2013regularization}
Li~Wan, Matthew Zeiler, Sixin Zhang, Yann Le~Cun, and Rob Fergus.
\newblock Regularization of neural networks using dropconnect.
\newblock In {\em International Conference on Machine Learning}, pages
  1058--1066, 2013.

\bibitem{wei2020implicit}
Colin Wei, Sham Kakade, and Tengyu Ma.
\newblock The implicit and explicit regularization effects of dropout.
\newblock {\em arXiv preprint arXiv:2002.12915}, 2020.

\bibitem{zaremba2014recurrent}
Wojciech Zaremba, Ilya Sutskever, and Oriol Vinyals.
\newblock Recurrent neural network regularization.
\newblock {\em arXiv preprint arXiv:1409.2329}, 2014.

\end{thebibliography}

\appendix

\onecolumn

\newpage

\section*{Appendix}
\label{secappendix:main_proofs}

For any vector $a=(a_1, \cdots, a_f)$, we denote by $\mathrm{Diag}(a_1, \cdots, a_f)$ the matrix in $\R^{f \times f}$ with the vector $a$ in the diagonal and zeroes everywhere else. For a matrix $A \in \R^{f \times f}$, we denote $\mathrm{Diag}(A) = \mathrm{Diag}(A_{11}, \ldots, A_{ff})$. For a matrix $A$ with singular values $\lambda_1, \ldots, \lambda_r$ we denote the 1-norm as $\pnorm{A}{1} = \sum_{i=1}^r \lambda_i$. 

\section{On the assumptions}
\label{secappendix:on_the_assumptions}

In order to establish \refTheorem{thm:main_proposition_maintext} and \refProposition{prop:convergence_final_case_one} and for these results to be applicable in a range of scenarios, we opted for Assumptions \ref{ass:eigenvalues_different_maintext},~\ref{ass:nonvanishing_maintext}. These technical assumptions are sufficient for our proofs and in fact allow for fairly generic data matrices $Y$.

To see this, consider first that the subset of matrices of a fixed rank that do not satisfy \refAssumption{ass:eigenvalues_different_maintext} has measure zero. Hence if the data is for example drawn randomly from a data distribution with continuous support, then the assumption will hold with high probability. Consider second that \refAssumption{ass:nonvanishing_maintext} is not very restrictive either: for example, if there is a real Hadamard orthogonal matrix $S \in \mathrm{O}(f)$ in dimension $f$, that is, $|S_{ij}|^2 = 1/f$ for all $i,j$, then we can use this $S$ as an example that satisfies \refAssumption{ass:nonvanishing_maintext} independently of the eigenvalues $\sigma_1, \ldots, \sigma_r$. Furthermore like before, the measure of the subset of $S \in \mathrm{O}(f)$ not satisfying \refAssumption{ass:nonvanishing_maintext} is zero. 

Additionally, while \refAssumption{ass:nonvanishing_maintext} represents a sufficient condition for \refLemma{lemma:assumption_holds_if_nonzero} to hold, we can give the following heuristic argument as to why \refLemma{lemma:assumption_holds_if_nonzero} may likely hold independently of \refAssumption{ass:nonvanishing_maintext}: 
Observe from the definition of $\bar{M}_b$, recall \eqref{eqn:Alternative_representation_of_Mb} and \eqref{eqn:Definition_Mbbar}, that we need $f-1$ satisfied constraints on $S \in \bar{M}_b$ in order for it to belong to $M_b$. Here $S$ is to be understood as an equivalence class. Also, in order for $M_b$ to be a manifold, we at least need that the number of constraints is not larger than the dimension of the base space $\bar{M}_b$. Recall now that we have an explicit parametrization of $T_{W} \bar{M}_b$ in \eqref{eqn:Tangent_space_at_W_of_Mbbar}. Hence, the dimensional constraints require for $\dim(X) + \dim(E) = 1/2(\rho^{2} - \rho) + \rho( f - \rho) = 1/2\rho + (f - 1)\rho - (1/2)\rho^2 \geq f-1$ to hold. Note this equality as a function of $\rho$ is achieved for $\rho = 1$. On the other hand, if $\rho = f$, then the inequality reads $1/2(f-1)f \geq f-1$. This inequality holds for $f \geq 2$. Summarizing, these inequalities hold for any $\rho \leq f$ and $f \geq 2$, and we may therefore expect $\mathrm{rk}(\mathrm{D}_{W} T) = f-1$ to hold for some $W \in M_b$.

\section{Proofs of \refSection{sec:Preliminaries}}

\subsection{Proof \texorpdfstring{of \eqref{eqn:Data_whitened_objective_function}}{} -- Data whitening} 
\label{secappendix:data_whitening}

For simplicity, let $\mathcal{W} = W_2 W_1$. Recall also that $Y = \mathcal{Y} \mathcal{X}^{\mathrm{T}} (\mathcal{X} \mathcal{X}^{\mathrm{T}})^{-1/2}$. We now apply the identity
\begin{equation}
	\pnorm{A - B}{\mathrm{F}}^2
	= \pnorm{A}{\mathrm{F}}^2 - 2 \langle A, B \rangle_{\mathrm{F}} + \pnorm{B}{\mathrm{F}}^2
	= \mathrm{Tr}[AA^{\mathrm{T}}] - 2 \mathrm{Tr}[A^{\mathrm{T}}B] + \mathrm{Tr}[BB^{\mathrm{T}}]
	\label{eqn:Frobenius_norm_identity}
\end{equation}
twice, to obtain
\begin{align}
	&
	\mathcal{R}(W)
	= \norm{\mathcal{Y} - \mathcal{W}\mathcal{X}}^2_{\mathrm{F}}
	\eqcom{\ref{eqn:Frobenius_norm_identity}}= \mathrm{Tr}[ \mathcal{Y}\mathcal{Y}^{\mathrm{T}} ] - 2 \mathrm{Tr}[ \mathcal{Y}\mathcal{X}^{\mathrm{T}}\mathcal{W} ] + \mathrm{Tr}[ \mathcal{W}\mathcal{X}\mathcal{X}^{\mathrm{T}}\mathcal{W}^{\mathrm{T}} ]
	\nonumber \\ &
	= \mathrm{Tr}[\mathcal{Y}\mathcal{Y}^{\mathrm{T}}] - 2 \mathrm{Tr}[ \mathcal{Y}\mathcal{X}^{\mathrm{T}}(\mathcal{X}\mathcal{X}^{\mathrm{T}})^{-1/2}(\mathcal{W}(\mathcal{X}\mathcal{X}^{\mathrm{T}})^{1/2})^{\mathrm{T}} ]
	+ \mathrm{Tr}[ (\mathcal{W}(\mathcal{X}\mathcal{X}^{\mathrm{T}})^{1/2})(\mathcal{W}(\mathcal{X}\mathcal{X}^{\mathrm{T}})^{1/2})^{\mathrm{T}} ]
	\nonumber \\ &
	=
	\mathrm{Tr}[YY^{\mathrm{T}}] - 2 \mathrm{Tr}[ Y (\mathcal{W}(\mathcal{X}\mathcal{X}^{\mathrm{T}})^{1/2})^{\mathrm{T}} ]
	\nonumber \\ &
	\phantom{==}
	+ \mathrm{Tr}[ (\mathcal{W}(\mathcal{X}\mathcal{X}^{\mathrm{T}})^{1/2})(\mathcal{W}(\mathcal{X}\mathcal{X}^{\mathrm{T}})^{1/2})^{\mathrm{T}} ]
	+ \mathrm{Tr}[\mathcal{Y}\mathcal{Y}^{\mathrm{T}}] - \mathrm{Tr}[YY^{\mathrm{T}}]
	\nonumber \\ &
	\eqcom{\ref{eqn:Frobenius_norm_identity}}= \pnorm{Y - \mathcal{W}(\mathcal{X}\mathcal{X}^{\mathrm{T}})^{1/2} }{\mathrm{F}} + \mathrm{Tr}[ \mathcal{Y}\mathcal{Y}^{\mathrm{T}} ] - \mathrm{Tr}[ YY^{\mathrm{T}} ].
\end{align}

\subsection{Proof of \refLemma{lemma:dropout_loss_explicit_with_p}}
\label{secappendix:dropout_J}

\noindent
\emph{Proof of \eqref{eqn:dropout_regular}.}  Note that
\eqref{eqn:dropout_regular} is a known expression for \emph{Dropout} in literature. In particular, consult \cite[Eq.~(10)]{cavazza2017dropout}, \cite[Eq.~(10)]{mianjy2018implicit} and \cite[Lemma~A.1]{mianjy2019dropout}.

\noindent
\emph{Proof of \eqref{eqn:dropconnect}.}
When using \emph{Dropconnect}, we have for $i \in \{1,2\}$ that each matrix element $F_{ijk} \sim \mathrm{Ber}(p)$ is independent and identically distributed as indicated. We find that
\begin{align}
	\mathcal{J}(W)
	&
	= \expectation{ \pnorm{Y - (W_2 \odot F_2) (W_1 \odot F_1)}{\mathrm{F}}^2 } \nonumber \\ &
	= \expectationBig{ \pnorm{Y - p^2 W_2 W_1 + p^2 W_2 W_1 - (W_2 \odot F_2) (W_1 \odot F_1) }{\mathrm{F}}^2 }
	\nonumber \\ &
	= \expectationBig{ \pnorm{Y - p^2 W_2 W_1}{\mathrm{F}}^2 + \pnorm{ p^2 W_2 W_1 - (W_2 \odot F_2) (W_1 \odot F_1) }{\mathrm{F}}^2
	\nonumber \\ &
	\phantom{= \mathbb{E}\Bigl[}+ 2\sum_{ij} \bigl( ( Y - p^2 W_2 W_1 )^{\mathrm{T}}( p^2 W_2 W_1 -(W_2 \odot F_2) (W_1 \odot F_1) \bigl)_{ij} }.
	\label{eqn:bias_variance_Dropconnect}
\end{align}
Note that $\expectationWrt{ (W_2 \odot F_2) (W_1 \odot F_1) }{} = p^2 W_2 W_1$, so the right-most term equals zero. Furthermore, we can expand
\begin{align}
	\mathbb{E}\Bigl[ \pnorm{ p^2 W_2 W_1 - (W_2 \odot F_2) (W_1 \odot F_1) }{\mathrm{F}}^2 \Bigr] & = \pnorm{ p^2 W_2 W_1}{\mathrm{F}}^2 + \expectationWrt{\pnorm{(W_2 \odot F_2) (W_1 \odot F_1) }{\mathrm{F}}^2}{} \nonumber \\ &
	- 2 \mathbb{E}\Bigl[ \sum_{ij} ( (p^2 W_2 W_1)^{\mathrm{T}}(W_2 \odot F_2 W_1 \odot F_1) )_{ij}  \Bigr]
	.
	\label{eqn:bias_variance_Dropconnect2}
\end{align}
After now (i) substituting \eqref{eqn:bias_variance_Dropconnect2} into \eqref{eqn:bias_variance_Dropconnect} and rearranging terms, and then (ii) writing out the Frobenius norm, it follows that
\begin{align}
	&
	\mathcal{J}(W) - \pnorm{Y - p^2 W_2W_1}{\mathrm{F}}^2 
	\eqcom{i}= 
	\expectation{ \norm{(W_2 \odot F_2)  (W_1 \odot F_1)}_F^2 }
	\nonumber \\ &
	= \expectationBig{ \expectationBig{ \pnorm{(W_2 \odot F_2) (W_1 \odot F_1)}{\mathrm{F}}^2 \Big\vert F_1 } }
	\eqcom{ii}= \expectationBig{ \expectationBig{ \sum_{a,b} \Bigl( \sum_{i} W_{2ai} F_{2ai} W_{1ib} F_{1ib} \Bigr)^2 \Big\vert F_1 } }
	.
\end{align}
Use (iii) the fact that $( \sum a_i b_i )^2 = \sum_i a_i^2 b_i^2 + \sum_{i \neq j} a_i b_i a_j b_j$ now twice, to conclude that
\begin{align}
	&
	\mathcal{J}(W) - \pnorm{Y - p^2 W_2W_1}{\mathrm{F}}^2 	
	\nonumber \\ &
	\eqcom{iii}= 
	\expectationBig{ 
		\sum_{a,b} 
		\Bigl( 
			(p - p^2) \sum_{i} W_{2ai}^2 W_{1ib}^2 F_{1ib}^2 + p^2 \bigl( \sum_{i} W_{2ai} W_{1ib} F_{1ib} \bigr)^2 
		\Bigr)
	}
	\nonumber \\ &
	= 
	\sum_{a,b} 
	\Bigl(
		p (p - p^2)  \sum_{i} W_{2ai}^2 W_{1ib}^2 + p^2 \Bigl( (p - p^2) \sum_{i} W_{2ai}^2 W_{1ib}^2  + p^2 \bigl( \sum_i W_{2ai} W_{1ib} \bigr)^2 \Bigr)
	\Bigr)
	\nonumber \\ &
	= 
	(p^2 - p^4) \mathrm{Tr}\bigl[ \mathrm{Diag}(W_1W_1^{\mathrm{T}}) \mathrm{Diag}(W_2W_2^{\mathrm{T}}) \bigr] + p^4 \pnorm{W_2 W_1}{\mathrm{F}}^2.
	\label{eqn:variance_computation_Dropconnect}
\end{align}
Substituting \eqref{eqn:variance_computation_Dropconnect} into \eqref{eqn:bias_variance_Dropconnect} results in \eqref{eqn:dropconnect}. This completes the proof.

\subsection{Proof of \refLemma{lemma:first_integral_main}}
\label{secappendix:Proof_that_Mdb_equals_Mb}

Let $W(t) = (W_2(t), W_1(t))$ denote a solution to \eqref{eqn:gradient_flow_on_I}. We will now prove the following facts:
\begin{itemize}[topsep=2pt,itemsep=2pt,partopsep=2pt,parsep=2pt,leftmargin=0pt]
	\item[(i)] If $\mathrm{Diag}(W_1(0) W_1^{\mathrm{T}}(0)) = \mathrm{Diag}(W_2^{\mathrm{T}}(0) W_2(0))$, then $\mathrm{Diag}(W_1(t) W_1^{\mathrm{T}}(t)) = \mathrm{Diag}(W_2^{\mathrm{T}}(t) W_2(t))$ for any $t \geq 0$.
	\item[(ii)] If $W_1(0) W_1^{\mathrm{T}}(0) = W_2^{\mathrm{T}}(0) W_2(0)$, then
	$W_1(t) W_1^{\mathrm{T}}(t) = W_2^{\mathrm{T}}(t) W_2(t)$ for any $t \geq 0$.
	\item[(iii)] If $\mathrm{Diag}(W_1(0) W_1^{\mathrm{T}}(0)) = \mathrm{Diag}(W_2^{\mathrm{T}}(0) W_2(0))$ and $W(t)$ converges as $t \to \infty$, then also
	$
	\lim_{t \to \infty} W_1(t) W_1^{\mathrm{T}}(t) = \lim_{t \to \infty} W_2^{\mathrm{T}}(t) W_2(t).
	$
\end{itemize}

Note that (i), (ii) show that $M_{db}, M_b$ are invariant sets for the differential equation \eqref{eqn:gradient_flow_on_I}, respectively. In the argumentation that follows, let $\nabla_i$ denote the gradient operator in matrix form for $i \in \{1,2\}$. For example, $\nabla_2 \mathcal{I}(W) \in \R^{e \times f}$ and $ (\nabla_2 \mathcal{I}(W))_{ij} = \partial \mathcal{I}(W)/ \partial (W_2)_{ij}$. The negative gradients of \eqref{eqn:loss_whitened_scaled_dropout} are computed e.g.\ in \cite{mianjy2018implicit} and are given by:
\begin{align}
-\nabla_{1} \mathcal{I}(W) 
&= 2W_2^{\mathrm{T}}(Y - W_2W_1) - 2\lambda \mathrm{Diag}(W_2^{\mathrm{T}}W_2)  W_1 ,
\nonumber \\ 
-\nabla_{2} \mathcal{I}(W) 
&= 2(Y - W_2W_1)W_1^{\mathrm{T}} - 2\lambda W_2 \mathrm{Diag}(W_1W_1^{\mathrm{T}}).
\label{eqn:loss_matrixfact_gradient}
\end{align}

\noindent 
\emph{Proof of (i).}
We take time derivatives of $W_1(t) W_1^{\mathrm{T}}(t)$, $W_2^{\mathrm{T}}(t) W_2(t)$ and substitute \eqref{eqn:gradient_flow_on_I}, i.e., $\d{} W_i / \d{t} = - \nabla_i \mathcal{I}(W(t))$ for $i=1,2$. This results in
\begin{align}
\frac{\d{}}{\d{t}} \bigl( W_1(t) W_1^{\mathrm{T}}(t) \bigr) &= -W_1(t) \nabla_1 \mathcal{I}(W(t))^{\mathrm{T}} -   \nabla_1 \mathcal{I}(W(t))W_1(t)^{\mathrm{T}}, 
\label{eqn:IC_for_Lemma_9a}
\\
\frac{\d{}}{\d{t}} \bigl( W_2^{\mathrm{T}}(t) W_2(t) \bigr) &= -\nabla_2 \mathcal{I}(W(t))^{\mathrm{T}} W_2(t)  - W_2(t)^{\mathrm{T}}\nabla_2 \mathcal{I}(W(t)),
\label{eqn:IC_for_Lemma_9b}
\end{align}
respectively. We subtract \eqref{eqn:IC_for_Lemma_9b} from \eqref{eqn:IC_for_Lemma_9a} and then substitute \eqref{eqn:loss_matrixfact_gradient}, to find that
\begin{align}
&
\frac{\d{}}{\d{t}} \bigl( W_1(t) W_1^{\mathrm{T}}(t) - W_2^{\mathrm{T}}(t) W_2(t) \bigr)
\nonumber \\ &
= -2\lambda \Bigl( \mathrm{Diag}(W_2^{\mathrm{T}}(t) W_2(t))W_1(t) W_1^{\mathrm{T}}(t) 
+ W_1(t) W_1^{\mathrm{T}}(t)\mathrm{Diag}(W_2^{\mathrm{T}}(t) W_2(t)) 
\nonumber \\ &
\phantom{= -2\lambda \Bigl(}- W_2^{\mathrm{T}}(t) W_2(t)\mathrm{Diag}(W_1(t) W_1^{\mathrm{T}}(t)) 
- \mathrm{Diag}(W_1(t) W_1^{\mathrm{T}}(t))W_2^{\mathrm{T}}(t) W_2(t) \Bigr).
\label{eqn:lemma_db_implies_b}
\end{align}
Conclude in particular that
\begin{equation}
\frac{\d{}}{\d{t}} \bigl( \mathrm{Diag}(W_1(t) W_1^{\mathrm{T}}(t)) - \mathrm{Diag}(W_2^{\mathrm{T}}(t) W_2(t)) \bigr) 
= 0,
\label{eqn:lemma_Ats_evolution}
\end{equation}
by taking diagonals. Its solution is given by
\begin{equation}
\mathrm{Diag}(W_1(t) W_1^{\mathrm{T}}(t)) - \mathrm{Diag}(W_2^{\mathrm{T}}(t) W_2(t)) 
= \mathrm{Diag}(W_1(0) W_1^{\mathrm{T}}(0)) - \mathrm{Diag}(W_2^{\mathrm{T}}(0) W_2(0)),
\label{eqn:lemma_Ats_evolution_solution}
\end{equation}
i.e., a constant. This proves (i).

\noindent
\emph{Proof of (ii).}
The implied and weaker assumption  $\mathrm{Diag}(W_1(0) W_1^{\mathrm{T}}(0)) = \mathrm{Diag}(W_2^{\mathrm{T}}(0) W_2(0))$ combined with \eqref{eqn:lemma_Ats_evolution_solution} reveals to us that
\begin{equation}
	\mathrm{Diag}(W_1(t) W_1^{\mathrm{T}}(t)) = \mathrm{Diag}(W_2^{\mathrm{T}}(t) W_2(t)) = A(t)
	\label{eqn:Function_At}
\end{equation}
say, for any $t \geq 0$. Combining $W_1(t) W_1^{\mathrm{T}}(t) - W_2^{\mathrm{T}}(t) W_2(t) = S(t)$, say, with \eqref{eqn:Function_At} lets us reduce \eqref{eqn:lemma_db_implies_b} to
\begin{equation}
	\frac{ \d{S(t)} }{ \d{t} } 
	= -2 \lambda \bigl( A(t)S(t) + S(t)A(t) \bigr).
	\label{eqn:lemma_reduction_Mdb_to_Mb_1}
\end{equation}
The solution of \eqref{eqn:lemma_reduction_Mdb_to_Mb_1} in a neighborhood $V$ of $0$ is given by
\begin{equation}
	S(t) 
	= \e{ -2\lambda \int_0^t A(s) \d{s} } S(0) \e{ -2\lambda \int_0^t A(s) \d{s} }.
	\label{eqn:lemma_reduction_Mdb_to_Mb_2}
\end{equation}
Since $S(0) = 0$ by assumption, we have that $S(t) = 0$ for all $t \geq 0$. This proves (ii).

\noindent
\emph{Proof of (iii).} We split into cases.

\noindent
\underline{Case 1:} Suppose that there exists an $l$ such that both the row $W_{1 l \cdot}(t)$ as well as the column $W_{2 \cdot l}(t)$ converge to $0$. Then
\begin{equation}
\bigl( W_1(t) W_1^{\mathrm{T}}(t) \bigr)_{ij}
= \sum_k W_{1ik}(t) W_{1jk}(t)
\to 0
\quad 
\textnormal{whenever}
\quad 
i = l 
\quad 
\textnormal{or}
\quad  
j=l
\end{equation}
and similarly 
\begin{equation}
\bigl( W_2^{\mathrm{T}}(t) W_2(t) \bigr)_{ij}
= \sum_k W_{2ki}(t) W_{2kj}(t)
\to 0
\quad 
\textnormal{whenever}
\quad 
i = l 
\quad 
\textnormal{or}
\quad  
j=l.
\end{equation}
In particular, we have that 
\begin{equation}
	\lim_{t \to \infty} \bigl( W_2^{\mathrm{T}}(t) W_2(t) \bigr)_{ij}
	= \lim_{t \to \infty} \bigl( W_1(t) W_1^{\mathrm{T}}(t) \bigr)_{ij}
	\quad
	\textnormal{whenever}
	\quad
	i = l
	\quad
	\textnormal{or}
	\quad
	j = l.
\end{equation}

\noindent
\underline{Case 2:} 
Consider now any $l$ for which either the row $W_{1l\cdot}(t)$ or the column $W_{2 \cdot l}(t)$ does not converge to zero. In particular, there must then exist a sufficiently large $t_l \geq 0$ and $\epsilon_l > 0$ such that
\begin{equation}
	A_{ll}(t) 
	= \sum_k W_{1lk}^2(t)
	\eqcom{i}= \sum_k W_{2kl}^2(t)
	\geq \epsilon_l
\end{equation}
for all $t \geq t_l$. We therefore also have by \eqref{eqn:lemma_reduction_Mdb_to_Mb_2} that
\begin{align}
	S_{lj}(t)
	&
	= 
	\e{ -2 \lambda \int_{t_l}^t A_{ll}(s) \d{s} } S_{l,j}( t_l ) \e{ -2 \lambda \int_{t_l}^t A_{jj}(s) \d{s} }
	\nonumber \\ &
	\leq | S_{lj}(t_l) | \e{ - 2 \lambda \varepsilon_l ( t - t_l ) }
	\to 0
	\quad
	\textnormal{for}
	\quad
	j = 1, \ldots, f.
\end{align}
Hence, we obtain $\lim_{t \to \infty} S_{lj}(t) = 0$ for any $j$ and so (iii) is proven.

\noindent
\emph{Proof that $M_b = M_{db}$.}
Fact (i) implies that $M_{db}$ is an invariant set for \eqref{eqn:gradient_flow_on_I}. Fact (iii) tell us that if $W(0) \in M_{db}$ and $W(t)$ converges, then $\lim_{t \to \infty} W(t) \in M_b$. Combining facts (i) and (iii), it must be that $M_{db} \subseteq M_b$.

The inclusion $M_b \subseteq M_{db}$ follows immediately from \refDefinition{definition:balanced_maintext}. This concludes the proof.

\section{Proof of \refProposition{prop:lambda_cotangent_space_is_enough} }
\label{secappendix:local_convergence_GF}

\refProposition{prop:lambda_cotangent_space_is_enough} is a specification of \cite[Proposition~3.1]{fehrman2019convergence}. To arrive at \refProposition{prop:lambda_cotangent_space_is_enough} , all we need to do is prove that \cite[Proposition~3.1]{fehrman2019convergence} holds with the implicit convergence rate there ($\lambda$) replaced by the convergence rate
\begin{equation}
\min_{w \in \bar{V}_{R_0,\delta_0}(x_0)} \min_{\substack{\norm{v} = 1\\ v \in \ker \nabla^2 f(w)^{\perp}}} \abs{v^{\mathrm{T}} \nabla^2 f(w) v},
\end{equation}
where $V_{R, \delta}(x_0)$ is defined in \eqref{eqn:definition_V_set}. The convergence rate appears implicitly in the proof of \cite[Proposition~3.1]{fehrman2019convergence} after the application of \cite[Lemma 2.9]{fehrman2019convergence} at \cite[(3.11)]{fehrman2019convergence}. Hence, we need to make small but appropriate modifications to these steps in the proof of \cite[Lemma 2.9]{fehrman2019convergence}.

\noindent
\emph{Modifications to the proof of \cite[Lemma 2.9]{fehrman2019convergence}.}
Let $x_0 \in M \cap U$. Since $M \cap U$ is a nonempty $\mathfrak{d}$-dimensional submanifold of $\R^d$, we have by \cite[Proposition 2.1]{fehrman2019convergence} that there exists a neighborhood $V_*(x_0)$ of $x_0$ such that:
\begin{itemize}[topsep=2pt,itemsep=2pt,partopsep=2pt,parsep=2pt,leftmargin=0pt]
\item[(a)] For every $x \in V_*(x_0)$, there exists a unique projection $x_{*} \in M \cap U$ say such that
$
\pnorm{x - x_{*}}{} = d(x, M \cap U).
$
\item[(b)] This projection map $x \to x_*$ is locally $C^1$-smooth.
\end{itemize}
Fix $R_0, \delta_0 > 0$ such that for any $\delta \in (0, \delta_0] , R \in (0, R_0]$, it holds that $\bar{V}_{R, \delta}(x_0) \subset V_{*}(x_0)$. 
There exists an $r \in (0, \infty)$ such that
\begin{equation}
\max_{y \in \bar{B}_{R_0}(x_0) \cap M \cap U} \norm{\nabla^2 f(y_*)}  \leq \frac{1}{r}.
\label{eqn:Upper_bound_1_over_r}
\end{equation}

Our modified proof will be complete when we find a $\lambda$ that satisfies the following conditions:
\begin{itemize}[topsep=2pt,itemsep=2pt,partopsep=2pt,parsep=2pt,leftmargin=0pt]
	\item[(i)] $0 < \lambda \leq \max_{y \in \bar{B}_{R_0}(x_0) \cap M \cap U} \pnorm{\nabla^2 f(y_*)}{}$;
	\item[(ii)] for any $x \in V_{R, \delta}(x_0)$, $\pnorm{(x - x_*) - r \nabla^2 f(x_*) \cdot (x - x_*) }{} \leq (1 - r \lambda) \pnorm{x - x_*}{}$; and
	\item[(iii)] for any $x \in V_{R, \delta}(x_0)$, $\bigl( \nabla^2 f(x_*) \cdot ( x - x_* ) \bigr) \cdot  ( x - x_* ) \geq \lambda \pnorm{ x - x_* }{}^2$.
\end{itemize}
Condition (iii) is essentially our addendum to the proof of \cite[Lemma 2.9]{fehrman2019convergence}. 

Note that by assumption, $M \cap U$ is a nondegenerate submanifold of $\mathcal{P}$ (see \refDefinition{definition:non_degenerate}), so there is an embedding $M \cap U \to \mathcal{P}$ inducing an orthogonal decomposition 
\begin{equation}
	\mathrm{T}_{w_*} \R^{d} = \mathrm{T}_{w_*}(M \cap U) \oplus N_{w_*} = P_{w_*} \oplus N_{w_*}
	\label{eqn:Candidate_convergence_rate__Tangent_space}
\end{equation}
for which $\nabla^2 f({w_*})|_{P_{w_*}} > 0$ and $\nabla^2 f({w_*})|_{N_{w_*}} = 0$ for any $w_{*}$. It holds moreover that for any $w^{\prime} \in \bar{B}_{R_0}(x_0) \cap M \cap U$ that $\mathrm{dim}(\ker \nabla^2 f(w^{\prime})) = s$.

Taking inspiration from the decomposition in \eqref{eqn:Candidate_convergence_rate__Tangent_space}, we will now prove that the candidate
\begin{equation}
\tilde{\lambda} 
= \min_{w \in \bar{V}_{R_0, \delta_0}(x_0)} \min_{\substack{\norm{v} = 1\\ v \in \ker \nabla^2 f(w_*)^{\perp} = P_{w_*}}} \abs{v^{\mathrm{T}} \nabla^2 f(w_*) v}
\label{eqn:Candidate_convergence_rate}
\end{equation}
satisfies Conditions (i)--(iii).

\noindent
\underline{Condition (i):}
The orthogonal decomposition in \eqref{eqn:Candidate_convergence_rate__Tangent_space} together with the compactness of $\bar{V}_{R_0, \delta_0}$ guarantees the strict positivity of \eqref{eqn:Candidate_convergence_rate}. That is, $\tilde{\lambda} > 0$. 

For any $w \in \bar{V}_{R_0, \delta_0}(x_0)$, it holds that $w_* \in \bar{B}_{R}(x_0) \cap M \cap U$. This implies that
\begin{equation}
\tilde{\lambda} 
\leq \max_{w \in \bar{V}_{R_0, \delta_0}(x_0)} \pnorm{\nabla^2 f(w_*)}{} 
\leq \max_{w \in \bar{B}_{R_0}(x_0) \cap M \cap U} \pnorm{\nabla^2 f(w_*)}{}.
\end{equation}

\noindent
\underline{Condition (ii):}
Let $x \in V_{R, \delta}(x_0)$. Since $x - x_* \in P_{x_*}$, it follows that
\begin{equation}
	\pnorm{ (x - x_*)  - r \nabla^2 f(x_*) \cdot (x - x_*)}{}^2 
	= \pnorm{ (1 - r \nabla^2 f(x_*)) \cdot (x - x_*)}{}^2.
\end{equation} 
Recall now that we have the positive bilinear form $\nabla^2 f(x_*))|_{P_{x_*}}$ on $P_{x_*}$. Let $\lambda_{\min}(\nabla^2f(x_*)|_{P_{x_*}}) > 0$ be the minimal eigenvalue of $\nabla^2 f(x_*))|_{P_{x_*}}$. By \eqref{eqn:Upper_bound_1_over_r}, 
\begin{equation}
	0 
	< (1 - r \nabla^2 f(x_*))|_{P_{x_*}} 
	\leq 1 - r\lambda_{\min}(\nabla^2f(x_*)|_{P_{x_*}})
\end{equation}
as a positive bilinear form, so that 
\begin{equation}
	\pnorm{(1 - r \nabla^2 f(x_*)) \cdot (x - x_*)}{}^2 
	\leq \bigl( 1 - r\lambda_{\min}(\nabla^2 f(x_*)|_{P_{x_*}}) \bigr) \pnorm{x - x_*}{}^2.
\end{equation}
We have by nondegeneracy that $P_{x_*} = \ker \nabla^2 f(x_*) ^{\perp}$. Therefore $\lambda_{\min}(\nabla^2 f(x_*)|_{P_{x_*}}) \geq \tilde{\lambda}$ for any $x \in V_{R, \delta}(x_0)$.

\noindent
\underline{Condition (iii).} Let $x \in V_{R,\delta}(x_0)$. Similar to (ii), from $\lambda_{\min}(\nabla^2 f(x_*)|_{P_{x_*}}) \geq \tilde{\lambda}$ we conclude also
\begin{equation}
	\bigl( \nabla^2 f(x_*) \cdot ( x - x_* ) \bigr) \cdot  ( x - x_* ) 
	\geq \tilde{\lambda} \pnorm{ x - x_* }{}^2.
\end{equation}
This completes the proof.

\section{Proofs of \refSection{sec:outline_proof}}
\label{secappendix:outline_proofs}

\subsection{Proof \texorpdfstring{of \refProposition{prop:exists_diagonal_reduction_Mb_to_M}}{} -- Reduction from \texorpdfstring{$M$ to $M_b$}{M to Mb}}
\label{secappendix:reduction_M_M_b}

Let $W = (W_2, W_1) \in M$ and let $\pi$ be the action from \eqref{eqn:Action_pi}. Note that $\pi(C)(W_2,W_1) \in M$, since $\pi$ preserves the conditions in \eqref{eqn:optimal_diagonal_main} for $W$ to be a minimum. Hence, $\pi$ is well defined. Note now also that the same conditions imply that for $i=1, \ldots, f$,
\begin{equation}
	\bigl( \mathrm{Diag}(W_2^{\mathrm{T}} W_2^{\mathrm{T}}) \bigr)_{ii} 
	> 0, 
	\quad
	\bigl( \mathrm{Diag}(W_1 W_1^{\mathrm{T}}) \bigr)_{ii} 
	> 0.
\end{equation} 
This enables us to define
\begin{equation}
	C_{W} 
	= \mathrm{Diag}(W_1W_1^{\mathrm{T}})^{1/4} \mathrm{Diag}(W_2W_2^{\mathrm{T}})^{-1/4}
	\label{eqn:Intermediate__Definition_DW}
\end{equation}
and then consider the point $\pi(C_{W})(W) = (\tilde{W}_2, \tilde{W}_1) $ say. For this particular point,
\begin{align}
	\mathrm{Diag}(\tilde{W}_2^{\mathrm{T}} \tilde{W}_2) 
	&
	\eqcom{\ref{eqn:Action_pi}}= 
	C_{W}^{\mathrm{T}} \mathrm{Diag}(W_2^{\mathrm{T}}W_2) C_{W} 
	\nonumber \\ &
	\eqcom{\ref{eqn:Intermediate__Definition_DW}}= 
	\mathrm{Diag}(W_2^{\mathrm{T}}W_2)^{1/2} \mathrm{Diag}(W_1W_1^{\mathrm{T}})^{1/2} 
	\eqcom{\ref{eqn:optimal_diagonal_main}}= \frac{ \norm{\mathcal{W}^*}_1 }{ f } \mathrm{I}_{f}
	\label{eqn:Intermediate__Diag_Wtilde_1}
	\\ &
	= \mathrm{Diag}(\tilde{W}_1 \tilde{W}_1^{\mathrm{T}}).
	\label{eqn:Intermediate__Diag_Wtilde_2}
\end{align}
Here, \eqref{eqn:Intermediate__Diag_Wtilde_2} follows using the same (but appropriately modified) argumentation as for \eqref{eqn:Intermediate__Diag_Wtilde_1}. 
Consequently, $\pi(C_{W})(W) \in M_{db}$. 
Recalling that $M_b = M_{db}$ by \refLemma{lemma:first_integral_main} concludes the proof. \qed

\subsection{Proof \texorpdfstring{of \refProposition{prop:characterization_Mb_main}}{} -- Characterization of \texorpdfstring{$M_b$}{Mb}.}
\label{secappendix:characterazing_Mb}

Recall $M_b$, $M_{db}$'s definitions in \eqref{eqn:Definition_Mb}, \eqref{eqn:Definition_Mdb}, respectively. We now introduce the following two extended sets:
\begin{gather}
	\bar{M}_b 
	= \{ W = (W_2, W_1) \in \mathcal{P} : W_2^{\mathrm{T}} W_2 = W_1 W_1^{\mathrm{T}}, W_2 W_1 = \mathcal{S}_\alpha[Y] \},
	\quad
	\textnormal{and}
	\label{eqn:Definition_Mbbar}
	\\
	\bar{M}_{db} 
	= \{ W = (W_2, W_1) \in \mathcal{P} : \mathrm{Diag}(W_2^{\mathrm{T}} W_2) = \mathrm{Diag}(W_1 W_1^{\mathrm{T}}), W_2W_1 = \mathcal{S}_\alpha[Y] \}.
	\label{eqn:Definition_Mdbbar}
\end{gather}
The sets $\bar{M}_{db}, \bar{M}_{b}$ also contain diagonally balanced and balanced points respectively, but these points are not necessarily minima. They are extensions because
\begin{equation}
	M_{db} 
	= \bar{M}_{db} \cap M, 
	\quad
	\textnormal{and} 
	\quad
	M_b 
	= \bar{M}_b \cap M.
\end{equation}

Recall the definitions of $\rho$ in (\ref{eqn:Definition_rho_and_alpha}), $\Sigma_2, \Sigma_1$ in \eqref{eqn:Definition_Sigma1_and_Sigma2}, and $\Sigma$ in \eqref{eqn:Definition_Sigma_squared}. 

\begin{lemma}
	\label{lemma:Specific_SVD_decomposition_of_W_in_Mbbar}
	If \refAssumption{ass:eigenvalues_different_maintext} holds, then there exist a full \gls{SVD} of $W = (W_2, W_1) \in \bar{M}_b$ of the form $(U \Sigma_2 S, S^{\mathrm{T}} \Sigma_1 V)$ where $S \in \mathrm{O}(f)$.
\end{lemma}

\begin{proof} 
Let $W = (W_2, W_1) \in \bar{M}_b$. Consider a compact \gls{SVD} of the form $W =(U_2 \tilde{\Sigma}_2 S_2, S_1^{\mathrm{T}} \tilde{\Sigma}_1 V_1)$. Note that for this compact \gls{SVD} in particular
\begin{equation}
	S_2 S_2^{\mathrm{T}} 
	= \mathrm{Id}_{\dim{\tilde{\Sigma}_2}}
	\quad
	\textnormal{and}
	\quad
	S_1 S_1^{\mathrm{T}} 
	= \mathrm{Id}_{\dim{\tilde{\Sigma}_1}}.
\end{equation}
We also suppose (without loss of generality) that the singular values of $\tilde{\Sigma}_2$ and $\tilde{\Sigma}_1$ are both ordered in the diagonal from largest to smallest. 

Observe that
\begin{equation}
	S_2^{\mathrm{T}} \tilde{\Sigma}_2^2 S_2
	= W_2^{\mathrm{T}} W_2
	\eqcom{\ref{eqn:Definition_Mbbar}}= W_1 W_1^{\mathrm{T}}
	= S_1^{\mathrm{T}} \tilde{\Sigma}_1^2 S_1.
	\label{eqn:SVD_decomposition_lemma_step}
\end{equation}
Uniqueness of the singular values combined with \eqref{eqn:SVD_decomposition_lemma_step} implies that there exists a permutation matrix $P$ such that $\tilde{\Sigma}_2 = P \tilde{\Sigma}_1$. Moreover, because the singular values of $\tilde{\Sigma}_2$ $\tilde{\Sigma}_1$ are ordered by construction, we must have that 
(i) $\tilde{\Sigma}_2 = \tilde{\Sigma}_1 = \tilde{\Sigma}$ say. 
From \eqref{eqn:SVD_decomposition_lemma_step}, it follows in particular that
\begin{equation}
	S_2^{\mathrm{T}} \tilde{\Sigma}^2 S_2
	= 
	S_1^{\mathrm{T}} \tilde{\Sigma}^2 S_1.
	\label{eqn:SVD_decomposition_lemma_step3}
\end{equation}

Suppose now that $\tilde{\Sigma} \in \R^{l \times l}$, that the singular values are given by $\lambda_1, \ldots, \lambda_s$ (each distinct), and that their multiplicities are given by $r_1, \ldots, r_s$. Recall that $\sum_{i=1}^s r_i = l$ necessarily. After left-, right-multiplying \eqref{eqn:SVD_decomposition_lemma_step3} by $S_2$, $S_1^T$, respectively, it follows that the matrix $L = S_2 S_1^{\mathrm{T}}$ commutes with $\tilde{\Sigma}^2$:
$
	\tilde{\Sigma}^2 L 
	= L \tilde{\Sigma}^2
$.
Combining this fact with the fact that
\begin{equation}
	\tilde{\Sigma}^2
	=
	\begin{psmallmatrix}
		\lambda_1^2 \mathrm{I}_{r_1 \times r_1} & & \\
		& \lambda_2^2 \mathrm{I}_{r_2 \times r_2} & & \\
		& & \ddots & \\
		& & & \lambda_s^2 \mathrm{I}_{r_s \times r_s} \\
	\end{psmallmatrix}
	,
\end{equation}
in which all off-diagonal elements are equal to zero, leads to the conclusion that the matrix $L$ must be a conformally partitioned block-diagonal matrix of the form
\begin{equation}
	L 
	=
	\begin{psmallmatrix}
		L_1 & & \\
		& L_2 & & \\
		& & \ddots & \\
		& & & L_s \\
	\end{psmallmatrix}
	.
\end{equation}
Furthermore, $L$ must have strictly positive entries and $L_1 \in O(r_1), \ldots, L_s \in O(r_s)$ because of the uniqueness of the eigenspaces for each eigenvalue and therefore $L \in O(l)$. Consequently, $L$ also commutes with $\tilde{\Sigma}$:
\begin{equation}
	\tilde{\Sigma} L = L \tilde{\Sigma}.
	\label{eqn:SVD_decomposition_lemma_step4}
\end{equation}

Let $U_{\mathrm{c}} \Sigma^2 V_{\mathrm{c}}$ now be a compact \gls{SVD} of $\mathcal{S}_\alpha[Y]$. Recall \eqref{eqn:Definition_rho_and_alpha} and \eqref{eqn:Shrinkage_thresholding_operator}, and conclude that $\Sigma^2$ is given by \eqref{eqn:Definition_Sigma_squared}. Observe that
\begin{equation}
	U_{\mathrm{c}} \Sigma^2 V_{\mathrm{c}}
	= \mathcal{S}_\alpha[Y]
	\eqcom{\ref{eqn:Definition_Mbbar}}= 
	W_2 W_1
	\eqcom{SVD}= 
	U_2 \tilde{\Sigma}_2 S_2 S_1^{\mathrm{T}} \tilde{\Sigma}_1 V_1
	\eqcom{i}= 
	U_2 \tilde{\Sigma} L \tilde{\Sigma} V_1 
	\eqcom{\ref{eqn:SVD_decomposition_lemma_step4}}= 
	U_2 L \tilde{\Sigma}^2 V_1
	.
	\label{eqn:SVD_decomposition_lemma_step5}
\end{equation}
Remark now that $(U_2 L)^{\mathrm{T}}(U_2 L) = \mathrm{Id}_{l}$. Consequently, the left-hand side as well as the right-hand side of \eqref{eqn:SVD_decomposition_lemma_step5} are compact \glspl{SVD}. By uniqueness of the singular values we must again have that there exists a permutation matrix $P^{\prime}$ such that $\tilde{\Sigma}^2 = P^{\prime} \Sigma^2$. By construction, the singular values of both diagonal matrices $\tilde{\Sigma}^2$ and $\Sigma^2$ were put in the same order. This implies that we must have $\tilde{\Sigma}^2 = \Sigma^2$. By positivity of the entries, we must consequently also have 
(ii) $\tilde{\Sigma} = \Sigma$. 
The singular values in $\Sigma^2$ have no multiplicity by \refAssumption{ass:eigenvalues_different_maintext}, so equating multiplicities yields $r_1 = \ldots = r_{s} = 1$, $l = \rho$. Moreover, $L \in O(r)$ is a diagonal matrix with $\{ -1, + 1\}$-valued entries.

The uniqueness of the left and right eigenvectors in the left-hand side as well as the right-hand side of \eqref{eqn:SVD_decomposition_lemma_step5} together with the fact that all eigenvalues of $\Sigma$ have multiplicity one, implies that there exists a diagonal matrix $D$ with entries in $\{ -1, + 1\}$ such that $U_{\mathrm{c}} D = U_2 L $ and 
(iii) $D V_{\mathrm{c}} = V_1$. 
In particular, 
(iv) $U_2 = U_{\mathrm{c}} D L^T = U_{\mathrm{c}} D L$. 
Also, from the facts that $L$ is a diagonal matrix with $\{ -1, + 1\}$-valued entries and both $S_1$, $S_2$ have orthonormal rows, we obtain 
from $L = S_2S_1^{\mathrm{T}}$ that
(v) $L S_1 = S_2$. 
Utilizing (i--v), together with  
(vi) the fact that $D, L, \Sigma$ are diagonal matrices which are thus symmetric and commute, 
we can rewrite the compact \gls{SVD} of $W$ as
\begin{align}
	\bigl( 
		U_2 \tilde{\Sigma}_2 S_2, 
		S_1^{\mathrm{T}} \tilde{\Sigma}_1 V_1
	\bigr) 
	& 
	\eqcom{i,ii}= 
	\bigl(
		U_2 \Sigma S_2, 
		S_1^{\mathrm{T}} \Sigma V_1
	\bigr) 
	\eqcom{iii,iv}= 
	\bigl(
		U_{\mathrm{c}} D L \Sigma S_2, 
		S_1^{\mathrm{T}} \Sigma D V_{\mathrm{c}}
	\bigr) 	
	\nonumber \\
	& 
	\eqcom{v}= 
	\bigl(
		U_{\mathrm{c}}  D L \Sigma S_2, 
		S_2^{\mathrm{T}} L \Sigma D V_{\mathrm{c}}
	\bigr) 
	\eqcom{vi}= 
	\bigl(
		U_{\mathrm{c}}  \Sigma (D L  S_2), 
		(S_2 L D)^{\mathrm{T}} \Sigma V_{\mathrm{c}}
	\bigr).
	\label{eqn:SVD_decomposition_lemma_step6}
\end{align}

We can extend the compact \gls{SVD} in \eqref{eqn:SVD_decomposition_lemma_step6} to a full \gls{SVD} by noting that $S$ will be the extension of $D L  S_2$  to an orthogonal matrix in $\mathrm{O}(f)$, and $U$ and $V$ will be the extensions of $U_{\mathrm{c}}$ and $V_{\mathrm{c}}$ to $\mathrm{O}(e)$ and $\mathrm{O}(h)$, respectively. Similarly, $\Sigma_2$ and $\Sigma_1$ will be the extension to a full \gls{SVD}.
\end{proof}

We next characterize $\bar{M}_b$ from \eqref{eqn:Definition_Mbbar} as a homogeneous manifold. Let us summarize the method first. Suppose that $G$ is a finite dimensional Lie group, that is, a group with a smooth manifold structure (for example, $\mathrm{GL}(n)$ or $\mathrm{SL}(n)$). Suppose for a moment that $\bar{M}_b$ is a set, and that there is a \emph{transitive} Lie group action $\pi: G \times \bar{M}_b \to \bar{M}_b$. A transitive action means that for any $a, b \in \bar{M}_b$, there exist a $g \in G$ such that $\pi(g)(a) = b$. We define the \emph{stabilizer subgroup} (also called \emph{isotropy subgroup}) of $\pi$ at $a \in \bar{M}_b$ as $\mathrm{Stab}_{G}(a) = \{ g \in G : \pi(g)(a) = a \}$. We will use that if for a point $a \in \bar{M}_b$, $\mathrm{Stab}_{G}(a) \subseteq G$ is a closed smooth Lie subgroup (closed in the topology of $G$), then there exists a smooth manifold structure on $\bar{M}_b$ which is that of the homogeneous manifold $G / \mathrm{Stab}_{G}(a)$ \cite[Thm.~21.20]{lee2013smooth}. Once we have a \emph{diffeomorphism} $\bar{M}_b \simeq G / \mathrm{Stab}_{G}(a)$ (a differentiable isomorphism with differentiable inverse) we can consider the projection map $\Pi: G \to G / \mathrm{Stab}_{G}(a) \simeq \bar{M}_{b}$ and look at the differential $\mathrm{D}\Pi: \mathfrak{g} \to \mathrm{T}_{0}( G / \mathrm{Stab}_{G}(a)) $ at $\Pi(Id) = [ \mathrm{Stab}_{G}(a) ] = 0$, where $\mathfrak{g}$ is the Lie algebra of $G$. The linear map $\mathrm{D}\Pi$ is surjective and the kernel is the Lie algebra of $\mathrm{Stab}_{G}(a)$, denoted by $\mathrm{Lie}
( \mathrm{Stab}_{G}(a))$. Hence as vector spaces
\begin{equation}
	\frac{\mathfrak{g}}{\mathrm{Lie}( \mathrm{Stab}_{G}(a))} 
	\simeq \mathrm{T}_a \bar{M}_b.
	\label{eqn:isomorphism_tangent_space_M_b_bar}
\end{equation}
We refer the reader to \cite[Ch. 4]{arvanitogeorgos2003introduction} for more details on homogeneous spaces.

\begin{lemma}
\label{lemma:Diffeomorphism_of_Mbbar}
If \refAssumption{ass:eigenvalues_different_maintext} holds, then there is a diffeomorphism $\bar{M}_b \simeq \mathrm{O(f) / (I_{\rho}  \oplus O(f-\rho))}$, i.e., the manifold $\bar{M}_b$ is a homogeneous space.
\end{lemma}

\begin{proof}
Consider the smooth Lie group action $\pi: \mathrm{O}(f) \times \bar{M}_b \to \bar{M}_b$ given by
\begin{equation}
	\pi(L)(W_2, W_1) 
	= (W_2 L, L^{\mathrm{T}} W_1).
	\label{eqn:definition_orthogonalgroup_action_lemma}
\end{equation}

For $W = (W_2,W_1) \in \bar{M}_b$, we are first going to determine the stabilizer subgroup
\begin{equation}
	\mathrm{Stab}_{\mathrm{O}(f)}(W)
	= \{ S' \in \mathrm{O}(f) : \pi(S')(W) = W \}.
	\label{eqn:Stabilizer_subgroup}
\end{equation}
Let to that end $(U \Sigma_2 S, S^{\mathrm{T}} \Sigma_1 V)$ be an \gls{SVD} of $W$, which exists by \refLemma{lemma:Specific_SVD_decomposition_of_W_in_Mbbar}. 
Note then that for any orthogonal matrix $S' \in \mathrm{O}(f)$ of the form
\begin{equation}
	S^{\prime} 
	= S^{\mathrm{T}} 
	\begin{psmallmatrix}
		A & B \\
		C & D \\
	\end{psmallmatrix} 
	S 
	\quad
	\textnormal{where}
	\quad
	A \in \R^{\rho \times \rho},
	\label{eqn:Form_of_Sprime}
\end{equation}
we have that
\begin{align}
	\pi(S^{\prime})(W)
	&	
	\eqcom{\ref{eqn:definition_orthogonalgroup_action_lemma}}= \bigl( W_2 S', (S')^{\mathrm{T}} W_1 \bigr)
	\eqcom{SVD}= \bigl( U \Sigma_2 S S', (S^{\prime})^{\mathrm{T}} S^{\mathrm{T}} \Sigma_1 V \bigr)
	\nonumber \\ &
	\eqcom{\ref{eqn:Form_of_Sprime}}= \bigl(
	U \Sigma_2
	\begin{psmallmatrix}
		A & B \\
		C & D \\
	\end{psmallmatrix}
	S
	,
	S^{\mathrm{T}}
	\begin{psmallmatrix}
		A & B \\
		C & D \\
	\end{psmallmatrix}^{\mathrm{T}} \Sigma_1 V
	\bigr)
	\eqcom{\ref{eqn:Definition_Sigma1_and_Sigma2}}= \bigl(
	U
	\begin{psmallmatrix}
		\Sigma A & \Sigma B \\
		0 & 0 \\
	\end{psmallmatrix}
	S
	,
	S^{\mathrm{T}}
	\begin{psmallmatrix}
		A^{\mathrm{T}} \Sigma & 0 \\
		B^{\mathrm{T}} \Sigma & 0 \\
	\end{psmallmatrix} V
	\bigr)
	= W
\end{align}
if and only if 
\begin{equation}
	\begin{psmallmatrix}
		\Sigma A & \Sigma B \\
		0 & 0 \\
	\end{psmallmatrix}
	= \Sigma_2
	\quad
	\textnormal{and}
	\quad
	\begin{psmallmatrix}
		A^{\mathrm{T}} \Sigma & 0 \\
		B^{\mathrm{T}} \Sigma & 0 \\
	\end{psmallmatrix}
	= \Sigma_1.
	\label{eqn:SigmaAB_equals_Sigma2_and_ATBTSigma_equals_Sigma_1}
\end{equation}
Because of our \refAssumption{ass:eigenvalues_different_maintext} on the multiplicity of the eigenvalues, \eqref{eqn:SigmaAB_equals_Sigma2_and_ATBTSigma_equals_Sigma_1} holds if and only if $B = 0$ and $A = \mathrm{Id}_{\rho}$. We must then furthermore have that $C = 0$ and $D \in \mathrm{O}(f - \rho)$ because $S^{\prime} \in \mathrm{O}(f)$. We have shown that 
\begin{equation}
	\mathrm{Stab}_{\mathrm{O}(f)}(W) 
	\simeq S^{\mathrm{T}} (\mathrm{I_{\rho}  \oplus O(f-\rho)} )S,
\end{equation} 
the right-hand side of which is a closed, smooth Lie subgroup of $\mathrm{O}(f)$. 

Next, we prove the diffeomorphism.
\refLemma{lemma:Specific_SVD_decomposition_of_W_in_Mbbar} ensures that $\pi$ is transitive. Transitiveness ensures that the choice of $L$ in \eqref{eqn:definition_orthogonalgroup_action_lemma} only changes the stabilizer subgroup by conjugation, i.e., 
\begin{equation}
	\mathrm{Stab}_{\mathrm{O}(f)}(\pi(L)W) 
	= L^{-1} \mathrm{Stab}_{\mathrm{O}(f)}(W) L.
\end{equation}
The set $\bar{M}_b$ admits therefore a smooth manifold structure, and we have the following diffeomorphism of smooth manifolds \cite[Thm.~21.20]{lee2013smooth}:
\begin{equation}
	\bar{M}_b \simeq \frac{\mathrm{O}(f)}{\mathrm{I_{\rho}  \oplus O(f-\rho)}}.
\end{equation}
This completes the proof.
\end{proof}

\refLemma{lemma:Specific_SVD_decomposition_of_W_in_Mbbar} guarantees that each point $W \in \bar{M}_b$ has an \gls{SVD} of the form $(U \Sigma_2 S, S^{\mathrm{T}} \Sigma_1 V)$ where $[S] \in \mathrm{O(f) / (I_{\rho}  \oplus O(f-\rho))}$. Here, we understand $S$ as being any representative of the equivalence class $[S]$. 
Conclude using \eqref{eqn:optimal_diagonal_main}, \eqref{eqn:Definition_Mb} and \eqref{eqn:Definition_Mbbar} that if $W \in \bar{M}_b$, then $W \in M_b$ also if and only if moreover $\mathrm{Diag}(W_2^{\mathrm{T}} W_2) = \mathrm{Diag}(W_1 W_1^{\mathrm{T}}) = \pnorm{\Sigma^2}{1} \mathrm{I}_f / f$. 
Combined with the isomorphism in \refLemma{lemma:Diffeomorphism_of_Mbbar}, this provides us with the alternative representation in \eqref{eqn:Alternative_representation_of_Mb}.

All that remains is to prove that $M_b \neq \emptyset$. This fact was also proven in \cite{mianjy2018implicit}, but for completeness we will prove it here using the \emph{theory of majorization} instead. For any vector $a \in \realNumbers^f$, denote by $a^\downarrow \in \realNumbers^f$ the vector with the same components but sorted in descending order. Given two vectors $a, b \in \realNumbers^f$, we say that $a$ is \emph{majorized} by $b$, written as $a \prec b$, if
\begin{equation}
	\sum_{i=1}^l a_i^\downarrow 
	\leq \sum_{i=1}^l b_i^\downarrow
	\quad
	\textnormal{for} 
	\quad
	l=1, \ldots, f
	\quad
	\textnormal{and furthermore}
	\quad
	\sum_{i=1}^f a_i 
	= \sum_{i=1}^f b_i.
\end{equation}

\begin{lemma}
	\label{lemma:balanced_critical_points_nonnul}	
	It holds that $M_b \neq \emptyset$.
\end{lemma}

\begin{proof}
We temporarily abuse our notation and let $\mathrm{Diag}$ denote the map that: (a) maps vectors $y \in \R^{f}$ to an $\realNumbers$-valued diagonal $f \times f$ matrices with $y_1, \ldots, y_f$ for its diagonal entries, and (b) maps matrices $A \in \R^{f \times f}$ to $\realNumbers$-valued vectors with entries $A_{11}, \ldots, A_{ff}$. 

If $a \in [0,\infty)^{f}$ and $L \in \mathrm{O}(f)$, then as a linear map $\mathrm{Diag}(L \mathrm{Diag}(y) L^{\mathrm{T}}) = Py$ for some \emph{orthostochastic} matrix $P$; specifically, the doubly stochastic matrix that is formed by taking the square of the entries of $L \in \mathrm{O}(f)$ \cite[Definition B.5, p.34]{marshall1979inequalities}. Because $P$ is doubly stochastic, we have that $Pa \prec a$. Note now that \emph{Horn's theorem} states that the converse is also true \cite[Theorem B.6, p.35]{marshall1979inequalities}: if $a \prec b$, then there exists a orthostochastic matrix $Q$ such that $Qb = a$. In particular, there exists some $L \in \mathrm{O}(f)$ satisfying $\mathrm{Diag}(L \mathrm{Diag}(b) L^{\mathrm{T}}) = a$ whenever $a \prec b$. 

Consider now the two $f$-dimensional vectors
\begin{gather}
	a 
	= \Bigl( \frac{\pnorm{\Sigma^2}{}}{f}, \ldots, \frac{\pnorm{\Sigma^2}{}}{f} \Bigr),
	\\
	b
	= \Bigl( \sigma_1 - \frac{ \rho \lambda \kappa_{\rho} }{ f + \rho \lambda }, \ldots, \sigma_{\rho} - \frac{ \rho \lambda \kappa_{\rho} }{ f + \rho \lambda }, 0, \ldots, 0 \Bigr)
\end{gather}
specifically, and note in particular that $a \prec b$. Applying Horn's theorem proves that there exists an orthogonal matrix $L \in O(f)$ such that 
\begin{equation}
	\mathrm{Diag}
	\bigl( 
		L 
		\begin{psmallmatrix}
			\Sigma^2 & 0 \\
			0 & 0 \\
		\end{psmallmatrix}
		L^{\mathrm{T}} 
	\bigr) 
	= \frac{ \pnorm{\Sigma^2}{} }{f} \mathrm{I}_f.
\end{equation}
In particular, we have shown that the condition in \eqref{eqn:Alternative_representation_of_Mb} holds. Consequently, $M_b \neq \emptyset$.
\end{proof}

\subsection{Proof of \texorpdfstring{\refProposition{prop:Characterization_of_TWMb}}{} -- Characterization of \texorpdfstring{$\mathrm{T}_{W} M_b$}{TW Mb}.}
\label{sec:Appendix__Characterization_of_TWMb}

We start by describing the tangent space of $\bar{M}_b$. Using the diffeomorpishm in \refLemma{lemma:Diffeomorphism_of_Mbbar} together with \eqref{eqn:isomorphism_tangent_space_M_b_bar}, we find that for any $W \in \bar{M}_b$,
\begin{equation}
	\mathrm{T}_W \bar{M}_b 
	\simeq 
	\mathrm{T}_{0}\Bigl( \frac{ \mathrm{O}(f) }{ \mathrm{I_{\rho} \oplus O(f-\rho)} } \Bigr) 
	\simeq 
	\frac{ \mathfrak{o}(f)) }{ 0_{\rho} \oplus \mathfrak{o}(f-\rho) }.
	\label{eqn:Intermediate__TWMbbar_diffeomorpishm_with_of_0rho_ofmrho}
\end{equation}
Here, $\mathfrak{o}(s)$ denotes the Lie algebra of the orthogonal group $\mathrm{O(s)}$, and
\begin{equation}
	\frac{ \mathfrak{o}(f) }{ 0_{\rho}  \oplus \mathfrak{o}(f-\rho) }
	= 
	\Bigl\{ 
		\begin{psmallmatrix}
			X & E \\
			-E^{\mathrm{T}} & 0 \\
		\end{psmallmatrix}
		: 
		X \in \mathrm{Skew}(\R^{\rho \times \rho}), E \in \R^{\rho \times (f - \rho)} 
	\Bigr\}.
	\label{eqn:Intermediate__ofmod0rhoofmrho}
\end{equation}
Note also that the isomorphism in \eqref{eqn:Intermediate__TWMbbar_diffeomorpishm_with_of_0rho_ofmrho} is given by the differential $\mathrm{D}_{\mathrm{Id}} \pi$ of the action $\pi$ in \eqref{eqn:Action_pi} at the identity of $\mathrm{O}(f)$; that is, $\mathrm{T}_W ( \bar{M}_b ) = \mathrm{D}_{\mathrm{Id}}\pi(\mathfrak{o}(f) / 0_{\rho}  \oplus \mathfrak{o}(f-\rho))(W)$. 

Recall now that $W$ has a \gls{SVD} decomposition of the form $(U \Sigma_2 S, S^{\mathrm{T}} \Sigma_1 V)$ by \refLemma{lemma:Specific_SVD_decomposition_of_W_in_Mbbar}. We therefore have that for any $( X, E; -E^{\mathrm{T}}, 0 ) \in \mathfrak{o}(f) / 0_{\rho}  \oplus \mathfrak{o}(f-\rho)$, 
\begin{align}
	\mathrm{D}_{\mathrm{Id}} \pi\Bigl( 
	S^{\mathrm{T}} 
	\begin{psmallmatrix}
	X & E \\ 
	-E^{\mathrm{T}} & 0 \\
	\end{psmallmatrix}
	S 
	\Bigr)
	(W) 
	&
	\eqcom{\ref{eqn:definition_orthogonalgroup_action_lemma}}= \Bigl( 
		W_2 
		S^{\mathrm{T}} 
		\begin{psmallmatrix}
		X & E \\ 
		-E^{\mathrm{T}} & 0 \\
		\end{psmallmatrix}
		S 
		, 
		S^{\mathrm{T}} 
		\begin{psmallmatrix}
		X^{\mathrm{T}} & - E \\ 
		E^{\mathrm{T}} & 0 \\
		\end{psmallmatrix}
		S 
		W_1 
		\Bigr)
	\nonumber \\ &
	= \Bigl( U \Sigma_2 
	\begin{psmallmatrix}
	X & E \\ 
	-E^{\mathrm{T}} & 0 \\
	\end{psmallmatrix}
	S, 
	S^{\mathrm{T}} 
	\begin{psmallmatrix}
	X^{\mathrm{T}} & -E \\ 
	E^{\mathrm{T}} & 0 \\
	\end{psmallmatrix}
	\Sigma_1 V
	\Bigr).
	\label{eqn:Differential_of_pi_at_the_identity_element_of_Of}
\end{align}
Consequently,
\begin{align}
	&
	\mathrm{T}_W ( \bar{M}_b )
	= \mathrm{D}_{\mathrm{Id}}\pi(\mathfrak{o}(f) /0_{\rho}  \oplus \mathfrak{o}(f-\rho))(W)
	\nonumber \\ &
	= \Bigl\{ \Bigl(
	U \Sigma_2
	\begin{psmallmatrix}
			X & E \\
			-E^{\mathrm{T}} & 0 \\
		\end{psmallmatrix}
	S,
	S^{\mathrm{T}}
	\begin{psmallmatrix}
			X^{\mathrm{T}} & -E \\
			E^{\mathrm{T}} & 0 \\
		\end{psmallmatrix}
	\Sigma_1 V \Bigr)
	:
	X \in \mathrm{Skew}(\R^{\rho \times \rho}),
	E \in \R^{\rho \times (f -\rho)}
	\Bigr\}.
	\label{eqn:Tangent_space_at_W_of_Mbbar}
\end{align}

Next, recall that $\mathrm{D}_W T: \mathrm{T}_W \bar{M}_b  \to \mathrm{T}_{T(W)}\R^{f}$. Concretely, for any
\begin{equation}
	(V_2,V_1)
	= \Bigl(
	U \Sigma_2
	\begin{psmallmatrix}
			X & E \\
			-E^{\mathrm{T}} & 0
		\end{psmallmatrix}
	S,
	S^{\mathrm{T}}
	\begin{psmallmatrix}
			X^{\mathrm{T}} & -E \\
			E^{\mathrm{T}} & 0 \\
		\end{psmallmatrix}
	\Sigma_1 V
	\Bigr)
	\in \mathrm{T}_W \bar{M}_b
	\label{eqn:Intermediate__Point_V}
\end{equation}
say, we have that
\begin{align}
	&
	\mathrm{D}_W T(V_2,V_1)
	\nonumber \\ &
	\eqcom{\ref{eqn:Definition_T_map}}= \mathrm{Diag}\Bigl( \Bigl( \mathrm{D}_{\mathrm{Id}} \pi
	\Bigl(
	S^{\mathrm{T}}
	\begin{psmallmatrix}
		X & E \\
		-E^{\mathrm{T}} & 0 \\
	\end{psmallmatrix}
	S
	\Bigr)(W) \Bigr)_1 W_1^{\mathrm{T}}
	+ W_1 \Bigl( \mathrm{D}_{\mathrm{Id}} \pi
	\Bigl(
	S^{\mathrm{T}}
	\begin{psmallmatrix}
		X & E \\
		-E^{\mathrm{T}} & 0 \\
	\end{psmallmatrix}
	S
	\Bigr)(W) \Bigr)_1^{\mathrm{T}} \Bigr)
	\nonumber \\ &
	\eqcom{\ref{eqn:Differential_of_pi_at_the_identity_element_of_Of}}= \mathrm{Diag}\Bigl(
	S^{\mathrm{T}}
	\begin{psmallmatrix}
		X^{\mathrm{T}} & -E \\
		E^{\mathrm{T}} & 0 \\
	\end{psmallmatrix}
	\Sigma_1 V W_1^{\mathrm{T}}
	+
	W_1 V^{\mathrm{T}} \Sigma_1^{\mathrm{T}}
	\begin{psmallmatrix}
		X & E \\
		-E^{\mathrm{T}} & 0 \\
	\end{psmallmatrix}
	S
	\Bigr)
	\nonumber \\ &
	\eqcom{SVD}= \mathrm{Diag}\Bigl(
	S^{\mathrm{T}}
	\begin{psmallmatrix}
		X^{\mathrm{T}} & -E \\
		E^{\mathrm{T}} & 0 \\
	\end{psmallmatrix}
	\Sigma_1 V V^{\mathrm{T}} \Sigma_1^{\mathrm{T}} S
	+
	S^{\mathrm{T}} \Sigma_1 V V^{\mathrm{T}} \Sigma_1^{\mathrm{T}}
	\begin{psmallmatrix}
		X & E \\
		-E^{\mathrm{T}} & 0 \\
	\end{psmallmatrix}
	S
	\Bigr)
	\nonumber \\ &
	\eqcom{\ref{eqn:Definition_Sigma1_and_Sigma2},\ref{eqn:Definition_Sigma_squared}}= \mathrm{Diag}\Bigl(
	S^{\mathrm{T}}
	\begin{psmallmatrix}
		X^{\mathrm{T}} & -E \\
		E^{\mathrm{T}} & 0 \\
	\end{psmallmatrix}
	\begin{psmallmatrix}
		\Sigma^2 & 0 \\
		0 & 0 \\
	\end{psmallmatrix}
	S
	+
	S^{\mathrm{T}}
	\begin{psmallmatrix}
		\Sigma^2 & 0 \\
		0 & 0 \\
	\end{psmallmatrix}
	\begin{psmallmatrix}
		X & E \\
		-E^{\mathrm{T}} & 0 \\
	\end{psmallmatrix}
	S
	\Bigr)
	\nonumber \\ &
	= \mathrm{Diag}\Bigl(
	S^{\mathrm{T}}
	\begin{psmallmatrix}
		X^{\mathrm{T}} \Sigma^2 & 0 \\
		E^{\mathrm{T}} \Sigma^2 & 0 \\
	\end{psmallmatrix}
	S
	+
	S^{\mathrm{T}}
	\begin{psmallmatrix}
		\Sigma^2 X & \Sigma^2 E \\
		0 & 0 \\
	\end{psmallmatrix}
	S
	\Bigr)
	= 2 \mathrm{Diag}\Bigl( S^{\mathrm{T}}
	\begin{psmallmatrix}
		\Sigma^2 X  & \Sigma^2 E \\
		0 & 0 \\
	\end{psmallmatrix}
	S \Bigr).
\end{align}

Let now $W \in M_b \backslash \mathrm{Sing}(M_b)$. By \eqref{eqn:Definition__Singular_points_of_Mb}, $\mathrm{D}_{W} T$ has maximal rank $f-1$. By continuity, the full rank property holds in an open set. There thus exists an open neighborhood $\mathcal{N}_W \subseteq \bar{M}_{b}$ say of $W$ such that for any $W' \in \mathcal{N}_W$ the rank of $\mathrm{D}_{W'} T$ is constant and equal to $f-1$.
Note now that $T : \mathcal{N}_W \to \realNumbers^f$ is a smooth function and we have $T(W) = \pnorm{\Sigma^2}{1} / f$ by \eqref{eqn:optimal_diagonal_main} and \eqref{eqn:Definition_Mb}. In particular for any $W \in \mathcal{N}_W \cap M_b$ we have $\mathrm{D}_{W'} T$ is maximal and $T(W') = \pnorm{\Sigma^2}{1} / f$. The \emph{constant rank theorem} \cite[Theorem~5.22]{lee2013smooth} therefore applies, and there exists an open neighborhood $\mathcal{U}_{W} \subseteq \mathcal{N}_W$ of $W$ such that 
\begin{equation}
	T^{-1}\Bigl( \frac{ \pnorm{ \Sigma^2}{1} }{f} \Bigr) \cap \mathcal{U}_W 
	= M_b \cap U_{W}
\end{equation}
is a smooth embedded manifold in $\bar{M}_b$ of codimension $f-1$.

Note now furthermore that for any $W \in \bar{M}_{b}$, $\mathrm{Tr}[T(W)] = \pnorm{ \Sigma^2 }{}$ by the diffeomorphism in \refLemma{lemma:Diffeomorphism_of_Mbbar}. The map $\mathrm{D}_{W} T$ can therefore have rank $f - 1$ at most in particular. That is, any $f-1$ components of $T$ are regular at $W$ and we can therefore consider $M_b$ as being an embedded manifold in $\bar{M}_b$ \cite[Proposition~5.28]{lee2013smooth}. Hence, by \cite[Lemma~5.29]{lee2013smooth} we also have that for any $Q \in T^{-1} ( \pnorm{ \Sigma^2}{1} /f) \cap U_{W} = M_b \cap U_{W}$ we have the representation
\begin{equation}
	\ker \mathrm{D}_{Q} T 
	= \mathrm{T}_{Q} M_b,
\end{equation}
where we understand $\mathrm{T}_{Q} M_b$ as a subspace of $\mathrm{T}_{Q} \bar{M}_b$.

This concludes the proof. \qed

\subsection{Proof \texorpdfstring{of \refProposition{prop:Mb_is_almost_everywhere_nonsingular}}{} -- The set \texorpdfstring{$\mathrm{Sing}(M_b)$}{Sing(Mb)}}
\label{sec:Mb_is_almost_everywhere_nonsingular}

We start by proving that if there exists a point $W \in M_b$ such that $\mathrm{rk}( \mathrm{D}_W T ) = f-1$, or in other words $M_{b} \backslash \mathrm{Sing}(M_{b}) \neq \emptyset$, then \refProposition{prop:Mb_is_almost_everywhere_nonsingular} holds. The proof relies on an established fact for the singular loci in affine algebraic varieties, of which $M_{b}$ is one.

\begin{lemma}
	\label{lemma:one_point_implies_almost_everywhere}
	If there exists a point $W \in M_b$ such that $\mathrm{rk}( \mathrm{D}_W T ) = f-1$, then \refProposition{prop:Mb_is_almost_everywhere_nonsingular} holds.
\end{lemma}

\begin{proof}
	Fix any point $W = (W_2, W_1) \in \bar{M}_{b}$. Let $\pi$ be the action defined in \eqref{eqn:Action_pi}, and recall the representation of $M_b$ in \eqref{eqn:Alternative_representation_of_Mb} as well as the definition of $\bar{M}_b$ in \eqref{eqn:Definition_Mbbar}. Observe that the set $M_{b}$ can be defined as the set of solutions to the algebraic equations
	\begin{gather}
		L
		\in [L]
		\in \frac{\mathrm{O}(f)}{ \mathrm{I}_{\rho} \oplus \mathrm{O}(f - \rho)} \simeq \bar{M}_b
		\nonumber \\
		\frac{\norm{\Sigma^2}_1}{f}
		\mathrm{I}_{f \times f}
		=
		\mathrm{Diag}(L^{\mathrm{T}} W_2^{\mathrm{T}} W_2 L)
		=
		\mathrm{Diag}(L^{\mathrm{T}} W_1 W_1^{\mathrm{T}} L)
		\eqcom{\ref{eqn:Definition_T_map}}= 
		T( \pi(L)(W)).
		\label{eqn:lemma_nonsingular_algebraic_variety_explicit}
	\end{gather}
	We may therefore consider $M_{b}$ as a \emph{real algebraic variety} of $\mathcal{P}$; that is, the zero loci (the set of real solutions) of a set of real polynomials of finite degree with variables in $\mathcal{P}$. Let $P_1, \ldots, P_s$ with $s = \dim \mathcal{P} - \dim(\bar{M}_b)$ be the polynomials defining $\bar{M}_b$ at zero, that is, $\bar{M}_b = P_1^{-1}(0) \cap \ldots \cap P_s^{-1}(0)$. If we denote the gradient with respect to the coordinates in $\mathcal{P}$ by $\nabla$, then the matrix composed by $( \nabla P_i)_{i=1}^{s}$ has rank $\dim \mathcal{P} - \dim(\bar{M}_b)$ at $W$ whenever $P_1(W) = \ldots = P_s(W) = 0$. Eq.\ \eqref{eqn:lemma_nonsingular_algebraic_variety_explicit} also shows that $T$ defines $M_b$ and its differential $\mathrm{D}_{W} T$ via $f-1$ polynomials $Q_1, \ldots, Q_{f - 1}$ say (one less than $f$, since the trace of $T$ is fixed) plus the polynomials $\{ P_{i} \}_{i=1}^s$ needed to define $\bar{M}_b$. Recall that $\bar{M}_b$ is a smooth manifold and has no singular points. In particular, we have a matrix
	\begin{equation}
		\mathcal{S} 
		= 
		\bigl( 
			\nabla P_1, \ldots, \nabla P_s, \nabla Q_1, \ldots, \nabla Q_{f-1}
		\bigr)
	\end{equation}
	that satisfies the following: if $W \in \mathcal{P}$ is such that $P_i(W) = Q_j(W) = 0$ for all $i, j$, then $\mathcal{S}$ has rank at most $\dim \mathcal{P} - \dim(\bar{M}_b) + f-1$. The set of singular points $\mathrm{Sing}(M_b)$ can be then understood as the set of points $W \in \mathcal{P} \cap M_b$ that are not regular points; that is, the set of points $W \in M_b$ where $\mathcal{S}$ does not have maximal rank. This is a closed Zariski set in the algebraic variety $M_b$.

	Recall now that there exists $W^{*} \in M_b$ such that $\mathrm{rk}(\mathrm{D}_{W^{*}} T) = f-1$ is maximal by assumption. This implies that for $W^{\prime}$ in a neighborhood of $W^{*}$, $\mathrm{D}_{W^{\prime}} T$ has also constant rank $f-1$. Noting that the polynomials $Q_{i}$ for $i=1, \ldots, f-1$ are defined as $f-1$ components of $T(\pi(L)(W))$ up to a constant, we then have that $\mathcal{S}$ has exactly rank $\dim \mathcal{P} - \dim(\bar{M}_b) + f-1$ at $W^{*}$. By \cite[Prop.~3.3.10, (iii) $\to$ (ii)]{bochnak2013real}, there is an irreducible component of $M_b$ of codimension $f-1$ in $\bar{M}_b$ (or of dimension $\dim(\bar{M}_b) - (f - 1)$ in $\mathcal{P}$), and there is a unique component containing $W^{*}$.  By \cite[Prop.~3.3.14]{bochnak2013real}, $\mathrm{Sing}(M_b)$ is then an algebraic set of codimension strictly larger than $f-1$ in $\bar{M}_b$ (see also \cite[\S{6.2}]{smith2004invitation}). Alternatively we can say that $\mathrm{Sing}(M_b)$ is a proper closed Zariski set in $M_b$. Hence, $M_b$ is generically smooth or up to a closed algebraic set of dimension smaller than $M_b$. From the previous computation moreover, $M_{b}$ is then smooth manifold of codimension $f-1$ in $\bar{M}_b$ up to the closed lower dimensional algebraic set $\mathrm{Sing}(M_b)$.
\end{proof}

We next prove that \refAssumption{ass:nonvanishing_maintext} implies that there exists a point $W \in M_b$ such that $\mathrm{rk}( \mathrm{D}_W T ) = f-1$.

\begin{lemma}
	\label{lemma:assumption_holds_if_nonzero}
	Suppose \refAssumption{ass:eigenvalues_different_maintext} holds. 
	Let $W \in \bar{M}_b \cap T^{-1}(\norm{\Sigma^2}_1/f)$ with $(U \Sigma_2 S, S^{\mathrm{T}} \Sigma_1 V)$ an \gls{SVD} for it. If one of the first $\rho$ rows of $S \in O(f)$ has no zeros (\refAssumption{ass:nonvanishing_maintext}), then there exists a point $W \in M_b$ such that $\mathrm{rk}( \mathrm{D}_W T ) = f-1$. 
\end{lemma}

\begin{proof}
	The \gls{SVD} exists by \refLemma{lemma:Specific_SVD_decomposition_of_W_in_Mbbar}. Let
	\begin{equation}
		X 
		= 
		\begin{psmallmatrix}
			0 & X_{12} & \cdots & X_{1 \rho} \\
			- X_{12} & 0 & \cdots &  0 \\
			\vdots & \vdots & & \vdots \\
			- X_{1 \rho} & 0 & \cdots & 0 \\
		\end{psmallmatrix}
		\in \mathrm{Skew}(\R^{\rho \times \rho}),
		\quad
		B
		=
		\begin{psmallmatrix}
			B_{11} & B_{12} & \cdots & B_{1 \rho} \\
			0 & 0 & \cdots & 0 \\
			\vdots & \vdots & & \vdots \\
			0 & 0 & \cdots & 0 \\
		\end{psmallmatrix}
		\in \R^{\rho \times (f-\rho)}.
	\end{equation}
	Hence, the first row of $X$ is the vector $X_{1,\cdot} = (0, X_{1,2:\rho})$ where $X_{1,2:\rho} = (X_{12}, \ldots, X_{1 \rho})\in \R^{\rho -1}$ say, and the first row of $B$ is a vector $B_{1 \cdot} \in \R^{(f-\rho)}$ say. Let $\Sigma^{2} = \mathrm{Diag}(\theta_1, \ldots, \theta_{\rho}) \in \R^{\rho \times \rho}$ be as in \eqref{eqn:Definition_Sigma_squared}, and define $\theta_{2:\rho} = (\theta_2, \ldots, \theta_{\rho}) \in \R^{\rho -1}$. 
	
	Consider now the map
	\begin{align}
		&				
		\mathrm{Diag} 
		\Bigl( 
			S^{\mathrm{T}} 
			\begin{psmallmatrix}
				\Sigma^2 X  & \Sigma^2 B \\
				0_{(f-\rho) \times \rho} & 0_{ (f-\rho) \times (f-\rho)} \\
			\end{psmallmatrix} 
			S 
		\Bigr)
		\nonumber \\ &
		= 
		\mathrm{Diag}
		\bigl( 
			S^{\mathrm{T}} 
			\begin{psmallmatrix}
				\Sigma^2 X & 0_{\rho \times (f-\rho)} \\
				0_{(f-\rho) \times \rho} & 0_{(f-\rho) \times (f-\rho)} \\
			\end{psmallmatrix} 
			S 
		\bigr)
		+ 
		\mathrm{Diag}
		\bigl( 
			S^{\mathrm{T}} 
			\begin{psmallmatrix}
				0_{\rho \times \rho} & \Sigma^2 B \\
				0_{(f-\rho) \times \rho} & 0_{(f-\rho) \times (f-\rho)} \\
			\end{psmallmatrix} 
			S 
		\bigr)
		\nonumber \\ &
		= 
		\mathrm{Diag}
		\Bigl( 
			S^{\mathrm{T}} 
			\begin{psmallmatrix}
				0_{1 \times 1} & \theta_1 X_{1,2:\rho} & 0_{1 \times (f-\rho)} \\
				-( \theta_{2:\rho} \odot X_{1,2:\rho} )^{\mathrm{T}} & 0_{(\rho - 1) \times (\rho - 1) } & 0_{(\rho-1) \times (f-\rho)} \\
				0_{(f-\rho) \times 1} & 0_{(f-\rho) \times (\rho-1)} & 0_{(f-\rho) \times (f-\rho)} \\
			\end{psmallmatrix} 
			S 
		\Bigr) 
		\nonumber \\ &
		\phantom{=}
		+
		\mathrm{Diag}
		\Bigl( 
			S^{\mathrm{T}} 
				\begin{psmallmatrix}
				0_{1 \times 1} & 0_{1 \times (\rho-1)} & \theta_1 B_{1,:} \\
				0_{(\rho-1)\times 1} & 0_{(\rho-1) \times (\rho-1)} & 0_{(\rho-1) \times (f-\rho)} \\
				0_{(f-\rho) \times 1} & 0_{(f-\rho) \times (\rho-1)} & 0_{(f-\rho) \times (f-\rho)} \\
			\end{psmallmatrix} 
			S 
		\Bigr).
	\end{align}
	Observe now that because $\mathrm{Diag}(A) = \mathrm{Diag}(A^{\mathrm{T}})$ for any square matrix $A \in \R^{f \times f}$, we have
	\begin{align}
		&
		\mathrm{Diag}
		\Bigl( 
			S^{\mathrm{T}} 
			\begin{psmallmatrix}
				0_{1 \times 1} & \theta_1 X_{1,2:\rho} & 0_{1 \times (f-\rho)} \\
				-( \theta_{2:\rho} \odot X_{1,2:\rho} )^{\mathrm{T}} & 0_{(\rho - 1) \times (\rho - 1) } & 0_{(\rho-1) \times (f-\rho)} \\
				0_{(f-\rho) \times 1} & 0_{(f-\rho) \times (\rho-1)} & 0_{(f-\rho) \times (f-\rho)} \\
			\end{psmallmatrix} 
			S 
		\Bigr)
		\nonumber \\ &
		=
		\mathrm{Diag}
		\Bigl( 
			S^{\mathrm{T}} 
			\begin{psmallmatrix}
				0_{1 \times 1} & \theta_1 X_{1,2:\rho} - \theta_{2:\rho} \odot X_{1,2:\rho} & 0_{1 \times (f-\rho)} \\
				0_{(f-1) \times 1} & 0_{(f-1) \times (\rho-1)} & 0_{(f-1) \times (f-\rho)} \\
			\end{psmallmatrix} 
			S 
		\Bigr).		
	\end{align}
	Define now $x = ((\theta_1 - \theta_2) X_{12}, \ldots, (\theta_{1} - \theta_{\rho}) X_{1 \rho})$, $b = \theta_1 B_{1,:}$. We have shown that
	\begin{equation}			
		\mathrm{Diag} 
		\Bigl( 
			S^{\mathrm{T}} 
			\begin{psmallmatrix}
				\Sigma^2 X  & \Sigma^2 B \\
				0_{(f-\rho) \times \rho} & 0_{ (f-\rho) \times (f-\rho)} \\
			\end{psmallmatrix} 
			S 
		\Bigr)		
		= 
		\mathrm{Diag}
		\Bigl( 
			S^{\mathrm{T}} 
			\begin{psmallmatrix}
				0_{1 \times 1} & x & b \\
				0_{(f-1) \times 1} & 0_{(f-1) \times (\rho - 1)} & 0_{(f-1) \times (f-\rho)} \\
			\end{psmallmatrix} 
			S 
		\Bigr).
	\end{equation}

	Let $S_{\cdot,1}, \ldots, S_{\cdot,f}$ be the columns of $S$, and denote the $j$-th component of the column $S_{\cdot, i}$ by $S_{ij}$. We have now
	\begin{align}
		&				
		\mathrm{Diag} 
		\Bigl( 
			S^{\mathrm{T}} 
			\begin{psmallmatrix}
				\Sigma^2 X  & \Sigma^2 B \\
				0_{(f-\rho) \times \rho} & 0_{ (f-\rho) \times (f-\rho)} \\
			\end{psmallmatrix} 
			S 
		\Bigr)	
		=		
		\mathrm{Diag}
		\Bigl(
			S^{\mathrm{T}} 
			\begin{psmallmatrix}
				(0, x, b)  \\
				0_{(f-1) \times f} \\
			\end{psmallmatrix} 
			S 
		\Bigr) 
		\nonumber \\ &		
		= 
		\mathrm{Diag}
		\Bigl( 
			\begin{psmallmatrix}
				S_{11}(0, x, b)  \\
				\cdots \\
				S_{1f}(0, x, b)
			\end{psmallmatrix} 
			S 
		\Bigr)  
		= 
		\mathrm{Diag}
		\Bigl( 
			S_{11} \langle S_{\cdot, 1}, (0,x, b) \rangle, \ldots, S_{1f} \langle S_{\cdot, f}, (0,x, b) \rangle
		\Bigr).
		\label{eqn:lemma_map_rank_f-1_needed}
	\end{align}

	Recall that $\theta_i \neq \theta_j$ for $i \neq j$, and note that $x \in \R^{\rho-1}$, $b \in \R^{f-\rho}$ are free variables. To prove \refLemma{lemma:assumption_holds_if_nonzero}, we need to find $x$, $b$ such that the map in \eqref{eqn:lemma_map_rank_f-1_needed} has rank $f-1$. Note that the vector
	\begin{equation}
		\bigl( 
			S_{11} \langle S_{\cdot, 1}, (0,x, b) \rangle, 
			\ldots, 
			S_{1f} \langle S_{\cdot, f}, (0,x, b) \rangle
		\bigr)^{\mathrm{T}}
		=
		(0,x, b) 
		S 
		\mathrm{Diag}( S_{11}, \ldots, S_{1f}).
		\label{eqn:lemma_map_rank_f-1}
	\end{equation}
	The subspace spanned by vectors of the form $(0,x, b) \in \R^{f}$ has dimension $f-1$. Therefore, since $S$ is an orthogonal matrix (note that taking any representative of $S$ from the homogeneous space $\bar{M}_b$ in \eqref{eqn:Definition_Mbbar} also works), we have that the linear map in \eqref{eqn:lemma_map_rank_f-1} has rank $f-1$ if $\mathrm{Diag}( S_{11}, \ldots, S_{1f})$ has maximal rank. This happens whenever $S_{1i} \neq 0$ for all $i=1, \ldots, f$. 
	
	The argument above also works if $s \leq \rho$: consider then instead
	\begin{equation}
		X 
		= 
		\begin{psmallmatrix}
			0 & \cdots & 0 & - X_{s1} & 0 & \cdots & 0 \\
			\vdots & & \vdots & \vdots & \vdots & & \vdots \\
			0 & \cdots & 0 & - X_{s(s-1))} & 0 & \cdots & 0 \\			
			X_{s1} & \cdots & X_{s(s-1)} & 0 & X_{s(s+1))} & \cdots & X_{s\rho}\\
			0 & \cdots & 0 & - X_{s(s+1))} & 0 & \cdots & 0 \\						
			\vdots & & & \vdots & \vdots & & \vdots \\
			0 & \cdots & 0 & - X_{s\rho} & 0 & \cdots & 0 \\
		\end{psmallmatrix}
		,
		\quad
		B
		=
		\begin{psmallmatrix}
			0 & 0 & \cdots & 0 \\
			\vdots & \vdots & & \vdots \\
			0 & 0 & \cdots & 0 \\
			B_{s1} & B_{s2} & \cdots & B_{s \rho} \\	
			0 & 0 & \cdots & 0 \\
			\vdots & \vdots & & \vdots \\
			0 & 0 & \cdots & 0 \\
		\end{psmallmatrix}
		\in \R^{\rho \times (f-\rho)}
		.
	\end{equation}
	The diagonal matrix in \eqref{eqn:lemma_map_rank_f-1} will then turn out to be $\mathrm{Diag}(S_{s1}, \ldots, S_{sf})$ instead. Hence, in case one of the first $\rho$ rows of $S$ has no zeros then the map has rank $f-1$. This concludes the proof.
\end{proof}

In the case that $\rho = 1$, we can let go of \refAssumption{ass:nonvanishing_maintext} and compute $M_b$ exactly. This is implied by the following lemma.

\begin{lemma}
	Suppose \refAssumption{ass:eigenvalues_different_maintext} holds and $\rho = 1$. For any $W = (U \Sigma_2 S, S^{\mathrm{T}} \Sigma_1 V) \in M_b$ that satisfies \eqref{eqn:condition_balancedness_M_b}, it holds that $| S_{1j}|^2 = 1/f$ for $j \in \{ 1, \ldots, f \}$. Furthermore, $M_{b}$ is a union of finitely many points.
	\label{lemma:case_rho_is_one_M_b}
\end{lemma}

\begin{proof}
	We first calculate $S$'s entries. Recall $\Sigma_2$'s and $\Sigma_1$'s expressions in \eqref{eqn:Definition_Sigma1_and_Sigma2}. If $\rho = 1$, then $\Sigma^2 = \Sigma_{\mathcal{W}^*} = \eta \in \R$. Consequently,
	\begin{equation}
		\mathrm{Diag}(W_2^{\mathrm{T}} W_2)
		\eqcom{SVD}= 
		\mathrm{Diag}
		\bigl( 
			S^{\mathrm{T}} 
			\begin{psmallmatrix}
				\eta & 0 \\
				0 & 0 \\
			\end{psmallmatrix} 
			S 
		\bigr) 
		= 
		\mathrm{Diag}
		\bigl(
			|S_{11}|^2 \eta, \cdots, |S_{1f}|^2 \eta 
		\bigr)
		.
		\label{eqn:case_rho_is_one_1}
	\end{equation}
	By \refAssumption{ass:eigenvalues_different_maintext} and \eqref{eqn:condition_balancedness_M_b}, \eqref{eqn:case_rho_is_one_1} must equal $(\eta / f) \mathrm{I}_{f}$. Consequently,
	\begin{equation}
		|S_{11}|^2 = \ldots = |S_{1f}|^2 = \frac{1}{f}.
		\label{eqn:case_rho_is_one_2}
	\end{equation}

	We next show that $M_b$ is a union of finitely many points. Whatever choice for $S_{1, \cdot}$ is made (as long as it satisfies \eqref{eqn:case_rho_is_one_2}), we can then complete the system $\{ S_{1,\cdot}\}$ to an orthonormal basis $\{ S_{1,\cdot}, \ldots, S_{f,\cdot} \}$ say. In other words, this procedure constructs a matrix $S$ that is moreover in $\mathrm{O}(f)$. The resulting $S$ also gives an element of $\bar{M}_b$ and under the quotient $\mathrm{O}(f) \to \mathrm{O}(f)/(\mathrm{1} \oplus \mathrm{O}(f-1))$ all the basis completions of $S$ are equivalent modulo $\mathrm{1} \oplus \mathrm{O}(f-1)$. Note that we can \emph{a priori} choose the signs of the one dimensional `seed' subspace corresponding to the singular value $\Sigma$ in $U$ and $V$. In particular by \refLemma{lemma:assumption_holds_if_nonzero} we obtain that \refAssumption{ass:nonvanishing_maintext} holds in the case that $\rho = 1$. Moreover, for each choice of signs in the vector $S_{1, \cdot}$, we can select a unique disjoint element in $M_b$. Hence, there are only finitely many points in $M_{b}$ when $\rho = 1$.
\end{proof}

\subsection{Proof of \texorpdfstring{\refProposition{prop:hessian_expression}}{} -- Computing \texorpdfstring{$\nabla^2 \mathcal{I}(W)$}{the second derivative of I(W)}}
\label{sec:Proof_of_the_bilinear_form_of_the_Hessian}

We first recall several properties of vectorization in \eqref{eqn:Definition_vectorization}:
\begin{itemize}[topsep=2pt,itemsep=2pt,partopsep=2pt,parsep=2pt,leftmargin=0pt]
	\item[--] If $A, B \in \R^{a \times b}$, then $\mathrm{Tr}[ A^{\mathrm{T}} B ] = \mathrm{vec}(A)^{\mathrm{T}} \mathrm{vec}(B)$. 
	\item[--] If $A \in \R^{e \times f}, B \in \R^{f \times h}$, then $\mathrm{vec}(AB) = B^{\mathrm{T}} \otimes \mathrm{I}_{e \times e} \mathrm{vec}(A) = \mathrm{I}_{h \times h} \otimes A \mathrm{vec}(B)$. Here, $\otimes$ is understood as the Kronecker tensor product compatible with the vectorization $\mathrm{vec}$.
	\item[--] If $A \in \realNumbers^{e \times f}, B \in \realNumbers^{f \times h}, C \in \realNumbers^{h \times g}$, then $\mathrm{vec}(ABC) = (C^{\mathrm{T}} \otimes A) \mathrm{vec}(B)$.  
	\item[--] Let $K: \R^{ab} \to \R^{ba}$ be the linear map such that $\mathrm{vec}(A^{\mathrm{T}}) = K \mathrm{vec}(A)$ holds for any $A \in \R^{a \times b}$. 
\end{itemize}
We also rewrite \eqref{eqn:loss_matrixfact_gradient} using the vectorization notation. Specifically, we have that
\begin{align}
	&
	\mathrm{vec}(\nabla_{1} \mathcal{I}(W)) 
	= \mathrm{vec}(-2W_2^{\mathrm{T}}Y + 2W_2^{\mathrm{T}}W_2W_1) + \mathrm{vec}(2\lambda \mathrm{Diag}(W_2^{\mathrm{T}}W_2)  W_1 )
	\nonumber \\ &
	= -2\mathrm{vec}(W_2^{\mathrm{T}}Y) + 2 \mathrm{I}_{h \times h } \otimes W_2^{\mathrm{T}} W_2 \mathrm{vec}(W_1) + 2 \lambda \mathrm{I}_{h \times h } \otimes \mathrm{Diag}(W_2^{\mathrm{T}} W_2) \mathrm{vec}(W_1),
	\label{eqn:vectorization_hessian1}
\end{align}
and
\begin{align}
	&
	\mathrm{vec}(\nabla_{2} \mathcal{I}(W)) 
	= \mathrm{vec}(-2YW_1^{\mathrm{T}} + 2W_2W_1W_1^{\mathrm{T}}) + \mathrm{vec}(2\lambda W_2 \mathrm{Diag}(W_1W_1^{\mathrm{T}})) 
	\nonumber \\ &
	= -2 \mathrm{vec}(Y W_1^{\mathrm{T}}) + 2  W_1W_1^{\mathrm{T}} \otimes \mathrm{I}_{e \times e } \mathrm{vec}(W_2) + 2 \lambda \mathrm{Diag}(W_1W_1^{\mathrm{T}}) \otimes \mathrm{I}_{e \times e} \mathrm{vec}(W_2).
	\label{eqn:vectorization_hessian2}
\end{align}
As a final preparatory step, let us denote for $i,j \in \{1,2\}$ the partial derivatives with respect to matrices $i$ and $j$ by $\partial_{i,j}$. For example, 
\begin{equation}
	\partial_{1,1} \mathcal{I} 
	\in \R^{(f \times h) \times (f \times h)}
	\quad 
	\textnormal{and}
	\quad
	(\partial_{1,1} \mathcal{I})_{kl, mn} 
	= \frac{ \partial^2 \mathcal{I}(W) }{ \partial W_{1kl} \partial W_{1mn} }.
\end{equation}

\noindent
\emph{Step 1: Calculating the partial derivatives $\partial_{ij} \mathcal{I}(W)$.} 
We start by computing the partial derivatives $\partial_{11} \mathcal{I}, \partial_{22} \mathcal{I}$ directly from \eqref{eqn:vectorization_hessian1}, \eqref{eqn:vectorization_hessian2}. For any vector $v \in \R^n$ and matrix $A \in \R^{n \times n}$, it holds that $\partial(Av)/ \partial v = A$. Therefore
\begin{gather}
	\partial_{1,1} \mathcal{I}(W) 
	= 
	2  \mathrm{I}_{h \times h } \otimes W_2^{\mathrm{T}} W_2 + 2 \lambda \mathrm{I}_{h \times h } \otimes \mathrm{Diag}(W_2^{\mathrm{T}} W_2),
	\nonumber \\
	\partial_{2,2} \mathcal{I}(W) 
	= 
	2  W_1W_1^{\mathrm{T}} \otimes \mathrm{I}_{e \times e } + 2 \lambda \mathrm{Diag}(W_1W_1^{\mathrm{T}}) \otimes \mathrm{I}_{e \times e }.
	\label{eqn:Partial_derivatives__d11IW_d22IW}
\end{gather}

Next we are going to calculate $\partial_{1,2} \mathcal{I}$. We first rewrite terms of $\nabla_1 \mathcal{I}(W)$ and $\nabla_2 \mathcal{I}(W)$ in \eqref{eqn:loss_matrixfact_gradient}. Specifically, note that
\begin{align}
	\mathrm{vec}\bigl( -2W_2^{\mathrm{T}}(Y - W_2W_1) \bigr) 
	& 
	\eqcom{i}= 
	-2 \mathrm{vec}(W_2^{\mathrm{T}} Y) 
	+ 2 W_1^{\mathrm{T}} \otimes W_2^{\mathrm{T}} \mathrm{vec}(W_2),
	\nonumber \\ & 
	\eqcom{ii}= 
	-2((Y - W_2W_1)^{\mathrm{T}} \otimes \mathrm{I}_{f \times f}) K \mathrm{vec}(W_2).
	\label{eqn:step_computation_hessian1}
\end{align}
Here, we have isolated 
(i) $W_2$ by using the identity $\mathrm{vec}(ABC) = C^{\mathrm{T}} \otimes A \mathrm{vec}(B)$; and 
(ii) $W_2^{\mathrm{T}}$ using the tensor $K$ that satisfies $\mathrm{vec}(W_2^{\mathrm{T}}) = K \mathrm{vec}(W_2)$. Similarly, note that
\begin{align}
	\mathrm{vec}( 2\lambda \mathrm{Diag}(W_2^{\mathrm{T}} W_2) W_1) 
	& 
	\eqcom{iii}= \mathrm{vec}( 2\lambda \sum_{i} P_i W_2^{\mathrm{T}} W_2 P_i W_1) \label{eqn:hessian_comp}
	\\ &
	\eqcom{iv}= 2 \lambda \sum_{i} (P_iW_1)^{\mathrm{T}} \otimes P_i W_2^{\mathrm{T}} \mathrm{vec}(W_2) 
	\nonumber \\ & 
	\eqcom{v}= 2 \lambda \sum_{i} ((W_2 P_i W_1)^{\mathrm{T}} \otimes P_i) K \mathrm{vec}(W_2).
	\label{eqn:step_computation_hessian2} 
\end{align}
Here, we 
(iii) utilized the fact that that $\mathrm{Diag}(A) = \sum_{i}^d P_i A P_i$ for some set of symmetric matrices $\{ P_i \}_i$, and then isolated 
(iv) $W_2$ as a vector from \eqref{eqn:hessian_comp} as well as
(v) $W_2^{\mathrm{T}}$ and using the tensor $K$. 

Recall \eqref{eqn:vectorization_hessian1}. We take the derivative of $\mathrm{vec}( \nabla_1 \mathcal{I}(W) )$ with respect to $W_2$ in vectorization notation. While some terms are linear in $W_2$, we use Leibniz's rule on terms including $W_2^{\mathrm{T}} W_2$. Leibniz's rule yields the expressions in \eqref{eqn:step_computation_hessian1}, \eqref{eqn:step_computation_hessian2} resulting in  the tensor
\begin{align}
	\partial_{2,1} \mathcal{I}(W) 
	&
	= -2 \bigl( (Y - W_2 W_1)^{\mathrm{T}} \otimes \mathrm{I}_{f \times f} \bigr) K 
	+ 2W_1^{\mathrm{T}} \otimes W_2^{\mathrm{T}} 
	\nonumber \\ &
	\phantom{=}
	+ 2 \lambda 
		\sum_{i} 
		\Bigl( 
		\bigl( (W_2 P_i W_1)^{\mathrm{T}} \otimes P_i \bigr) K 
		+ (P_i W_1)^{\mathrm{T}} \otimes P_i W_2^{\mathrm{T}} 
		\Bigr).
	\label{eqn:Partial_derivatives__d21IW_d12IW}
\end{align}

\noindent
\emph{Step 2: Evaluation at a vector.} Now that we have the partial derivatives of the Hessian, we want to apply it to vectors of the form $(V_2, V_1) \in \mathrm{T}_{W} \R^{e \times f} \times \R^{f \times h}$. Concretely, we will consider the vectorization of $(V_2, V_1)$ and then compute the elements of the left-hand side of \eqref{eqn:second_derivative_dropout_loss_2layer} one by one. 

First,
\begin{align}
	&
	\mathrm{vec}(V_1)^{\mathrm{T}} \partial_{1,1} \mathcal{I}(W) \mathrm{vec}(V_1) 
	\nonumber \\ &
	\eqcom{\ref{eqn:Partial_derivatives__d11IW_d22IW}}= 
	\mathrm{vec}(V_1)^{\mathrm{T}} 
	\bigl(
		2  \mathrm{I}_{h \times h } \otimes W_2^{\mathrm{T}} W_2 
		+ 2 \lambda \mathrm{I}_{h \times h } \otimes \mathrm{Diag}(W_2^{\mathrm{T}} W_2)
	\bigr) 
	\mathrm{vec}(V_1) 
	\nonumber \\ &
	= 
	\mathrm{vec}(V_1)^{\mathrm{T}} 
	\bigl(
		2  \mathrm{I}_{h \times h } \otimes W_2^{\mathrm{T}} W_2 \mathrm{vec}(V_1) 
		+ 2 \lambda \mathrm{I}_{h \times h } \otimes \mathrm{Diag}(W_2^{\mathrm{T}} W_2) \mathrm{vec}(V_1)
	\bigr) 
	\nonumber \\ &
	= 
	\mathrm{vec}(V_1)^{\mathrm{T}}
	\bigl( 
		2 \mathrm{vec}( W_2^{\mathrm{T}} W_2 V_1) 
		+ 2 \lambda \mathrm{vec}(\mathrm{Diag}(W_2^{\mathrm{T}} W_2) V_1)
	\bigr) 
	\nonumber \\ & 
	= 
	2 \mathrm{Tr}[ V_1^{\mathrm{T}} W_2^{\mathrm{T}} W_2 V_1 ] 
	+ 2 \lambda \mathrm{Tr}[ V_1^{\mathrm{T}} \mathrm{Diag}(W_2^{\mathrm{T}} W_2) V_1 ] 
	\label{eqn:computation_partial_derivatives2}
\end{align}
where we have used that $\mathrm{vec}(B)^{\mathrm{T}} \mathrm{vec}(A) = \mathrm{Tr}[B^{\mathrm{T}}A]$. Similarly
\begin{equation}
	\mathrm{vec}(V_2)^{\mathrm{T}} \partial_{2,2} \mathcal{I}(W) \mathrm{vec}(V_2) 
	\eqcom{\ref{eqn:Partial_derivatives__d11IW_d22IW}}= 
	2 \mathrm{Tr}[ V_2 W_1 W_1^{\mathrm{T}} V_2^{\mathrm{T}} ] 
	+ 2 \lambda \mathrm{Tr}[ V_2 \mathrm{Diag}(W_1 W_1^{\mathrm{T}}) V_2^{\mathrm{T}} ],
\end{equation}
and
\begin{align}
	\mathrm{vec}(V_1)^{\mathrm{T}} \partial_{1,2} \mathcal{I}(W) \mathrm{vec}(V_2) 
	&
	\eqcom{\ref{eqn:Partial_derivatives__d21IW_d12IW}}= 
	-2\mathrm{Tr}[ V_1^{\mathrm{T}} V_2^{\mathrm{T}}(Y - W_2 W_1) ] 
	+ 2\mathrm{Tr}[ V_1^{\mathrm{T}} W_2^{\mathrm{T}} V_2 W_1 ] 
	\\ &
	\phantom{=} 
	+ 2 \lambda 
	\bigl( 
		\mathrm{Tr}[ V_1^{\mathrm{T}} \mathrm{Diag}(V_2^{\mathrm{T}}W_2) W_1 ]
		+ \mathrm{Tr}[ V_1^{\mathrm{T}} \mathrm{Diag}(W_2^{\mathrm{T}} V_2) W_1 ] 
	\bigr).
	\nonumber 
\end{align}
We also have that $\partial_{2,1} \mathcal{I}(W)  = \partial_{1,2} \mathcal{I}(W)^{\mathrm{T}}$ because $\mathcal{I}$ is a smooth function. Therefore 
\begin{equation}
	\mathrm{vec}(V_2)^{\mathrm{T}} \partial_{2,1} \mathcal{I}(W) \mathrm{vec}(V_1) 
	= \mathrm{vec}(V_1)^{\mathrm{T}} \partial_{1,2} \mathcal{I}(W) \mathrm{vec}(V_2). 
	\label{eqn:vecV2Tpartial21IvecV2_equals_vecV1Tpartial12IvecV2}
\end{equation}

Adding \eqref{eqn:Partial_derivatives__d21IW_d12IW}--\eqref{eqn:vecV2Tpartial21IvecV2_equals_vecV1Tpartial12IvecV2} yields:
\begin{align}
	&
	\bigl( \mathrm{vec}(V_1), \mathrm{vec}(V_2) \bigr)^{\mathrm{T}} 
	\nabla^2 \mathcal{I}(W) 
	\bigl(\mathrm{vec}(V_1), \mathrm{vec}(V_2) \bigr) 
	\nonumber \\ &
	= 
	2 \mathrm{Tr}[ V_1^{\mathrm{T}} W_2^{\mathrm{T}} W_2 V_1 ] 
	+ 2 \lambda \mathrm{Tr}[ V_1^{\mathrm{T}} \mathrm{Diag}(W_2^{\mathrm{T}} W_2) V_1 ]
	\nonumber \\ &
	\phantom{=}
	+ 2 \mathrm{Tr}[ V_2 W_1 W_1^{\mathrm{T}} V_2^{\mathrm{T}} ] 
	+ 2 \lambda \mathrm{Tr}[ V_2 \mathrm{Diag}(W_1 W_1^{\mathrm{T}}) V_2^{\mathrm{T}} ]
	\nonumber \\ &
	\phantom{=}
	- 4 \mathrm{Tr}[ V_1^{\mathrm{T}} V_2^{\mathrm{T}}(Y - W_2 W_1) ] 
	+ 4 \mathrm{Tr}[ V_1^{\mathrm{T}} W_2^{\mathrm{T}} V_2 W_1 ]
	\nonumber \\ &
	\phantom{=}
	+ 4 \lambda 
	\bigl( 
		\mathrm{Tr}[ V_1^{\mathrm{T}} \mathrm{Diag}(V_2^{\mathrm{T}}W_2) W_1 ] 
		+ \mathrm{Tr}[ V_1^{\mathrm{T}} \mathrm{Diag}(W_2^{\mathrm{T}} V_2) W_1 ] 
	\bigr).
	\label{eqn:expansion_second_derivative}
\end{align}

Finally, note that
\begin{align}
	2 \pnorm{W_2 V_1 + V_2 W_1}{\mathrm{F}}^2 
	&
	= 
	2 \mathrm{Tr}[ (W_2 V_1 + V_2 W_1)^{\mathrm{T}}(W_2 V_1 + V_2 W_1) ]
	\nonumber \\ 
	&
	= 
	2 \mathrm{Tr}[ V_1^{\mathrm{T}} W_2^{\mathrm{T}} W_2 V_1 ] 
	+ 2 \mathrm{Tr}[W_1^{\mathrm{T}} V_2^{\mathrm{T}} V_2 W_1 ] 
	\nonumber \\
	& \phantom{= 
	2 \mathrm{Tr}[ V_1^{\mathrm{T}} W_2^{\mathrm{T}} W_2 V_1 ]} + 2 \mathrm{Tr}[ V_1^{\mathrm{T}} W_2^{\mathrm{T}} V_2 W_1 ] + 2 \mathrm{Tr}[ W_1^{\mathrm{T}} V_2^{\mathrm{T}} W_2 V_1 ]  \nonumber \\
	&
	= 
	2 \mathrm{Tr}[ V_1^{\mathrm{T}} W_2^{\mathrm{T}} W_2 V_1 ] 
	+ 2 \mathrm{Tr}[ V_2 W_1 W_1^{\mathrm{T}} V_2^{\mathrm{T}} ] 
	+ 4 \mathrm{Tr}[ V_1^{\mathrm{T}} W_2^{\mathrm{T}} V_2 W_1 ];
	\label{eqn:Frobenius_norm_squared_of_W2V1plusV1W2_in_terms_of_traces}
\end{align}
where in the last equality we have used the cyclic property of the trace. Now, for any $A, B \in \R^{n}$, $\pnorm{ A + B}{\mathrm{F}}^2 - \pnorm{ A - B}{\mathrm{F}}^2 = 4 \langle A, B \rangle$, so that
\begin{align}
	&
	2
	\bigl( 
		\pnorm{\mathrm{Diag}(V_2^{\mathrm{T}} W_2) + \mathrm{Diag}(W_1 V_1^{\mathrm{T}})}{\mathrm{F}}^2 
		- \pnorm{\mathrm{Diag}(V_2^{\mathrm{T}} W_2) - \mathrm{Diag}(W_1 V_1^{\mathrm{T}})}{\mathrm{F}}^2
	\bigr) 
	\nonumber \\ & 
	= 
	8 
	\bigl\langle 
		\mathrm{Diag}(W_1V_1^{\mathrm{T}}), 
		\mathrm{Diag}(V_2^{\mathrm{T}}W_2) 
	\bigr\rangle 
	= 
	8 \mathrm{Tr}[ \mathrm{Diag}(W_1V_1^{\mathrm{T}})\mathrm{Diag}(V_2^{\mathrm{T}}W_2) ] 
	\nonumber \\ & 
	\eqcom{vi}= 
	8 \mathrm{Tr}[ W_1V_1^{\mathrm{T}} \mathrm{Diag}(V_2^{\mathrm{T}}W_2) ] 
	\eqcom{vii}= 
	8 \mathrm{Tr}[ V_1^{\mathrm{T}} \mathrm{Diag}(V_2^{\mathrm{T}}W_2)W_1 ].
	\label{eqn:Inner_product_relation_Frobenius_norm_applied_to_DiagV2TW2plusDiagW1V1T}
\end{align}
Here, we have used 
(vi) that 
$
	\mathrm{Tr}[ A \mathrm{Diag}(B) ] 
	= \mathrm{Tr}[ \mathrm{Diag}(A)\mathrm{Diag}(B)]
$ for any $A, B$ square matrices of the same dimension,
and 
(vii) the cyclic property of the trace. 
Substituting \eqref{eqn:Frobenius_norm_squared_of_W2V1plusV1W2_in_terms_of_traces} and \eqref{eqn:Inner_product_relation_Frobenius_norm_applied_to_DiagV2TW2plusDiagW1V1T} into \eqref{eqn:expansion_second_derivative} completes the proof.
\qed

\subsection{Proof \texorpdfstring{of \refProposition{prop:Lower_bound_to_the_Hessian_restricted_to_directions_normal_to_the_manifold_of_minima}}{of the Hessian's lower bound}}
\label{sec:Proof_of_the_lower_bound_to_the_Hessian_restricted_to_directions_normal_to_the_manifold_of_minima}

\subsubsection{Obtaining \texorpdfstring{$\mathrm{T}_{W} M$}{the tangent space of M at W}}

We compute first $\mathrm{T}_{W} M_b$ in \refLemma{lemma:tangent_space_M_b}, which we will use to compute $\mathrm{T}_{W} M$ later:

\begin{lemma}
	If \refAssumption{ass:eigenvalues_different_maintext} holds, then for any $W \in M_b \backslash \mathrm{Sing}(M_b)$
	\begin{align}
		\mathrm{T}_{W} M_b 
		= 
		\Bigl\{ 
			&			
			\bigl( 
				U 
				\begin{psmallmatrix}
					\Sigma X & \Sigma E \\
					0_{(e-\rho) \times \rho} & 0_{(e-\rho) \times (f-\rho)} \\
				\end{psmallmatrix} 
				S 
				,
				S^{\mathrm{T}} 
				\begin{psmallmatrix}
					X^{\mathrm{T}} \Sigma & 0_{\rho \times (h-\rho)} \\
					E^{\mathrm{T}} \Sigma & 0_{(f-\rho) \times (h-\rho)} \\
				\end{psmallmatrix} 
				V 
			\bigr) 
			\nonumber
			\\ &
			: 			
			X \in \mathrm{Skew}(\R^{\rho \times \rho}), 
			E \in \R^{\rho \times (f - \rho)}, 
			(X, E) \in \ker \mathrm{D}_W T 
		\Bigr\}. 
		\label{eqn:tangent_M_b}
	\end{align}
	\label{lemma:tangent_space_M_b}
\end{lemma}

\begin{proof}
	Let $W = (U \Sigma_2 S, S^{\mathrm{T}} \Sigma_1 V ) \in M_b \backslash \mathrm{Sing}(M_b)$; such \gls{SVD} exists by \refLemma{lemma:Specific_SVD_decomposition_of_W_in_Mbbar}. By \refProposition{prop:Characterization_of_TWMb}, $\mathrm{T}_W M_{b} = \ker \mathrm{D}_W T$ where $\mathrm{D}_W T: \mathrm{T}_W \bar{M}_b  \to \mathrm{T}_{T(W)}\R^{f}$. Next, write
	\begin{align}
		&
		\ker \mathrm{D}_W T 
		= \{ (V_2,V_1) \in \mathrm{T}_W \bar{M}_b : \mathrm{D}_W T(V_2,V_1) = 0 \}
		\nonumber \\ &
		\eqcom{\ref{eqn:Bilinear_form_Mb_implicit_M_Of},\ref{eqn:Tangent_space_at_W_of_Mbbar}}= 
		\Bigl\{ 
			\Bigl( 
				U \Sigma_2 
				\begin{psmallmatrix}
				X & E \\
				-E^{\mathrm{T}} & 0
				\end{psmallmatrix} 
				S,
				S^{\mathrm{T}} 
				\begin{psmallmatrix}
				X^{\mathrm{T}} & -E \\
				E^{\mathrm{T}} & 0 \\
				\end{psmallmatrix} 
				\Sigma_1 V 
			\Bigr) 
			:
			\nonumber \\ &
			X \in \mathrm{Skew}(\R^{\rho \times \rho}), 
			E \in \R^{\rho \times (f -\rho)},		
			2 \mathrm{Diag}
			\bigr( 
				S 
				\begin{psmallmatrix}
				\Sigma^2 X & \Sigma^2 E \\
				0 & 0 \\
				\end{psmallmatrix} 
				S 
			\bigr)
			= 0
		\Bigr\}
		. 
		\label{eqn:Intermediate__kerDWT}
	\end{align}
	Hence, from the bilinear form $\mathrm{D}_{W} T$ defined in \eqref{eqn:Bilinear_form_Mb_implicit_M_Of} in \refProposition{prop:Characterization_of_TWMb}(b), we take the pairs $(X, E) \in \mathrm{Skew}(\R^{\rho \times \rho}) \times \R^{\rho \times (f - \rho)}$ that also belong to $\ker \mathrm{D}_W T$.
	The last step required to arrive at \eqref{eqn:tangent_M_b} is to substitute the definitions of $\Sigma_2$ and $\Sigma_1$, recall \eqref{eqn:Definition_Sigma1_and_Sigma2}, into \eqref{eqn:Intermediate__kerDWT}.
\end{proof}

Observe from \eqref{eqn:Definition_Mb} and \eqref{eqn:Action_pi} that, under the action of $H = (\R^{*})^{f}$, we always have $\pi(H)(M_b) \subseteq M$. \refProposition{prop:exists_diagonal_reduction_Mb_to_M} implies that $M \subseteq \pi(H)(M_b)$, and hence $\pi(H)(M_b) = M$. \refProposition{prop:exists_diagonal_reduction_Mb_to_M} also yields that the group action is free and so the map $\pi: H \times M_b \to M$ is bijective. We have moreover that on the open set $\pi(H \times M_b)$, $\pi$ has a continuous inverse given by
\begin{equation}
	\pi^{-1}(W)= (C_{W},\pi(C_{W})(W)) .
	\label{eqn:definition_inverse_action_pi}
\end{equation}
Here, $C_{W} = \mathrm{Diag}(W_1W_1^{\mathrm{T}})^{1/4} \mathrm{Diag}(W_2W_2^{\mathrm{T}})^{-1/4}$, which is discussed in the proof of \refProposition{prop:exists_diagonal_reduction_Mb_to_M}.
If $\pi$ is smooth, this allows us to obtain the tangent space of $\mathrm{T}_W M$ at every point $W = \pi(C)(W^{\prime}) \in M$ such that $W^{\prime} \in M_{b} \backslash \mathrm{Sing}(M_b)$.

For every point $W \in M_b \backslash \mathrm{Sing}(M_b)$, the action $\pi$ restricted to a smooth neighborhood $R_{\mathrm{Id}} \times U_{W} \subset H \times M_b$ is a map $\mathrm{D}_{(\mathrm{Id}, W)} \pi: \mathcal{H} \times \mathrm{T}_{W} \mathcal{P} \to \mathrm{T}_{W} \mathcal{P}$ with $\mathcal{H} = \mathrm{T}_{\mathrm{Id}} H = \mathrm{Lie}((\R^{*})^f)$ the Lie algebra of $H$. Furthermore, for every point $W \in M_b$, the differential $\mathrm{D}_{(\mathrm{Id},W)} \pi (D, V)$ at $D \in \mathcal{H}, V \in \mathrm{T}_{W} M_b$ is given by
\begin{align}
	\mathrm{D}_{(\mathrm{Id},W)} \pi (0, V) 
	& 
	\eqcom{\ref{eqn:Action_pi}}= V 
	,
	\nonumber \\
	\mathrm{D}_{(\mathrm{Id},W)} \pi (D, 0) 
	& 
	\eqcom{\ref{eqn:Action_pi}}= (W_2D, -DW_1)
	\eqcom{\ref{eqn:Alternative_representation_of_Mb}}=
	\bigl(
	U
	\begin{psmallmatrix}
		\Sigma & 0 \\
		0 & 0 \\
	\end{psmallmatrix}
	S D,
	-D S^{\mathrm{T}}
	\begin{psmallmatrix}
		\Sigma & 0 \\
		0 & 0 \\
	\end{psmallmatrix}
	\Sigma_1 V
	\bigr)
	\label{eqn:Differential_at_W_of_pi}
\end{align}
say. For every point $W \in M_b$, we will define the vector space
\begin{equation}
	\mathrm{D}_W \pi(\mathcal{H}) 
	= \{ \mathrm{D}_{(\mathrm{Id},W)} \pi (D, 0) : D \in \mathcal{H} \}.
	\label{eqn:definition_lie_algebra_directions_tangent_space} 
\end{equation}
Recall now finally that for $V = (V_2, V_1)$ and $R=(R_2, R_1) \in \mathrm{T}_{W} \mathcal{P}$ we have the Euclidean inner product $ \langle \cdot, \cdot \rangle : \mathrm{T}_{W} \mathcal{P} \times \mathrm{T}_{W} \mathcal{P} \to \R$ defined as $\langle V, R \rangle = \langle V_2, R_2 \rangle_{\mathrm{F}} + \langle V_1, R_1 \rangle_{\mathrm{F}}$.

We are now in position to prove the following:

\begin{lemma}
	Suppose Assumptions~\ref{ass:eigenvalues_different_maintext}, \ref{ass:nonvanishing_maintext} hold. Let $\pi$ be the Lie group action of $H$ on $M$ defined in \eqref{eqn:Action_pi} and $\mathcal{H} = \mathrm{Lie}((\R^{*})^f)$. If $W = (U \Sigma_2 S, S^{\mathrm{T}} \Sigma_1 V) \in M_b \backslash \mathrm{Sing}(M_b)$, then
	\begin{align}
		&
		\mathrm{T}_W M 
		\label{eqn:Tangent_space_at_W_of_M}		
		\\ &
		= 
		\mathrm{T}_{W} M_b \oplus \mathrm{D}_W \pi(\mathcal{H}) 
		\nonumber \\ &
		= 
		\Bigl\{ 
			\Bigl( 
			U 
			\begin{psmallmatrix}
				\Sigma X & \Sigma B \\
				0_{(e-\rho) \times \rho} & 0_{(e-\rho) \times (f-\rho)} \\
			\end{psmallmatrix} 
			S,
			S^{\mathrm{T}} 
			\begin{psmallmatrix}
				X^{\mathrm{T}} \Sigma & 0_{\rho \times (h-\rho)} \\
				B^{\mathrm{T}} \Sigma & 0_{(f-\rho) \times (h-\rho)} \\
			\end{psmallmatrix} 
			\Sigma_1 V 
			\Bigr) 
			\nonumber \\ &
			\phantom{= \Bigl\{}			
			:
			X \in \mathrm{Skew}(\R^{\rho \times \rho}), B \in \R^{\rho \times (f - \rho)}, (X, B) \in \ker \mathrm{D}_W T 
		\Bigr\} 
		\nonumber \\ &
		\phantom{=}
		\oplus 
		\Bigl\{
			\Bigl( 
			U 
			\begin{psmallmatrix}
				\Sigma & 0_{\rho \times (f-\rho)} \\
				0_{(e-\rho) \times \rho} & 0_{(e-\rho) \times (f-\rho)} \\
			\end{psmallmatrix} 
			S D,
			-D S^{\mathrm{T}} 
			\begin{psmallmatrix}
				\Sigma & 0_{\rho \times (h-\rho)} \\
				0_{(f-\rho) \times \rho} & 0_{(f-\rho) \times (h-\rho)} \\
			\end{psmallmatrix} 
			V 
			\Bigr) 
			:
			D \in \mathrm{Diag}(\R^{f \times f}) 
		\Bigr\}.
		\nonumber
	\end{align}
	\label{lemma:Tangent_space_at_W_of_M}
\end{lemma}

\begin{proof}
	Let $W = (U \Sigma_2 S, S^{\mathrm{T}} \Sigma_1 V) \in M_b \backslash \mathrm{Sing}(M_b)$. Start by noting that \refProposition{prop:Characterization_of_TWMb} implies that $\mathrm{T}_W M_{b} = \ker \mathrm{D}_{W} T$. Here, we understand that $\ker \mathrm{D}_{W} T \subseteq \mathrm{T}_{W} \bar{M}_b$. 
	In order to expand this result to $\mathrm{T}_{W} M$, we will use the smooth action of $H = (\R^{*})^{f}$ on $M_b$. 
		
	Let $R_{\mathrm{Id}} \times U_{W} \subset H \times M_b$ be a neighborhood such that we can compute the differential of $\pi$ in \eqref{eqn:Differential_at_W_of_pi}. Note by combining \eqref{eqn:tangent_M_b}, \eqref{eqn:Differential_at_W_of_pi} and \eqref{eqn:definition_lie_algebra_directions_tangent_space} that for any $K = (K_2, K_1) \in \mathrm{T}_{W} M_b$ and $Q = (Q_2, Q_1) \in \mathrm{D}_{W} \pi(\mathcal{H})$ we have that $\langle K, Q \rangle = 0$. In other words, $\mathrm{T}_{W} M_b$ is orthogonal to $\mathrm{D}_{W} \pi(\mathcal{H})$. Hence, the sum of the subspaces $\mathrm{T}_{W} M_b$ and $\mathrm{D}\pi( \mathcal{H})$ is orthogonal. We also know that $R_{\mathrm{Id}} \times U_{W}$ is a smooth submanifold of $H \times M_b$ and $\pi$ is smooth, bijective and with continuous inverse. Therefore, by dimension counting, we must have that $\pi$ is a local diffeomorphism and so $\mathrm{T}_{W} M = \mathrm{T}_{W} M_{b} \oplus \mathrm{D}_W \pi( \mathcal{H})$. 

	To arrive at the expression in \eqref{eqn:Tangent_space_at_W_of_M}, we simply use the expressions for $\mathrm{T}_{W} M_b$ from \refLemma{lemma:tangent_space_M_b} together with \eqref{eqn:Differential_at_W_of_pi}.
\end{proof}

\subsubsection{Obtaining \texorpdfstring{$\mathrm{T}^{\perp}_{W} M$}{the cotangent space}}
\label{sec:Obtaining_cotangent_space_of_TWM}

We now compute the cotangent space $\mathrm{T}^{\perp}_{W} M$ by embedding $\mathrm{T}^{\perp}_{W} M \subset \mathrm{T}_{W} \mathcal{P}$ and obtaining the orthogonal complement of $\mathrm{T}_{W} M$.

\begin{lemma}
	Suppose Assumptions~\ref{ass:eigenvalues_different_maintext}, \ref{ass:nonvanishing_maintext} hold.
	For $W = (U \Sigma_2 S, S^{\mathrm{T}} \Sigma_1 V) \in M \cap M_{b} \backslash \mathrm{Sing}(M)$,
	\begin{align}
		\mathrm{T}^{\perp}_W M 
		= 
		\Bigl\{ 
			&
			(K_2, K_1) 
			= \bigl( 
				U 
				\begin{psmallmatrix}
				A_2 & B_2 \\
				C_2 & D_2 \\
				\end{psmallmatrix} 
				S,
				S^{\mathrm{T}} \begin{psmallmatrix}
				A_1 & B_1 \\
				C_1 & D_1 \\
				\end{psmallmatrix} 
				V 
			\bigr) 
			\in \mathrm{T}_W \mathcal{P}
			: 
			\nonumber \\ & 
			X \in \mathrm{Skew}(\R^{\rho \times \rho}), 
			B \in \R^{\rho \times (f - \rho)}, 
			\nonumber \\ & 
			\langle \Sigma (A_2 + A_1^{\mathrm{T}} ), X \rangle + \langle \Sigma (B_2 + C_1^{\mathrm{T}}), B \rangle 
			= 0, 
			\nonumber \\ & 
			\mathrm{Diag}(K_2^{\mathrm{T}} W_2) 
			= \mathrm{Diag}(K_1 W_1^{\mathrm{T}}), 
			2 \mathrm{Diag}
			\bigr( 
				S 
				\begin{psmallmatrix}
					\Sigma^2 X & \Sigma^2 B \\
					0_{(f-\rho) \times \rho} & 0_{(f-\rho) \times (f-\rho)} \\
				\end{psmallmatrix} 
				S 
			\bigr)
			= 0 
		\Bigr\}.
		\label{eqn:Cotangent_space_of_M}
	\end{align}
	\label{lemma:tangent_cotangent_M}
\end{lemma}

\begin{proof}
Let $W = (U \Sigma_2 S, S^{\mathrm{T}} \Sigma_1 V) \in M \cap M_b\backslash \mathrm{Sing}(M)$. Taking the orthogonal complement in \eqref{eqn:Tangent_space_at_W_of_M}, we obtain that 
\begin{equation}
	\mathrm{T}^{\perp}_W M 
	= 
	\mathrm{T}^{\perp}_{W} M_b \cap (\mathrm{D}{\pi}(\mathcal{H}))^{\perp}.
	\label{eqn:Cotangent_space_of_M_is_the_interaction_of_two_cotangent_spaces}
\end{equation} 
We will now determine both subspaces in the right-hand side of \eqref{eqn:Cotangent_space_of_M_is_the_interaction_of_two_cotangent_spaces}. Taking the intersection of these two sets will then immediately result in \eqref{eqn:Cotangent_space_of_M}.

\noindent
\emph{Determining $\mathrm{T}_{W}^{\perp} M_b$.} 
Recall first the definition of a cotangent space, that is
\begin{equation}
	\mathrm{T}^{\perp}_W M_b
	= \bigl\{ K \in \mathrm{T}_{W} \mathcal{P} = \mathrm{T}_{W_2} \R^{e \times f} \times \mathrm{T}_{W_1} \R^{f \times h} 
	: 
	\forall R \in \mathrm{T}_{W} M_b, \langle K, R \rangle = 0 \bigr\}.
	\label{eqn:Arbitrary_cotangent_space_at_W}
\end{equation}
Furthermore, note that for any $K = (K_2, K_1) \in T_W \mathcal{P}$, there exist matrices $A_1, A_2 \in \R^{\rho \times \rho}$ and matrices $B_2, B_1, C_2, C_1, D_2, D_1$ of appropriate dimensions such that 
\begin{equation}
	K 
	= \Bigl( 
		U 
		\begin{psmallmatrix}
		A_2 & B_2 \\
		C_2 & D_2 \\
		\end{psmallmatrix} 
		S,
		S^{\mathrm{T}} 
		\begin{psmallmatrix}
		A_1 & B_1 \\
		C_1 & D_1 \\
		\end{psmallmatrix} 
		V 
	\Bigr).
	\label{eqn:Representation_of_any_K_in_TWP}
\end{equation}
This is because $U$, $S$, and $V$ are orthogonal matrices and thus
\begin{equation}
	(a,b) \in \mathrm{T}_{W} \mathcal{P}
	\Rightarrow
	(U a S, S^{\mathrm{T}} b V) \in \mathrm{T}_{W} \mathcal{P}.
\end{equation}	

We now investigate the inner product condition in \eqref{eqn:Arbitrary_cotangent_space_at_W}. \refLemma{lemma:tangent_space_M_b} implies that if $R \in \mathrm{T}_{W} M_b$, then there exist $(X,E) \in \mathrm{Skew}(\realNumbers^{\rho \times \rho}) \times \realNumbers^{\rho \times (f-\rho)}$ such that
\begin{equation}
	R 
	= (R_2, R_1)
	= \Bigl( 
		U 
		\begin{psmallmatrix}
			\Sigma X & \Sigma E \\
			0 & 0 \\
		\end{psmallmatrix} 
		S,
		S^{\mathrm{T}} 
		\begin{psmallmatrix}
			X^{\mathrm{T}} \Sigma & 0 \\
			E^{\mathrm{T}}\Sigma & 0 \\
		\end{psmallmatrix} 
		V 
	\Bigr)
	\quad
	\textnormal{and}
	\quad
	2 \mathrm{Diag}
	\bigr( 
		S 
		\begin{psmallmatrix}
			\Sigma^2 X & \Sigma^2 E \\
			0 & 0 \\
		\end{psmallmatrix} 
		S 
	\bigr)
	= 0.
	\label{eqn:Representation_of_any_R_in_TWMb}
\end{equation}
For any $K \in \mathrm{T}_W \mathcal{P}$, the inner product condition in \eqref{eqn:Arbitrary_cotangent_space_at_W} reduces to 
\begin{align}
	0 
	= \langle K, R \rangle 
	&
	= \langle K_2, R_2 \rangle + \langle K_1, R_1 \rangle 
	\nonumber \\ & 
	\eqcom{i}= 
	\langle A_2 , \Sigma X \rangle 
	+ \langle B_2 , \Sigma E \rangle 
	+ \langle A_1 , X^{\mathrm{T}} \Sigma \rangle 
	+ \langle C_1 , E^{\mathrm{T}} \Sigma \rangle 
	\nonumber \\ & 
	\eqcom{ii}= 
	\langle \Sigma A_2 , X \rangle 
	+ \langle  A_1 \Sigma , X^{\mathrm{T}}  \rangle  
	+ \langle \Sigma B_2 , E \rangle 
	+ \langle  C_1 \Sigma , E^{\mathrm{T}} \rangle 
	\nonumber \\ & 
	\eqcom{iii}= 
	\langle \Sigma (A_2 + A_1^{\mathrm{T}}) , X \rangle 
	+ \langle \Sigma (B_2 + C_1^{\mathrm{T}}), E \rangle.
\end{align}
Here, we used (i) the representations in \eqref{eqn:Representation_of_any_K_in_TWP} and \eqref{eqn:Representation_of_any_R_in_TWMb}, (ii) that $\langle M, \Sigma N \rangle = \langle \Sigma M,  N \rangle$ for any matrices $M, N$ of appropriate size because $\Sigma$ is diagonal, and (iii) that $\langle M, N \rangle = \langle M^{\mathrm{T}}, N^{\mathrm{T}} \rangle$ for any matrices $M, N$ of the same size. 

Summarizing, we have that
\begin{align}
	T^{\perp}_W M_b 
	= \Big\{ & 
	(K_2,K_1)
	=
	\bigl( 
		U 
		\begin{psmallmatrix}
		A_2 & B_2 \\
		C_2 & D_2 \\ 
		\end{psmallmatrix} 
		S,
		S^{\mathrm{T}} 
		\begin{psmallmatrix}
		A_1 & B_1 \\
		C_1 & D_1 \\
		\end{psmallmatrix} 
		V 
	\bigr) 
	\in \mathrm{T}_W \mathcal{P}
	: 
	\nonumber \\ & 
	\langle \Sigma (A_2 + A_1^{\mathrm{T}}) , X \rangle 
	+ \langle \Sigma (B_2 + C_1^{\mathrm{T}}) , E \rangle 
	= 0, 
	\nonumber \\ &
	X \in \mathrm{Skew}(\R^{\rho \times \rho}), 
	E \in \R^{\rho \times (f - \rho)}, 
	2 \mathrm{Diag}
	\bigr( 
		S 
		\begin{psmallmatrix}
		\Sigma^2 X & \Sigma^2 E \\
		0 & 0 \\
		\end{psmallmatrix} 
		S 
	\bigr)
	= 0
	\Big\}.
	\label{eqn:Cotangent_space_of_Mb}
\end{align}

\noindent
\emph{Determining $(\mathrm{D}{\pi}(\mathcal{H}))^{\perp}$.}
Recall the definition of an orthogonal complement, that is
\begin{equation}
	(\mathrm{D}_W{\pi}(\mathcal{H}))^{\perp}
	= \bigl\{ 
		K \in \mathrm{T}_W \mathcal{P}
		: 
		\langle K, \mathrm{D}_W \pi(D)(W) \rangle = 0 
		\, \forall D \in \mathrm{Diag}(\R^f)
	\bigr\}
	.
	\label{eqn:Orthogonal_complement_of_DWpiH}
\end{equation}

We now investigate the inner product condition in \eqref{eqn:Orthogonal_complement_of_DWpiH}. For any $K \in \mathrm{T}_W \mathcal{P}, D \in \mathcal{H} = \mathrm{Diag}(\R^f)$, recalling \eqref{eqn:Differential_at_W_of_pi}, this condition reduces to
\begin{align}
	0 
	= 
	\langle K, \mathrm{D}_W \pi(D)(W) \rangle 
	&
	= 
	\langle K_2, W_2 D \rangle 
	+ \langle K_1, - D W_1 \rangle  
	\nonumber \\ & 
	\eqcom{i}= 
	\langle W_2^{\mathrm{T}} K_2, 
	D \rangle 
	- \langle K_1 W_1^{\mathrm{T}}, D \rangle 
	= 
	\langle W_2^{\mathrm{T}} K_2 - K_1 W_1^{\mathrm{T}}, D \rangle
	.
	\label{eqn:Intermediate__Inner_product_of_K_and_DpiHW}
\end{align}
Here, we used (i) that $\langle M, NO \rangle = \langle N^{\mathrm{T}} M, O \rangle = \langle M O^{\mathrm{T}}, N \rangle$ for any matrices $M, N, O$ with compatible dimensions. Now, because \eqref{eqn:Intermediate__Inner_product_of_K_and_DpiHW} holds for any $D \in \mathrm{Diag}(\R^f)$, we must have that 
\begin{equation}
	\mathrm{Diag}( W_2^{\mathrm{T}}K_2 - K_1 W_1^{\mathrm{T}} ) 
	= 0.
\end{equation} 

Summarizing, we have that
\begin{equation} 
	(\mathrm{D}{\pi}(\mathcal{H}))^{\perp} 
	= \{ 
		(K_2, K_1) 
		\in T_{W} \mathcal{P} 
		: 
		\mathrm{Diag}(K_1 W_1^{\mathrm{T}}) 
		= \mathrm{Diag}(W_2^{\mathrm{T}} K_2) 
	\}.
	\label{eqn:Cotangent_space_of_DpiH}
\end{equation}

\noindent
\emph{Concluding.}
As mentioned before, taking the intersection of \eqref{eqn:Cotangent_space_of_Mb} and \eqref{eqn:Cotangent_space_of_DpiH} results in \eqref{eqn:Cotangent_space_of_M}. This completes the proof.
\end{proof}

\subsubsection{\texorpdfstring{Lower bound of $\nabla^2 \mathcal{I}(W)$ restricted to $\mathrm{T}^{\perp}_W M$}{Lower bound of the Hessian in the directions orthogonal to the manifold of minima}}

We require the following lemma. This will be used in an optimization problem we encounter when looking for a lower bound to $\nabla^2 \mathcal{I}(W)|_{T^{\perp}_W M}$.

\begin{lemma}
Suppose \refAssumption{ass:eigenvalues_different_maintext} holds. Let
\begin{align}
	\mathcal{K} 
	&
	= 
	\Bigl\{ 
		\bigl( 
			U 
			\begin{psmallmatrix} 
				A_2 & B_2 \\
				0_{(e-\rho) \times \rho} & 0_{(e-\rho) \times (f-\rho)} \\
			\end{psmallmatrix} 
			S
			,
			S^{\mathrm{T}} 
			\begin{psmallmatrix}
				A_1 & 0_{\rho \times (h-\rho)} \\
				C_1 & 0_{(f-\rho) \times (h-\rho)} \\
			\end{psmallmatrix} 
			V 
		\bigr) 
		\in \mathrm{T}^{\perp}_{W} M
		\nonumber \\ &
		\phantom{= \Bigl\{}
		: 
		A_1, A_2 \in \R^{\rho \times \rho},
		B_2, C_1^{\mathrm{T}} \in \R^{\rho \times (f-\rho)}
	\Bigr\}
	.
	\label{eqn:subspace_V}
\end{align}
For any $W = ( U \Sigma_2 S, S^{\mathrm{T}} \Sigma_1 V ) \in M_b \backslash \mathrm{Sing}(M_b)$,
the following holds:
if 
\begin{equation}
	K 
	= 
	\bigl( 
		U 
		\begin{psmallmatrix} 
			A_2 & B_2 \\
			0_{(e-\rho) \times \rho} & 0_{(e-\rho) \times (f-\rho)} \\
		\end{psmallmatrix} 
		S
		,
		S^{\mathrm{T}} 
		\begin{psmallmatrix}
			A_1 & 0_{\rho \times (h-\rho)} \\
			C_1 & 0_{(f-\rho) \times (h-\rho)} \\
		\end{psmallmatrix} 
		V 
	\bigr)	
	\in \mathcal{K}
\end{equation}
say and
\begin{equation}
	\pnorm{A_2 \Sigma + \Sigma A_1}{\mathrm{F}} = 0,
	\quad
	A_1 = A_2^{\mathrm{T}}
	\quad
	\textnormal{and}
	\quad
	B_2 = C_1^{\mathrm{T}}
	,
	\label{eqn:Lemmas_assumptions_on_A1A2B2C1}
\end{equation}
then
\begin{equation}
	A_2 = A_1^{\mathrm{T}} = \Sigma X^{\prime}
	\quad
	\textnormal{and}
	\quad
	B_2 = C_1^{\mathrm{T}} = \Sigma E^{\prime}
	\label{eqn:Lemmas_consequences_on_A1A2B2C1}
\end{equation}
for some $X^{\prime} \in \mathrm{Skew}(\R^{\rho \times \rho})$ and $E^{\prime} \in \R^{\rho \times (f -\rho)}$. 
If additionally
\begin{equation}
	\mathrm{Diag}
	\Bigl(
	S^{\mathrm{T}}
	\begin{psmallmatrix}
		\Sigma (A_2 + A_1^{\mathrm{T}} ) & \Sigma (B_2+ C_1^{\mathrm{T}} ) \\
		0_{(f-\rho) \times \rho} & 0_{(f-\rho) \times (f-\rho)} \\
	\end{psmallmatrix}
	S
	\Bigr)
	= 0,
	\label{eqn:lemma_bilinear_form_is_in_tangent_space}
\end{equation}
then $K = 0$.
\label{lemma:optimization_problem_cotangent_auxiliary}
\end{lemma}

\begin{proof}
We first prove \eqref{eqn:Lemmas_consequences_on_A1A2B2C1}. It follows from \eqref{eqn:Lemmas_assumptions_on_A1A2B2C1} that if 
$\norm{A_2 \Sigma + \Sigma A_1}_{\mathrm{F}} = 0$, then $A_2 \Sigma + \Sigma A_1 = A_2 \Sigma + \Sigma A_2^{\mathrm{T}}= 0$ by property of the Frobenius norm. If now $A_2 = \Sigma X^{\prime}$ say, then $\Sigma X^{\prime} \Sigma + \Sigma (X^{\prime})^{\mathrm{T}} \Sigma =0$. Since $\Sigma$ is invertible, left and right multiplication with its inverse shows that $X^{\prime} \in \mathrm{Skew}(\R^{\rho \times \rho})$. The identity $B_2 = \Sigma B^{\prime} = C_1^{\mathrm{T}}$ follows similarly. 

We next prove that if \eqref{eqn:lemma_bilinear_form_is_in_tangent_space} holds besides \eqref{eqn:Lemmas_assumptions_on_A1A2B2C1}, then in fact $K=0$. We will do so by showing that $K \in \mathrm{T}_W M$, because then
\begin{equation}
	K 
	\in \mathcal{K} \cap \mathrm{T}_W M
	\eqcom{\ref{eqn:subspace_V}}\subseteq \mathrm{T}_W^\perp M \cap \mathrm{T}_W M 
	= \{ 0 \}.
\end{equation}

\noindent
\emph{Verification that $K \in \mathrm{T}_W M$.}
Recall that
\begin{align}
	&
	\ker \mathrm{D}_W T
	= \bigl\{  
		(V_2,V_1) \in \mathrm{T}_W \bar{M}_b 
		: 
		\mathrm{D}_W T(V_2,V_1) = 0
	\bigr\}
	\nonumber \\ &
	\eqcom{\ref{eqn:Bilinear_form_Mb_implicit_M_Of}}= \Bigl\{  
		\Bigl(
			U \Sigma_2
			\begin{psmallmatrix}
					X & E \\
					-E^{\mathrm{T}} & 0
				\end{psmallmatrix}
			S,
			S^{\mathrm{T}}
			\begin{psmallmatrix}
					X^{\mathrm{T}} & -E \\
					E^{\mathrm{T}} & 0 \\
				\end{psmallmatrix}
			\Sigma_1 V
		\Bigr)
		\in T_W \bar{M}_b 
		: 
		2 \mathrm{Diag} 
		\bigr( 
			S^{\mathrm{T}}
			\begin{psmallmatrix}
			\Sigma^2 X & \Sigma^2 E \\
			0 & 0 \\
			\end{psmallmatrix} 
			S 
		\bigr)
		= 0
	\Bigr\}.
	\label{eqn:Intermediate__Recall_ker_DWT}	
\end{align}
Thus since
\begin{equation}
	2 \mathrm{Diag} 
	\bigr( 
		S^{\mathrm{T}}
		\begin{psmallmatrix}
		\Sigma^2 X' & \Sigma^2 E' \\
		0 & 0 \\
		\end{psmallmatrix} 
		S 
	\bigr)	
	\eqcom{\ref{eqn:Lemmas_consequences_on_A1A2B2C1}}= 
	\mathrm{Diag}
	\Bigl(
	S^{\mathrm{T}}
	\begin{psmallmatrix}
		\Sigma (A_2 + A_1^{\mathrm{T}} ) & \Sigma (B_2+ C_1^{\mathrm{T}} )  \\
		0 & 0 \\
	\end{psmallmatrix}
	S
	\Bigr)
	= 0
\end{equation}
by assumption \eqref{eqn:lemma_bilinear_form_is_in_tangent_space}, clearly also
\begin{equation}
	\Bigl(	
		U \Sigma_2
		\begin{psmallmatrix}
				X' & E' \\
				-(E')^{\mathrm{T}} & 0
			\end{psmallmatrix}
		S,
		S^{\mathrm{T}}
		\begin{psmallmatrix}
				X^{\mathrm{T}} & -E' \\
				(E')^{\mathrm{T}} & 0 \\
			\end{psmallmatrix}
		\Sigma_1 V
	\Bigr)	
	\eqcom{\ref{eqn:Intermediate__Recall_ker_DWT}}\in \ker \mathrm{D}_W T.
	\label{eqn:Intermediate__Xprime_Bprime_in_kernel_DWT}	
\end{equation}
Note now lastly that
\begin{equation}
	K 
	\eqcom{\ref{eqn:Lemmas_consequences_on_A1A2B2C1}}= 	
	\bigl( 
		U 
		\begin{psmallmatrix}
		\Sigma X' & \Sigma E' \\
		0 & 0 \\
		\end{psmallmatrix} 
		S,
		S^{\mathrm{T}} \begin{psmallmatrix}
		( \Sigma X' )^{\mathrm{T}} & 0 \\
		( \Sigma E' )^{\mathrm{T}} & 0 \\
		\end{psmallmatrix} 
		V 
	\bigr). 
	\label{eqn:Intermediate__Ks_form_as_function_of_Xprime_Eprime}
\end{equation}
Utilizing \eqref{eqn:Intermediate__Xprime_Bprime_in_kernel_DWT} and \eqref{eqn:Intermediate__Ks_form_as_function_of_Xprime_Eprime} together with \eqref{eqn:tangent_M_b} of \refLemma{lemma:tangent_space_M_b}, we conclude that $K \in \mathrm{T}_{W} M_b \subseteq \mathrm{T}_{W} M$. This finishes the proof.
\end{proof}

We now define a bilinear form that will appear in the computation of the lower bound of the Hessian $\nabla^2 \mathcal{I}(W)$.

\begin{definition}
	Let $W = (U \Sigma_2 S, S^{\mathrm{T}} \Sigma_1 V) \in \bar{M}_b$ where $S \in \mathrm{O}(f)$ and let $\Sigma \in \mathrm{Diag}(\R^{\rho \times \rho})$ be defined as in \eqref{eqn:Definition_Sigma_squared}. Define the map $\bar{T}_{W}: \R^{\rho \times \rho} \times \R^{\rho \times (f -\rho) } \to \R^{f}$ by
	\begin{equation}
		\bar{T}_{W}(A,B)
		= \mathrm{Diag}
		\bigl(
		S^{\mathrm{T}}
		\begin{psmallmatrix}
			\Sigma A & \Sigma B \\
			0 & 0 \\
		\end{psmallmatrix}
		S
		\bigr),
		\label{eqn:definition_T_bar_map}
	\end{equation}
	and the bilinear form $\mathcal{T}: (\R^{\rho \times \rho} \times \R^{\rho \times (f -\rho) }) \times (\R^{\rho \times \rho} \times \R^{\rho \times (f -\rho) }) \to \R$ by
	\begin{align}
		\mathcal{T}_W \bigl( (A,B), (A^{\prime}, B^{\prime}) \bigr)
		&
		= \bigl\langle \bar{T}_W(A,B), \bar{T}_W(A^{\prime}, B^{\prime}) \bigr\rangle
		\nonumber \\
		&
		= \mathrm{Tr}
		\Bigl[
		\mathrm{Diag}
		\bigl(
			S^{\mathrm{T}}
			\begin{psmallmatrix}
				\Sigma A & \Sigma B \\
				0 & 0 \\
			\end{psmallmatrix}
			S
			\bigr)
		\mathrm{Diag}
		\bigl(
			S^{\mathrm{T}}
			\begin{psmallmatrix}
				\Sigma A^{\prime} & \Sigma B^{\prime}\\
				0 & 0
			\end{psmallmatrix} S
			\bigr)
		\Bigr].
		\label{eqn:Definition__mathcalTW_bilinear_form}
	\end{align}
	\label{def:bilinear_form}
\end{definition}
Observe that, when using notation as in \eqref{eqn:Bilinear_form_Mb_implicit_M_Of}, we have $\mathrm{D}_{W} T (V_2, V_1) = \bar{T}_{W}( \Sigma X, \Sigma B)$.

We also introduce some extra notation. For a positive definite symmetric bilinear form $A: E \times E \to \R$ on a real vector space $E$ with norm $\pnorm{\cdot}{}$, we denote $A > l$ for $l \in \R_{+}$ to indicate that $v^{\mathrm{T}}Av > l \pnorm{v}{}^2$ for all $v \in E$. We are now in position to prove a lower bound on the Hessian using $\mathrm{T}^{\perp}_W M$. 

\begin{lemma}
	Suppose Assumptions~\ref{ass:eigenvalues_different_maintext}, \ref{ass:nonvanishing_maintext} hold.
	Let $W \in M_b \cap M \backslash \mathrm{Sing}(M) \subseteq M$. We have that $\nabla^2 \mathcal{I}(W)$ restricted to $\mathrm{T}^{\perp}_W M$ is a positive definite bilinear form. Furthermore,
	\begin{equation}
		\nabla^2 \mathcal{I}(W)|_{\mathrm{T}^{\perp}_W M} 
		\geq \omega
	\end{equation}
	where
	\begin{equation}
		\omega 
		= 
		\begin{cases}
			\min 
			\{
				\zeta_{W}, 2 \frac{\lambda \kappa_{\rho}\rho }{f + \lambda \rho} - 2\sigma_{\rho + 1}
			\} 
			\quad 
			\textnormal{if} 
			\quad 
			\rho < f, 
			\\
			\min 
			\{
				\zeta_{W}, 2(\sigma_{\rho} -\sigma_{\rho + 1})
			\} 
			\quad 
			\textnormal{if} 
			\quad 
			\rho = f.
		\end{cases}				
	\end{equation}
	Here, $\zeta_{W} > 0$ is strictly positive and depends on the point $W$, $\lambda$ and $\Sigma$. If $\rho = r$ (recall from \eqref{eqn:Definition_Sigma_squared} that we have $\rho \leq r$), then we set $\sigma_{\rho + 1} = \sigma_{r + 1} = 0$.
	\label{lemma:bound_hessian}
\end{lemma}

\begin{proof}
To arrive at the result, we will give a lower bound to the solution 
\begin{align}
	\mathcal{H}^{\mathrm{opt}}_W
	=
	\begin{cases}
		\textnormal{minimum of} 
		& 
		\bigl(\mathrm{vec}(V_1), \mathrm{vec}(V_2) \bigr)^{\mathrm{T}}
		\nabla^2 \mathcal{I}(W)
		\bigl(\mathrm{vec}(V_1), \mathrm{vec}(V_2) \bigr)
		\\
		\textnormal{obtained over} 
		& 
		(V_2, V_1) \in \mathrm{T}_{W} \mathcal{P} 
		\\
		\textnormal{subject to}
		&
		\pnorm{ (V_2,V_1) }{\mathrm{F}}
		= 1,
		(V_2, V_1) \in \mathrm{T}^{\perp}_W M
	\end{cases}
	\label{eqn:optimization_hessian}
\end{align}
say, that holds for any $W \in M_b \cap M \backslash \mathrm{Sing}(M)$. We consider first the case that $\rho < f$.

\noindent
\emph{Step 1: Simplifying the objective function.}
Let $(V_2, V_1) \in T_W \mathcal{P}$, $W \in M_b \cap M \backslash \mathrm{Sing}(M)$. Since $W \in M_b$, we have by \eqref{eqn:condition_balancedness_M_b} that 
\begin{align}
	\mathrm{Tr}
	\bigl[ 
		V_1^{\mathrm{T}} \mathrm{Diag}(W_2^{\mathrm{T}} W_2) V_1 
	\bigr]
	&
	= 
	\frac{\pnorm{\Sigma^{2}}{1}}{f} \pnorm{V_1}{\mathrm{F}}^2
	\quad
	\textnormal{and similarly}
	\nonumber \\
	\mathrm{Tr}
	\bigl[ 
		V_2 \mathrm{Diag}(W_1 W_1^{\mathrm{T}}) V_2^{\mathrm{T}} 
	\bigr]
	&
	= 
	\frac{\pnorm{\Sigma^{2}}{1}}{f} \pnorm{V_2}{\mathrm{F}}^2.
	\label{eqn:Intermediate__Tr_V1TDiagW2TW2V1_and_Tr_V2DiagW1W1TV2T}
\end{align}
Substituting \eqref{eqn:Intermediate__Tr_V1TDiagW2TW2V1_and_Tr_V2DiagW1W1TV2T} into \eqref{eqn:second_derivative_dropout_loss_2layer}, we find that
\begin{align}
	&
	\bigl( 
		\mathrm{vec}(V_1), 
		\mathrm{vec}(V_2) 
	\bigr)^{\mathrm{T}}
	\nabla^2 \mathcal{I}(W)
	\bigl( 
		\mathrm{vec}(V_1), 
		\mathrm{vec}(V_2) 
	\bigr)
	\nonumber \\ &
	=
	2 \pnorm{W_2 V_1 + V_2 W_1}{\mathrm{F}}^2
	+ 2\lambda \frac{\pnorm{\Sigma^{2}}{1}}{f} 
	\bigl( 
		\pnorm{V_1}{\mathrm{F}}^2 
		+ \pnorm{V_2}{\mathrm{F}}^2 
	\bigr)
	- 4 \mathrm{Tr}
	\bigl[ 
		V_1^{\mathrm{T}} V_2^{\mathrm{T}} ( Y - \mathcal{S}_\alpha[Y] ) 
	\bigr]
	\nonumber \\ &
	\phantom{=}
	+ 2 \lambda
	\bigl(
		\pnorm{ \mathrm{Diag}(V_2^{\mathrm{T}} W_2) + \mathrm{Diag}( W_1^{\mathrm{T}} V_1 ) }{\mathrm{F}}^2
		- \pnorm{ \mathrm{Diag}(V_2^{\mathrm{T}} W_2) - \mathrm{Diag}( W_1^{\mathrm{T}} V_1 ) }{\mathrm{F}}^2
	\bigr)
	.
	\label{eqn:simplified_hessian}
\end{align}
Substituting \eqref{eqn:simplified_hessian} into \eqref{eqn:optimization_hessian} and using the facts that:
\begin{itemize}[topsep=2pt,itemsep=2pt,partopsep=2pt,parsep=2pt,leftmargin=0pt]
	\item[--] if $\pnorm{ (V_2,V_1) }{\mathrm{F}} = 1$, then $\pnorm{(V_2,V_1)}{\mathrm{F}}^2 = \pnorm{V_1}{\mathrm{F}}^2 + \pnorm{V_2}{\mathrm{F}}^2 = 1$;
	\item[--] if $(V_2,V_1) \in \mathrm{T}^{\perp}_W M$, then $\mathrm{Diag}(V_2^{\mathrm{T}} W_2) - \mathrm{Diag}(W_1 V_1^{\mathrm{T}}) = 0$ by \refLemma{lemma:tangent_cotangent_M};
	\item[--] and
		$
			\pnorm{\Sigma^2}{1}
			= \sum_{i=1}^{\rho} ( \sigma_i -  {\lambda \rho \kappa_{\rho}} / {(f + \lambda \rho)} )
			= \rho \kappa_{\rho} - \rho \frac{\lambda  \rho \kappa_{\rho}}{f + \lambda \rho}
			= {\rho \kappa_{\rho} f} / {( f + \lambda \rho )}
		$,
		which can be seen from $\Sigma^2$'s singular values shown in \eqref{eqn:Definition_Sigma_squared} and then recalling \eqref{eqn:Definition_rho_and_alpha};
\end{itemize}
we find that
\begin{align}
	\mathcal{H}^{\mathrm{opt}}_W
	=
	\begin{cases}
		\textnormal{minimum of} 
		& 
		2 \pnorm{W_2 V_1 + V_2 W_1}{\mathrm{F}}^2
		+ 2 \lambda \frac{\rho \kappa_{\rho}}{ f + \lambda \rho }
		\\ &
		- 4\mathrm{Tr}[ V_1^{\mathrm{T}} V_2^{\mathrm{T}} ( Y - \mathcal{S}_\alpha[Y] ) ]
		+ 8 \lambda \pnorm{\mathrm{Diag}(V_2^{\mathrm{T}} W_2)}{\mathrm{F}}^2	
		\\
		\textnormal{obtained over} 
		& 
		V_2, V_1 \in \mathrm{T}_{W} \mathcal{P} 
		\\
		\textnormal{subject to}
		&
		\pnorm{ V_2 }{\mathrm{F}}^{2}
		+ 
		\pnorm{ V_1 }{\mathrm{F}}^{2}
		= 1,
		(V_2, V_1) \in T^{\perp}_W M
		.
	\end{cases}
	\label{eqn:simplified_hessian2}
\end{align}

\noindent
\emph{Step 2: Change of variables.}
We now apply a change of variables to the minimization problem in \eqref{eqn:simplified_hessian2}. Specifically, we utilize the orthogonal matrices $U,S,V$ of the \gls{SVD} $W = (U \Sigma_2 S, S^{\mathrm{T}} \Sigma_1 V)$ by letting
\begin{equation}
	U \tilde{V}_2 S 
	= V_2
	\quad
	\textnormal{and}
	\quad
	S^{\mathrm{T}} \tilde{V}_1 V 
	= V_1
	\label{eqn:Intermediate__Change_of_variables_V1_V2_using_USV}
\end{equation}
say. We examine next the consequences of this change of variables to the three relevant terms in \eqref{eqn:simplified_hessian2}.

Under the change of variables in \eqref{eqn:Intermediate__Change_of_variables_V1_V2_using_USV}, the first term in \eqref{eqn:simplified_hessian2} satisfies
\begin{align}
	2 \pnorm{W_2 V_1 + V_2 W_1}{\mathrm{F}}^2
	&	
	\eqcom{SVD}= 
	2 \pnorm{ U \Sigma_2 S V_1 + V_2 S^{\mathrm{T}} \Sigma_1 V }{\mathrm{F}}^2
	\nonumber \\ &
	\eqcom{\ref{eqn:Intermediate__Change_of_variables_V1_V2_using_USV}}= 
	2 \pnorm{ U \Sigma_2 S S^{\mathrm{T}} \tilde{V}_1 V + U \tilde{V}_2 S  S^{\mathrm{T}} \Sigma_1 V }{\mathrm{F}}^2
	\eqcom{i,ii}=
	2 \pnorm{ \Sigma_2 \tilde{V}_1 + \tilde{V}_2 \Sigma_1 }{\mathrm{F}}^2
	\label{eqn:Intermediate__First_term_after_change_of_variables}
\end{align}
since (i) $S S^{\mathrm{T}} = \mathrm{Id}$ and (ii) the Frobenius norm is unitarily invariant, i.e., $\pnorm{ U(\cdot) V }{\mathrm{F}} = \pnorm{\cdot}{\mathrm{F}}$. 

Recall the definition of $\Sigma_{Y}$ in \refSection{sec:Characterization_of_the_set_of_global_minimizers}. Introducing
\begin{equation}
	\Lambda
	= \begin{psmallmatrix}
		\Sigma_Y & 0_{r \times (h - r)} \\
		0_{(e - r) \times r} & 0_{(e - r) \times (h - r)} \\
	\end{psmallmatrix},
	\label{eqn:definition_Lambda}
\end{equation}
note that
\begin{align}
	U^{\mathrm{T}} 
	\bigl( Y - \mathcal{S}_\alpha[Y] \bigr) 
	V^{\mathrm{T}} 
	&
	\eqcom{iii}= 
	U^{\mathrm{T}} 
	\bigl( 
		U 
		\begin{psmallmatrix}
			\Sigma_Y & 0 \\
			0 & 0 \\	
		\end{psmallmatrix}
		V 
		- \mathcal{S}_\alpha[Y] 
	\bigr) 
	V^{\mathrm{T}} 
	\nonumber \\ &
	\eqcom{iv}= 
	U^{\mathrm{T}} ( U \Lambda V - U \Sigma_2 \Sigma_1 V ) V^{\mathrm{T}} 
	\eqcom{v}= 
	\Lambda - \Sigma_2 \Sigma_1
	\label{eqn:Intermediate__UTYminWstarVT}
\end{align}
by 
(iii) lifting $Y$'s compact \gls{SVD} defined in \refSection{sec:Characterization_of_the_set_of_global_minimizers} to a full \gls{SVD}, 
and since 
(iv) $W \in M$ and therefore $\mathcal{S}_\alpha[Y] = W_2 W_1 = U \Sigma_2 S S_T \Sigma_1 V = U \Sigma_2 \Sigma_1 V$ by \eqref{eqn:optimal_diagonal_main}, 
and 
(v) $U^{\mathrm{T}} U = \mathrm{Id}_{e \times e}$ and $V V^{\mathrm{T}} = \mathrm{Id}_{h \times h}$. 
Conclude then that under the change of variables in \eqref{eqn:Intermediate__Change_of_variables_V1_V2_using_USV} the third term in \eqref{eqn:simplified_hessian2} satisfies
\begin{align}
	&
	- 4\mathrm{Tr}
	[ 
		V_1^{\mathrm{T}} V_2^{\mathrm{T}} ( Y - \mathcal{S}_\alpha[Y] ) 
	]
	\eqcom{\ref{eqn:Intermediate__Change_of_variables_V1_V2_using_USV}}=
	- 4\mathrm{Tr}
	[ 
		(S^{\mathrm{T}} \tilde{V}_1 V)^{\mathrm{T}} (U \tilde{V}_2 S)^{\mathrm{T}} ( Y - \mathcal{S}_\alpha[Y] ) 
	]
	\nonumber \\ &
	\eqcom{vi}= 
	- 4\mathrm{Tr}
	[ 
		\tilde{V}_1^{\mathrm{T}} \tilde{V}_2^{\mathrm{T}} U^{\mathrm{T}} ( Y - \mathcal{S}_\alpha[Y] ) V 
	]
	\eqcom{\ref{eqn:Intermediate__UTYminWstarVT}}= 
	- 4\mathrm{Tr}
	[ 
		\tilde{V}_1^{\mathrm{T}} \tilde{V}_2^{\mathrm{T}} ( \Lambda - \Sigma_2 \Sigma_1 ) 
	],	
	\label{eqn:Intermediate__Third_term_after_change_of_variables}
\end{align}
because (vi) of $S S^{\mathrm{T}} = \mathrm{Id}$ and the cyclic property of the trace.

Under the change of variables in \eqref{eqn:Intermediate__Change_of_variables_V1_V2_using_USV}, the fourth term in \eqref{eqn:simplified_hessian2} satisfies
\begin{align}
	8 \lambda \pnorm{\mathrm{Diag}(V_2^{\mathrm{T}} W_2)}{\mathrm{F}}^2
	&
	\eqcom{\ref{eqn:Intermediate__Change_of_variables_V1_V2_using_USV}}= 
	8 \lambda \pnorm{\mathrm{Diag}( (U \tilde{V}_2 S)^{\mathrm{T}} W_2 )}{\mathrm{F}}^2
	\eqcom{W's SVD}= 
	8 \lambda \pnorm{\mathrm{Diag}( (U \tilde{V}_2 S)^{\mathrm{T}} U \Sigma_2 S )}{\mathrm{F}}^2	
	\nonumber \\ &
	\eqcom{vi}= 8 \lambda \pnorm{\mathrm{Diag}( S^{\mathrm{T}} \tilde{V}_2^{\mathrm{T}} \Sigma_2 S )}{\mathrm{F}}^2,
	\label{eqn:Intermediate__Fourth_term_after_change_of_variables}
\end{align}
since (vi) $U^{\mathrm{T}} U = \mathrm{Id}_{e \times e}$.

Applying the change of coordinates in \eqref{eqn:Intermediate__Change_of_variables_V1_V2_using_USV} to \eqref{eqn:simplified_hessian2}---by substituting \eqref{eqn:Intermediate__First_term_after_change_of_variables},  \eqref{eqn:Intermediate__Third_term_after_change_of_variables}, and \eqref{eqn:Intermediate__Fourth_term_after_change_of_variables} into \eqref{eqn:simplified_hessian2}---thus yields
\begin{align}
	\mathcal{H}^{\mathrm{opt}}_W
	=
	\begin{cases}
		\textnormal{minimum of} 
		& 
		2\pnorm{\Sigma_{2} \tilde{V}_1 + \tilde{V}_2 \Sigma_{1}}{\mathrm{F}}^2
		+ 2 \frac{\rho \kappa_{\rho} \lambda}{f + \lambda \rho}
		\\ &
		- 4 \mathrm{Tr}[ \tilde{V}_1^{\mathrm{T}}\tilde{V}_2^{\mathrm{T}} (\Lambda - \Sigma_2 \Sigma_1) ]
		+ 8 \lambda \pnorm{\mathrm{Diag}( S^{\mathrm{T}} \tilde{V}_2^{\mathrm{T}} \Sigma_2 S )}{\mathrm{F}}^2		
		\\
		\textnormal{obtained over} 
		& 
		\tilde{V}_2, \tilde{V}_1
		\\
		\textnormal{subject to}
		&
		\pnorm{ \tilde{V}_2 }{\mathrm{F}}^{2}
		+ 
		\pnorm{ \tilde{V}_1 }{\mathrm{F}}^{2}
		= 1,
		(U \tilde{V}_2 S, S^{\mathrm{T}} \tilde{V}_1 V) \in T^{\perp}_W M
		.
	\end{cases}
	\label{eqn:minimization_1}
\end{align}

\noindent
\emph{Step 3: Block matrix parametrization.}
We will now write $\tilde{V}_2$ and $\tilde{V}_1$ as block matrices in a manner similar to the parametrization in \refLemma{lemma:tangent_cotangent_M}. In particular, we let
\begin{equation}
	\tilde{V}_2
	=
	\begin{psmallmatrix}
		A_2 & B_2 \\
		C_2 & D_2 \\
	\end{psmallmatrix}
	\enskip
	\textnormal{where}
	\enskip
	A_2 \in \R^{\rho \times \rho},
	B_2 \in \R^{\rho \times (f - \rho)},
	C_2 \in \R^{(e - \rho) \times \rho},
	D_2 \in \R^{(e - \rho) \times (f - \rho)},
	\label{eqn:Block_matrix_parametrization_for_V2}
\end{equation}
and
\begin{equation}
	\tilde{V}_1
	=
	\begin{psmallmatrix}
		A_1 & B_1 \\
		C_1 & D_1 \\
	\end{psmallmatrix}
	\enskip
	\textnormal{where}
	\enskip
	A_1 \in \R^{\rho \times \rho},
	B_1 \in \R^{\rho \times (h - \rho)},
	C_1 \in \R^{(f -\rho) \times \rho},
	D_1 \in \R^{(f -\rho) \times (h -\rho)}.
	\label{eqn:Block_matrix_parametrization_for_V1}
\end{equation}

We expand the first term of \eqref{eqn:minimization_1}. Utilizing $\Sigma_2, \Sigma_1$'s definitions in \eqref{eqn:Definition_Sigma_squared}, we find that
\begin{align}
	\pnorm{\Sigma_{2} \tilde{V}_1 + \tilde{V}_2 \Sigma_{1}}{\mathrm{F}}^2
    &
	= \|
	\begin{psmallmatrix}
		\Sigma & 0 \\
		0 & 0 \\
	\end{psmallmatrix}\begin{psmallmatrix}
		A_1 & B_1 \\
		C_1 & D_1 \\
	\end{psmallmatrix}  + \begin{psmallmatrix}
		A_2 & B_2 \\
		C_2 & D_2 \\
	\end{psmallmatrix} \begin{psmallmatrix}
		\Sigma & 0 \\
		0 & 0 \\
	\end{psmallmatrix}
	\|_{\mathrm{F}}^2
	=
	\|
	\begin{psmallmatrix}
		\Sigma A_1 + A_2 \Sigma & \Sigma B_1 \\
		C_2 \Sigma & 0 \\
	\end{psmallmatrix}
	\|_{\mathrm{F}}^2 
	\nonumber \\ &
	= \pnorm{ \Sigma A_1 + A_2 \Sigma }{\mathrm{F}}^2 + \pnorm{\Sigma B_1}{\mathrm{F}}^2 + \pnorm{C_2\Sigma}{\mathrm{F}}^2.
	\label{eqn:minimization_intercomp1}
\end{align}

We now tackle the third term of \eqref{eqn:minimization_1}. Recall the definitions of $\Sigma_2$, $\Sigma_1$, $\Lambda$ in \eqref{eqn:Definition_Sigma_squared}, \eqref{eqn:definition_Lambda} respectively, and let $\Sigma_{\min} \in \R^{(e - \rho) \times (h - \rho)}$ be defined such that
\begin{equation}
	\Lambda - \Sigma_2 \Sigma_1 
	= 
	\begin{psmallmatrix}
		{\lambda \kappa_{\rho} \rho} \mathrm{Id}_{\rho \times \rho} / {(f + \lambda \rho)} & 0 \\
		0 & \Sigma_{\min} \\
	\end{psmallmatrix}.
	\label{eqn:definition_sigma_min}
\end{equation}
Note that $\Sigma_{\min}$ consists of values $\sigma_{\rho + 1}, \ldots, \sigma_r$ in its upper left diagonal. Substituting \eqref{eqn:definition_sigma_min} into the third term of \eqref{eqn:minimization_1}, we find that
\begin{align}
	\mathrm{Tr}
	[ 
		\tilde{V}_1^{\mathrm{T}} \tilde{V}_2^{\mathrm{T}} (\Lambda - \Sigma_2 \Sigma_1) 
	]
	&
	= \mathrm{Tr}
	\Bigl[
		\begin{psmallmatrix}
			A_1 & B_1 \\
			C_1 & D_1 \\
		\end{psmallmatrix}^{\mathrm{T}}
		\begin{psmallmatrix}
			A_2 & B_2 \\
			C_2 & D_2 \\
		\end{psmallmatrix}^{\mathrm{T}}
		\begin{psmallmatrix}
			{\lambda \kappa_{\rho} \rho} \mathrm{Id}_{\rho \times \rho} / {(f + \lambda \rho)} & 0 \\
			0 & \Sigma_{\min} \\
		\end{psmallmatrix}
		\Bigr]
	\nonumber \\ &
	= \mathrm{Tr}
	\Bigl[
		\begin{psmallmatrix}
			A_2 A_1 + B_2 C_1 & A_2 B_1 + B_2 D_1 \\
			C_2 A_1 + D_2 C_1 & C_2 B_1 + D_2 D_1 \\
		\end{psmallmatrix}^{\mathrm{T}}
		\begin{psmallmatrix}
			{\lambda \kappa_{\rho} \rho} \mathrm{Id}_{\rho \times \rho} / {(f + \lambda \rho)} & 0 \\
			0 & \Sigma_{\min} \\
		\end{psmallmatrix}
		\Bigr]
	\nonumber \\ &
	\eqcom{i}= \frac{\lambda \kappa_{\rho} \rho}{f + \lambda \rho} \mathrm{Tr}[A_2A_1 + B_2C_1] + \mathrm{Tr}[ \Sigma_{\min}^{\mathrm{T}} (C_2B_1 + D_2D_1) ].
	\label{eqn:minimization_intercomp2}		
\end{align}
where (i) we used that $\mathrm{Tr}[ A^{\mathrm{T}}B ] = \mathrm{Tr}[ B^{\mathrm{T}}A ]$ for any pair of matrices $A, B$ of compatible dimensions.

We finally simplify the fourth term of \eqref{eqn:minimization_1}. Recall again $W$'s  \gls{SVD} $(U \Sigma_2 S, S^{\mathrm{T}} \Sigma_1 V)$; and that if $(V_2,V_1) \in \mathrm{T}^{\perp}_W M$, then $\mathrm{Diag}(V_2^{\mathrm{T}} W_2) = \mathrm{Diag}(W_1 V_1^{\mathrm{T}})$ by \refLemma{lemma:tangent_cotangent_M}. If we use the parametrization from \eqref{eqn:Block_matrix_parametrization_for_V2} and \eqref{eqn:Block_matrix_parametrization_for_V1}, this latter relation on diagonals is equivalent to equating
\begin{align}
	\mathrm{Diag}( W_2^{\mathrm{T}} V_2 )
	&
	\eqcom{W's SVD, \ref{eqn:Intermediate__Change_of_variables_V1_V2_using_USV}}= \mathrm{Diag}( S^{\mathrm{T}} \Sigma_2^{\mathrm{T}} U^{\mathrm{T}} U \tilde{V}_2 S ) 
	\eqcom{\ref{eqn:Definition_Sigma_squared},\ref{eqn:Block_matrix_parametrization_for_V2}}=\mathrm{Diag}
	\bigl(
		S^{\mathrm{T}}
		\begin{psmallmatrix}
			\Sigma^{\mathrm{T}} & 0 \\
			0 & 0 \\
		\end{psmallmatrix}
		\begin{psmallmatrix}
			A_2 & B_2 \\
			C_2 & D_2 \\
		\end{psmallmatrix}
		S
	\bigr)
	\nonumber \\ &
	= \mathrm{Diag}
	\bigl(
		S^{\mathrm{T}}
		\begin{psmallmatrix}
			\Sigma^{\mathrm{T}} A_2 & \Sigma^{\mathrm{T}} B_2 \\
			0 & 0 \\
		\end{psmallmatrix}
		S
	\bigr)
	\label{eqn:Diagonals_are_equal_2}
\end{align}
to the expression
\begin{align}
	\mathrm{Diag}( W_1 V_1^{\mathrm{T}} )
	= \mathrm{Diag}
	\bigl(
		S^{\mathrm{T}}
		\begin{psmallmatrix}
			\Sigma A_1^{\mathrm{T}} & \Sigma C_1^{\mathrm{T}} \\
			0 & 0 \\
		\end{psmallmatrix}
		S
	\bigr)
	,
	\label{eqn:Diagonals_are_equal_1}
\end{align}
the latter of hich can be shown in fashion similar to \eqref{eqn:Diagonals_are_equal_2}. Recall (i) that for any pair of matrices $A, B$, $\mathrm{Diag}(A^{\mathrm{T}}B) = \mathrm{Diag}(B^{\mathrm{T}}A)$. Therefore, 
\begin{align}
	2\mathrm{Diag}( V_2^{\mathrm{T}} W_2 )
	&
	\eqcom{i}= \mathrm{Diag}( W_1 V_1^{\mathrm{T}} ) + \mathrm{Diag}(W_2^{\mathrm{T}}V_2)
	\nonumber \\ &
	\eqcom{\ref{eqn:Diagonals_are_equal_2},\ref{eqn:Diagonals_are_equal_1}}= \mathrm{Diag}
	\bigl(
	S^{\mathrm{T}}
	\begin{psmallmatrix}
			\Sigma^{\mathrm{T}} A_2 + \Sigma A_1^{\mathrm{T}} & \Sigma^{\mathrm{T}} B_2+ \Sigma C_1^{\mathrm{T}}  \\
			0 & 0 \\
		\end{psmallmatrix}
	S
	\bigr)
	\nonumber \\ &
	\eqcom{ii}= \mathrm{Diag}
	\bigl(
	S^{\mathrm{T}}
	\begin{psmallmatrix}
			\Sigma ( A_2 + A_1^{\mathrm{T}} ) & \Sigma ( B_2+ C_1^{\mathrm{T}} )  \\
			0 & 0 \\
		\end{psmallmatrix}
	S
	\bigr)
	\label{eqn:minimization_intercom0}
\end{align}
where (ii) we have used that $\Sigma$ is a diagonal matrix.

Applying the matrix parametrization in \eqref{eqn:Block_matrix_parametrization_for_V2}, \eqref{eqn:Block_matrix_parametrization_for_V1} to \eqref{eqn:minimization_1}---by substituting \eqref{eqn:minimization_intercomp1}, \eqref{eqn:minimization_intercomp2} and \eqref{eqn:minimization_intercom0} into \eqref{eqn:minimization_1}---yields
\begin{align}
	\mathcal{H}^{\mathrm{opt}}_W
	=
	\begin{cases}
		\textnormal{minimum of}
		&
		2 \bigl(
		\pnorm{\Sigma A_1 + A_2 \Sigma}{\mathrm{F}}^2
		+ \pnorm{\Sigma B_1}{\mathrm{F}}^2
		+ \pnorm{C_2 \Sigma}{\mathrm{F}}^2
		\bigr)
		\\ &
		+ 2 \frac{ \lambda \rho \kappa_{\rho}}{f + \lambda \rho}
		\\ &
		- 4 \frac{\lambda \kappa_{\rho} \rho}{f + \lambda \rho} \mathrm{Tr}[ A_2 A_1 + B_2 C_1 ]
		\\ &
		- 4\mathrm{Tr}[ \Sigma_{\min}^{\mathrm{T}} (C_2B_1 + D_2D_1) ]
		\\ &
		+ 2 \lambda \mathrm{Tr}
		\Bigl[
			\mathrm{Diag}
			\bigl(
			S^{\mathrm{T}}
			\begin{psmallmatrix}
				\Sigma (A_2 + A_1^{\mathrm{T}} ) & \Sigma (B_2 + C_1^{\mathrm{T}}) \\
				0 & 0 \\
			\end{psmallmatrix}
			S
			\bigr)^2
		\Bigr]
		\\
		\textnormal{obtained over}
		&
		A_1, B_1, C_1, D_1; A_2, B_2, C_2, D_2
		\\
		\textnormal{subject to}
		&
		\pnorm{A_1}{\mathrm{F}}^{2}
		+ \pnorm{B_1}{\mathrm{F}}^{2}
		+ \cdots
		+ \pnorm{D_2}{\mathrm{F}}^{2}
		= 1,
		\\ &
		\bigl(
		U
		\begin{psmallmatrix}
			A_2 & B_2 \\
			C_2 & D_2 \\
		\end{psmallmatrix}
		S,
		S^{\mathrm{T}}
		\begin{psmallmatrix}
			A_1 & B_1 \\
			C_1 & D_1 \\
		\end{psmallmatrix}
		V
		\bigr)
		\in T^{\perp}_W M
		.
	\end{cases}
	\label{eqn:minimization_2}
\end{align}

\noindent
\emph{Step 5: The first bounds.}
We now start with bounding the objective function in \eqref{eqn:minimization_2}. Here, we utilize an auxiliary lemma---\refLemma{lemma:inequalities_GM}---twice. \refLemma{lemma:inequalities_GM} and its proof can be found in \refAppendixSection{secappendix:Proof_auxiliary_lemmas} .

First, we lower bound the second part of the first term in \eqref{eqn:minimization_2}. Denote the singular values of $\Sigma$ by $\chi_1, \ldots, \chi_{\rho}$; these satisfy $\chi_i > \chi_{i + 1}$ and $\chi_i^{2} = \sigma_{i} - (\lambda \kappa_{\rho} \rho) / (f + \lambda \rho)$ for $i=1, \ldots, \rho - 1$. From the fact that $\Sigma$ is an invertible, positive and diagonal matrix with minimal eigenvalue $\chi_{\rho}$, we conclude using (i) \refLemma{lemma:inequalities_GM}c that
\begin{equation}
	\pnorm{\Sigma B_1}{\mathrm{F}}^2 + \pnorm{C_2 \Sigma}{\mathrm{F}}^2
	\eqcom{i}\geq \chi_{\rho}^{2} 
	\bigl(
		\pnorm{B_1}{\mathrm{F}}^2 
		+ \pnorm{C_2}{\mathrm{F}}^2
	\bigr)
	=
	\Bigl(
		\sigma_{\rho}
		- \frac{\lambda \kappa_{\rho}\rho }{f + \lambda \rho}
	\Bigr)
	\bigl(
		\pnorm{B_1}{\mathrm{F}}^2
		+ \pnorm{C_2}{\mathrm{F}}^2
	\bigr).
	\label{eqn:minimization_intercom4}
\end{equation}

Next, we upper bound the next-to-last term in \eqref{eqn:minimization_2}. Recall that the largest singular value of $\Sigma_{\min}$ is $\sigma_{\rho+1}$; all of its singular values are in the set $\{ \sigma_{\rho+1}, \sigma_{\rho+2}, \ldots, \sigma_r, 0 \}$. Using (ii) the cyclic property of the trace, and (iii) \refLemma{lemma:inequalities_GM}b, we therefore have
\begin{align}
	\mathrm{Tr}[ \Sigma_{\min}^{\mathrm{T}} ( C_2B_1 + D_2D_1 ) ] 
	& 
	= 
	\mathrm{Tr}[ \Sigma_{\min}^{\mathrm{T}} C_2 B_1 ] 
	+ \mathrm{Tr}[ \Sigma_{\min}^{\mathrm{T}} D_2 D_1 ] 
	\nonumber \\ & 
	\eqcom{ii}= 
	\mathrm{Tr}[ C_2 B_1 \Sigma_{\min}^{\mathrm{T}} ] 
	+ \mathrm{Tr}[ D_2 D_1 \Sigma_{\min}^{\mathrm{T}} ] 	
	\nonumber \\ & 
	\eqcom{iii}\leq 
	\frac{ \sigma_{\rho+1} }{2} 
	\bigl(
		\mathrm{Tr}[ C_2 C_2^{\mathrm{T}} ]
		+ \mathrm{Tr}[ B_1 B_1^{\mathrm{T}} ]
		+ \mathrm{Tr}[ D_2 D_2^{\mathrm{T}} ]
		+ \mathrm{Tr}[ D_1 D_1^{\mathrm{T}} ]
	\bigr)
	\nonumber \\ &
	=
	\frac{ \sigma_{\rho+1} }{2} 
	\bigl(
		\pnorm{B_1}{\mathrm{F}}^2
		+ \pnorm{C_2}{\mathrm{F}}^2
		+ \pnorm{D_1}{\mathrm{F}}^2		
		+ \pnorm{D_2}{\mathrm{F}}^2 
	\bigr).
	\label{eqn:inequality_GM_minimization}
\end{align}

Using \eqref{eqn:minimization_intercom4} and \eqref{eqn:inequality_GM_minimization} to bound their respective terms in \eqref{eqn:minimization_2}, together with the constraint 
$
	\pnorm{A_1}{\mathrm{F}}^{2}
	+ \pnorm{B_1}{\mathrm{F}}^{2}
	+ \cdots
	+ \pnorm{D_2}{\mathrm{F}}^{2}
	= 1
$,
we obtain the following lower bound to \eqref{eqn:minimization_2}:
\begin{align}
	\mathcal{H}^{\mathrm{opt}}_W
	\geq
	\begin{cases}
		\textnormal{minimum of}
		&
		2 
		\pnorm{\Sigma A_1 + A_2 \Sigma}{\mathrm{F}}^2
		+ 2 
		\bigl(
			\sigma_{\rho}
			- \frac{\lambda \kappa_{\rho}\rho }{f + \lambda \rho}
		\bigr)
		\bigl(
			\pnorm{B_1}{\mathrm{F}}^2
			+ \pnorm{C_2}{\mathrm{F}}^2
		\bigr)
		\\ &
		+ 2 \frac{ \lambda \rho \kappa_{\rho}}{f + \lambda \rho}
		\bigl( 
			\pnorm{A_1}{\mathrm{F}}^{2}
			+ \pnorm{B_1}{\mathrm{F}}^{2}
			+ \cdots
			+ \pnorm{D_2}{\mathrm{F}}^{2}
		\bigr)
		\\ &
		- 4 \frac{\lambda \kappa_{\rho} \rho}{f + \lambda \rho} \mathrm{Tr}[ A_2 A_1 + B_2 C_1 ]
		\\ &
		- 2 \sigma_{\rho+1} 
		\bigl(
			\pnorm{B_1}{\mathrm{F}}^2
			+ \pnorm{C_2}{\mathrm{F}}^2
			+ \pnorm{D_1}{\mathrm{F}}^2		
			+ \pnorm{D_2}{\mathrm{F}}^2 
		\bigr)
		\\ &
		+ 2 \lambda \mathrm{Tr}
		\bigl[
			\mathrm{Diag}
			\bigl(
			S^{\mathrm{T}}
			\begin{psmallmatrix}
				\Sigma (A_2 + A_1^{\mathrm{T}} ) & \Sigma (B_2 + C_1^{\mathrm{T}}) \\
				0 & 0 \\
			\end{psmallmatrix}
			S
			\bigr)^2
			\bigr]
		\\
		\textnormal{obtained over}
		&
		A_1, B_1, C_1, D_1; A_2, B_2, C_2, D_2
		\\
		\textnormal{subject to}
		&
		\pnorm{A_1}{\mathrm{F}}^{2}
		+ \pnorm{B_1}{\mathrm{F}}^{2}
		+ \cdots
		+ \pnorm{D_2}{\mathrm{F}}^{2}
		= 1,
		\\ &
		\bigl(
		U
		\begin{psmallmatrix}
			A_2 & B_2 \\
			C_2 & D_2 \\
		\end{psmallmatrix}
		S,
		S^{\mathrm{T}}
		\begin{psmallmatrix}
			A_1 & B_1 \\
			C_1 & D_1 \\
		\end{psmallmatrix}
		V
		\bigr)
		\in T^{\perp}_W M
		.
	\end{cases}
	\label{eqn:minimization_2_intercom}
\end{align}

\noindent
\emph{Step 6: Splitting the minimization over two subspaces, with two quadratic forms.}
We examine now \eqref{eqn:minimization_2_intercom} closely. We split the objective function into the sum of
\begin{align}
	\mathcal{B}_1( B_1, C_2, D_1, D_2 ) 
	&	
	=
	2 
	(
		\sigma_{\rho}
		- \sigma_{\rho+1}
	)
	\bigl(
		\pnorm{B_1}{\mathrm{F}}^2
		+ \pnorm{C_2}{\mathrm{F}}^2
	\bigr)
	+
	2
	\Bigl(
		\frac{ \lambda \rho \kappa_{\rho}}{f + \lambda \rho}
		- \sigma_{\rho+1} 
	\Bigr)
	\bigl( 
		\pnorm{D_1}{\mathrm{F}}^{2}
		+ \pnorm{D_2}{\mathrm{F}}^{2}
	\bigr)
	\label{eqn:Definition__mathcalB1_B1C2D1D2}
\end{align}
and
\begin{align}
	&
	\mathcal{B}_2( A_1, A_2, B_2, C_1 )
	=
	2 \pnorm{\Sigma A_1 + A_2 \Sigma}{\mathrm{F}}^2
	+ 2 \frac{ \lambda \rho \kappa_{\rho}}{f + \lambda \rho}
	\Bigl( 
		\pnorm{A_1}{\mathrm{F}}^{2}
		+ \pnorm{A_2}{\mathrm{F}}^{2}		
		- 2 \mathrm{Tr}[ A_2 A_1 ]
		\nonumber \\ &
		+ \pnorm{B_2}{\mathrm{F}}^{2}		
		+ \pnorm{C_1}{\mathrm{F}}^{2}
		- 2 \mathrm{Tr}[ B_2 C_1 ]
	\Bigr)	
	+ 2 \lambda \mathrm{Tr}
	\bigl[
		\mathrm{Diag}
		\bigl(
		S^{\mathrm{T}}
		\begin{psmallmatrix}
			\Sigma (A_2 + A_1^{\mathrm{T}} ) & \Sigma (B_2 + C_1^{\mathrm{T}}) \\
			0 & 0 \\
		\end{psmallmatrix}
		S
		\bigr)^2
	\bigr].	
	\label{eqn:Definition__mathcalB2_A1A2B2C1}
\end{align} 
Also observe in \eqref{eqn:Definition__mathcalB1_B1C2D1D2} that the coefficients in front of $\pnorm{B_1}{\mathrm{F}}^2 + \pnorm{C_2}{\mathrm{F}}^2$ and $\pnorm{D_1}{\mathrm{F}}^2 + \pnorm{D_2}{\mathrm{F}}^2$ are both strictly positive; visit \refSection{sec:Characterization_of_the_set_of_global_minimizers} and recall \eqref{eqn:Definition_rho_and_alpha} specifically.

In the dimensions provided in \eqref{eqn:Block_matrix_parametrization_for_V2} and \eqref{eqn:Block_matrix_parametrization_for_V1}, let
\begin{align}
	\mathcal{V}_1
	= \mathrm{span}
	\Bigl\{ 
		&	
		\bigl(
		U
		\begin{psmallmatrix}
			0 & 0 \\
			C_2 & 0 \\
		\end{psmallmatrix}
		S,
		S^{\mathrm{T}}
		\begin{psmallmatrix}
			0 & 0 \\
			0 & 0 \\
		\end{psmallmatrix}
		V
		\bigr),
		\bigl(
		U
		\begin{psmallmatrix}
			0 & 0 \\
			0 & D_2 \\
		\end{psmallmatrix}
		S,
		S^{\mathrm{T}}
		\begin{psmallmatrix}
			0 & 0 \\
			0 & 0 \\
		\end{psmallmatrix}
		V
		\bigr),
		\label{eqn:Span_V1}
		\\ &		
		\bigl(
		U
		\begin{psmallmatrix}
			0 & 0 \\
			0 & 0 \\
		\end{psmallmatrix}
		S,
		S^{\mathrm{T}}
		\begin{psmallmatrix}
			0 & B_1 \\
			0 & 0 \\
		\end{psmallmatrix}
		V
		\bigr),
		\bigl(
		U
		\begin{psmallmatrix}
			0 & 0 \\
			0 & 0 \\
		\end{psmallmatrix}
		S,
		S^{\mathrm{T}}
		\begin{psmallmatrix}
			0 & 0 \\
			0 & D_1 \\
		\end{psmallmatrix}
		V
		\bigr)
		\nonumber \\ &
		: 
		B_1 \in \R^{\rho \times (h-\rho)},
		C_2 \in \R^{(e-\rho) \times \rho}, 
		D_1 \in \R^{(f-\rho) \times (h-\rho)},
		D_2 \in \R^{(e-\rho) \times (f-\rho)}
	\Bigr\}
	.
	\nonumber
\end{align}
Note now that $\mathcal{V}_1 \subseteq \mathrm{T}_{W}^{\perp} M$. Indeed, when we examine the definition of $\mathrm{T}_{W}^{\perp} M$ in \refLemma{lemma:tangent_cotangent_M}, we can see that nearly every condition pertains to $A_1, A_2, B_2, C_1$ only and not $B_1, C_2, D_1, D_2$---the only exception is possibly the condition $\mathrm{Diag}(V_2^{\mathrm{T}} W_2) = \mathrm{Diag}(W_1 V_1^{\mathrm{T}})$. But in fact in \eqref{eqn:Diagonals_are_equal_2} and \eqref{eqn:Diagonals_are_equal_1}, we can see that the matrices $B_1, C_2, D_1, D_2$ do not appear in this constraint. Hence, any $v \in \mathcal{V}_1$ will satisfy the conditions in the definition of $\mathrm{T}_{W}^{\perp} M$ in \refLemma{lemma:tangent_cotangent_M}. This observation yields therefore that $\mathcal{V}_1 \subseteq \mathrm{T}_{W}^{\perp} M$.

Consider now the orthogonal complement of $\mathcal{V}_1$ in $\mathrm{T}^{\perp}_{W} M$ given by 
\begin{equation}
	\mathcal{V}_2 
	= 
	\Bigl\{ 
		\bigl(
		U
		\begin{psmallmatrix}
			A_2 & B_2 \\
			C_2 & D_2 \\
		\end{psmallmatrix}
		S,
		S^{\mathrm{T}}
		\begin{psmallmatrix}
			A_1 & B_1 \\
			C_1 & D_1 \\
		\end{psmallmatrix}
		V
		\bigr)
		\in \mathrm{T}_{W}^{\perp} M
		: 
		B_1, C_2, D_1, D_2 = 0
	\Bigr\}
	\cap 
	\mathrm{T}^{\perp}_{W} M
	.
	\label{eqn:Definition_subspace_V2}
\end{equation}

From the definitions of $\mathcal{V}_1$ and $\mathcal{V}_2$ we have:
\begin{itemize}[topsep=2pt,itemsep=2pt,partopsep=2pt,parsep=2pt,leftmargin=0pt]
	\item[--] $\mathcal{V}_1 \subseteq \mathrm{T}_{W}^{\perp} M$,
	\item[--] $\mathcal{V}_1 \oplus \mathcal{V}_2 = \mathrm{T}_{W}^{\perp} M$, and
	\item[--] $\mathcal{V}_1 \perp \mathcal{V}_2$.
\end{itemize}
Using \eqref{eqn:Definition__mathcalB1_B1C2D1D2}--\eqref{eqn:Definition_subspace_V2} we can then lower bound
\begin{align}
	\mathcal{H}^{\mathrm{opt}}_W
	\geq
	\begin{cases}
		\textnormal{minimum of}
		&
		\xi
		\bigl( 
			\pnorm{B_1}{\mathrm{F}}^{2}
			+ \pnorm{C_2}{\mathrm{F}}^{2}		
			+ \pnorm{D_1}{\mathrm{F}}^{2}
			+ \pnorm{D_2}{\mathrm{F}}^{2}
		\bigr)
		+ \mathcal{B}_2( A_1, A_2, B_2, C_1 )
		\\
		\textnormal{obtained over}
		&
		( 0, 0, B_1, 0, C_2, 0, D_1, D_2 ) \in \mathcal{V}_1,
		\\ 
		&
		( A_1, A_2, 0, B_2, 0, C_1, 0, 0 ) \in \mathcal{V}_2,
		\\
		\textnormal{subject to}
		&
		\pnorm{A_1}{\mathrm{F}}^{2}
		+ \pnorm{B_1}{\mathrm{F}}^{2}
		+ \cdots
		+ \pnorm{D_2}{\mathrm{F}}^{2}
		= 1
		.
	\end{cases}
\end{align}
Here,
\begin{equation}
	\xi
	= 2 \min
	\Bigl\{
		\frac{ \lambda \rho \kappa_{\rho}}{f + \lambda \rho}
		- \sigma_{\rho+1} 
		,
		\sigma_\rho - \sigma_{\rho+1}		
	\Bigr\}
	> 0.
	\label{eqn:Definition_zeta}
\end{equation}

Now critically, note that $\mathcal{B}_1$ and $\mathcal{B}_2$ are quadratic forms, i.e., for any $\eta \in \realNumbers$, $\mathcal{B}_1(\eta \cdot) = \eta^2 \mathcal{B}_1(\cdot)$ and $\mathcal{B}_2(\eta \cdot) = \eta^2 \mathcal{B}_2(\cdot)$. We can therefore apply \refLemma{lemma:decoupling_minimization_subspace}, to find that
\begin{align}
	\mathcal{H}^{\mathrm{opt}}_W
	\geq
	\min
	\Bigl\{
		\xi
		,
		\min_{
			\substack{
				\pnorm{v_2}{\mathrm{F}}^2 = 1 \\
				v_2 \in \mathcal{V}_2
			}
		} 
		\mathcal{B}_2(v))
	\Bigr\}
	.
	\label{eqn:minimization_decoupling_lowerbounds}
\end{align}

\noindent
\emph{Step 7: Lower bounding the minimum of $\mathcal{B}_2$.}
We will now prove that
\begin{align}
	&
	\min_{
		\substack{
			\pnorm{v_2}{\mathrm{F}}^2 = 1 \\
			v_2 \in \mathcal{V}_2
		}
	} 
	\mathcal{B}_2(v)
	=	
	\min_{
		\substack{
			\pnorm{v_2}{\mathrm{F}}^2 = 1 \\
			v_2 \in \mathcal{V}_2
		}
	} 	
	2 \pnorm{\Sigma A_1 + A_2 \Sigma}{\mathrm{F}}^2
	+ 2 \frac{ \lambda \rho \kappa_{\rho}}{f + \lambda \rho}
	\Bigl( 
		\pnorm{A_1}{\mathrm{F}}^{2}
		+ \pnorm{A_2}{\mathrm{F}}^{2}		
		- 2 \mathrm{Tr}[ A_2 A_1 ]
		\nonumber \\ &
		+ \pnorm{B_2}{\mathrm{F}}^{2}		
		+ \pnorm{C_1}{\mathrm{F}}^{2}
		- 2 \mathrm{Tr}[ B_2 C_1 ]
	\Bigr)	
	+ 2 \lambda \mathrm{Tr}
	\bigl[
		\mathrm{Diag}
		\bigl(
		S^{\mathrm{T}}
		\begin{psmallmatrix}
			\Sigma (A_2 + A_1^{\mathrm{T}} ) & \Sigma (B_2 + C_1^{\mathrm{T}}) \\
			0 & 0 \\
		\end{psmallmatrix}
		S
		\bigr)^2
	\bigr]	
	\label{eqn:minimization_3}
\end{align}
has a strictly positive lower bound. Note that \eqref{eqn:minimization_3} can only equal zero if and only if at every solution $(A_1^*,A_2^*,B_2^*,C_1^*)$,
\begin{itemize}[topsep=2pt,itemsep=2pt,partopsep=2pt,parsep=2pt,leftmargin=0pt]
	\item[C1.] $2 \pnorm{\Sigma A_1^* + A_2^* \Sigma}{\mathrm{F}}^2 = 0$, and
	\item[C2.] 
	$
	2 \frac{ \lambda \rho \kappa_{\rho}}{f + \lambda \rho}
	\Bigl( 
		\pnorm{A_1^*}{\mathrm{F}}^{2}
		+ \pnorm{A_2^*}{\mathrm{F}}^{2}		
		- 2 \mathrm{Tr}[ A_2^* A_1^* ]
		+ \pnorm{B_2^*}{\mathrm{F}}^{2}		
		+ \pnorm{C_1^*}{\mathrm{F}}^{2}
		- 2 \mathrm{Tr}[ B_2^* C_1^* ]
	\Bigr)
	= 0
	$, and
	\item[C3.]
	$
	2 \lambda \mathrm{Tr}
	\bigl[
		\mathrm{Diag}
		\bigl(
		S^{\mathrm{T}}
		\begin{psmallmatrix}
			\Sigma (A_2^* + (A_1^*)^{\mathrm{T}} ) & \Sigma (B_2^* + (C_1^*)^{\mathrm{T}}) \\
			0 & 0 \\
		\end{psmallmatrix}
		S
		\bigr)^2
	\bigr]
	= 0
	$.
\end{itemize}
This is because both 
\begin{equation}
	\pnorm{A_1}{\mathrm{F}}^2 + \pnorm{A_2}{\mathrm{F}}^2 - 2\mathrm{Tr}(A_2A_1) 
	\geq 0
	\quad
	\textnormal{and} 
	\quad
	\pnorm{B_2}{\mathrm{F}}^2 + \pnorm{C_1}{\mathrm{F}}^2 - 2\mathrm{Tr}(B_2C_1) 
	\geq 0
	\label{eqn:minimization_3prime}
\end{equation}
are nonnegative: see \refLemma{lemma:inequalities_GM}a in \refAppendixSection{secappendix:Proof_auxiliary_lemmas}. We next prove that if conditions C1--C3 hold, then necessarily $(A_1^*,A_2^*,B_1^*,C_2^*) = 0$. Consequently, we must then have a positive lower bound as there are no such solutions in the optimization domain of \eqref{eqn:minimization_3}.

Condition C1 equals zero if and only if $A_2^* \Sigma + \Sigma A_1^* = 0$. This is a standard property of a norm. Condition C2 equals zero if and only if $(A_1^*)^{\mathrm{T}} = A_2^*$, $B_2^* = (C_1^*)^{\mathrm{T}}$. This is an additional consequence of \refLemma{lemma:inequalities_GM}a. Equivalent to conditions C1, C2 are therefore the statements that
\begin{equation}
	A_2^* \Sigma + \Sigma A_1^* = 0, 
	\quad
	(A_1^*)^{\mathrm{T}} = A_2^*
	\quad
	\textnormal{and} 
	\quad
	B_2^* = (C_1^*)^{\mathrm{T}}.
	\label{eqn:Equivalent_conditions_of_C1_and_C2_as_prerequisites_for_bilinear_form_argument}
\end{equation}
Condition C3 is equivalent to
\begin{align}
	&
	2 \lambda \mathrm{Tr}
	\bigl[
		\mathrm{Diag}
		\bigl(
		S^{\mathrm{T}}
		\begin{psmallmatrix}
			\Sigma ( A_2^* + (A_1^*)^{\mathrm{T}} ) & \Sigma ( B_2^* + (C_1^*)^{\mathrm{T}} ) \\
			0 & 0 \\
		\end{psmallmatrix}
		S
		\bigr)^2
	\bigr]
	\nonumber \\ &	
	= 
	2 \lambda \mathcal{T}_{W}
	\bigl( 
		A_2^* + (A_1^*)^{\mathrm{T}}, B_2^* + (C_1^*)^{\mathrm{T}}, 
		A_2^* + (A_1^*)^{\mathrm{T}}, B_2^* + (C_1^*)^{\mathrm{T}}
	\bigr)
	= 0
\end{align}
by \refDefinition{def:bilinear_form}. By $\bar{T}_{W}$'s definition in \eqref{eqn:definition_T_bar_map} and $\mathcal{T}_W$'s definition in \eqref{eqn:Definition__mathcalTW_bilinear_form}, we must then have that 
\begin{equation}
	\bar{T}_W
	\bigl( 
		A_2^* + (A_1^*)^{\mathrm{T}}, 
		B_2^* + (C_1^*)^{\mathrm{T}} 
	\bigr)
	=
	\mathrm{Diag}
	\bigl(
		S^{\mathrm{T}}
		\begin{psmallmatrix}
			\Sigma ( A_2^* + (A_1^*)^{\mathrm{T}} ) & \Sigma ( B_2^* + (C_1^*)^{\mathrm{T}} ) \\
			0 & 0 \\
		\end{psmallmatrix}
		S
	\bigr)
	= 0
	\label{eqn:Equivalent_condition_of_C3_as_prerequisite_for_bilinear_form_argument}
\end{equation} 
also. We have proven that if conditions C1--C3 are all met, then all prerequisites of \refLemma{lemma:optimization_problem_cotangent_auxiliary} are met; compare \eqref
{eqn:Equivalent_conditions_of_C1_and_C2_as_prerequisites_for_bilinear_form_argument} to \eqref{eqn:Lemmas_assumptions_on_A1A2B2C1} and \eqref{eqn:Equivalent_condition_of_C3_as_prerequisite_for_bilinear_form_argument} to \eqref{eqn:lemma_bilinear_form_is_in_tangent_space}. \refLemma{lemma:optimization_problem_cotangent_auxiliary} implies that $( A_1^*, A_2^*, B_2^*, C_1^* ) = 0$. 

We finally form the lower bound. We have proven that there is no solution in the optimization domain that satisfies conditions C1--C3 simultaneously. Consequently, 
\begin{equation}
	\zeta_W 
	=
	\min_{
		\substack{
			\pnorm{v_2}{\mathrm{F}}^2 = 1 \\
			v_2 \in \mathcal{V}_2
		}
	} 
	\mathcal{B}_2(v)
	> 0.
	\label{eqn:Positivity_of_zetaW}
\end{equation}
Substituting \eqref{eqn:Positivity_of_zetaW} into \eqref{eqn:minimization_decoupling_lowerbounds}, we obtain that
\begin{equation}
	\mathcal{H}^{\mathrm{opt}}_W
	\geq
	\min 
	\Big\{ 
		\zeta_{W}
		, 
		2 \frac{\lambda \kappa_{\rho}\rho }{f + \lambda \rho} - 2\sigma_{\rho + 1}		
		, 
		2( \sigma_{\rho} - \sigma_{\rho + 1})
	\Bigr\}.
\end{equation}
Because $\sigma_{\rho} \geq {\lambda \kappa_{\rho}\rho } / ( {f + \lambda \rho} )$, we also have that
\begin{equation}
	\mathcal{H}^{\mathrm{opt}}_W
	\geq 
	\min 
	\Big\{ 
		\zeta_{W}
		, 
		2 \frac{\lambda \kappa_{\rho}\rho }{f + \lambda \rho} - 2\sigma_{\rho + 1}
	\Bigr\}
	.
	\label{eqn:minimization_lower_bound_unknown}
\end{equation}
This concludes the case that $\rho < f$.

Now consider the case $\rho = f$. The proof is mostly the same except for the fact that in Lemmas~\ref{lemma:tangent_space_M_b}--\ref{lemma:optimization_problem_cotangent_auxiliary}, all coordinates indicated to `have dimension $f - \rho = 0$' need to be removed from the subsequent calculations. Furthermore, we then also use that $\mathcal{W}^*$ equals the rank-$f$ approximation of $S_{\alpha}[Y]$ as $f = \rho \leq r$---recall the discussion below \eqref{eqn:optimal_diagonal_main}. Concretely, the matrices $B_2, C_1, D_1, D_2$ do not appear in the calculations and ultimately this will yield functions $\mathcal{B}_1$, $\mathcal{B}_2$ different from \eqref{eqn:Definition__mathcalB1_B1C2D1D2}, \eqref{eqn:Definition__mathcalB2_A1A2B2C1}, respectively. Specifically, we find the function
\begin{align}
	\mathcal{B}_1( B_1, C_2 ) 
	&	
	=
	2 
	(
		\sigma_{\rho}
		- \sigma_{\rho+1}
	)
	\bigl(
		\pnorm{B_1}{\mathrm{F}}^2
		+ \pnorm{C_2}{\mathrm{F}}^2
	\bigr)
	\label{eqn:Definition__mathcalB1_D1D2_case_rho_equals_f}
\end{align}
and similarly the function
\begin{align}
	\mathcal{B}_2( A_1, A_2 )
	&	
	=
	2 \pnorm{\Sigma A_1 + A_2 \Sigma}{\mathrm{F}}^2
	+ 2 \frac{ \lambda \rho \kappa_{\rho}}{f + \lambda \rho}
	\Bigl( 
		\pnorm{A_1}{\mathrm{F}}^{2}
		+ \pnorm{A_2}{\mathrm{F}}^{2}		
		- 2 \mathrm{Tr}[ A_2 A_1 ]
	\Bigr)	
	\nonumber \\ &
	\phantom{=}
	+ 2 \lambda \mathrm{Tr}
	\bigl[
		\mathrm{Diag}
		\bigl(
		S^{\mathrm{T}}
			\Sigma (A_2 + A_1^{\mathrm{T}} ) 
		S
		\bigr)^2
	\bigr].	
	\label{eqn:Definition__mathcalB2_A1A2_case_rho_equals_f}
\end{align} 
Observe now that in \eqref{eqn:Definition__mathcalB1_D1D2_case_rho_equals_f} there is only one quadratic term involving $B_1$ and $C_2$, and its coefficient results in a replacement for \eqref{eqn:Definition_zeta}: 
\begin{equation}
	\xi 
	= 2 
	(
		\sigma_\rho - \sigma_{\rho+1}
	)
	.
\end{equation}
These changes carry over to \eqref{eqn:minimization_lower_bound_unknown}, and the minimum becomes
\begin{equation}
	\mathcal{H}^{\mathrm{opt}}_W
	\geq 
	\min 
	\Big\{ 
		\zeta_{W}
		, 
		2(\sigma_{\rho} - \sigma_{\rho + 1})
	\Bigr\}.
	\label{eqn:minimization_lower_bound_unknown_second}
\end{equation}
This concludes the case that $\rho = f$.

Note finally that if $\rho = r$, then $\sigma_{\rho + 1} = 0$ by \eqref{eqn:Definition_Sigma_squared}. This concludes the proof.
\end{proof}

We can improve the result in \refLemma{lemma:bound_hessian} in case $\rho = 1$. Note that whenever $e = 1$, then necessarily $\rho = 1$ also by \eqref{eqn:optimal_diagonal_main}. In this case we can explicitly calculate the minima and we do not need \refAssumption{ass:nonvanishing_maintext}. Note that this case occurs when $p$ is either sufficiently small or when we have rank one data (that is, when $r = 1$). 

\begin{lemma}
	Suppose that \refAssumption{ass:eigenvalues_different_maintext} holds and that $\rho = 1$. If $W \in M_b \backslash \mathrm{Sing}(M_b)$, then \refLemma{lemma:bound_hessian} holds with 
	\begin{equation}
		\omega 
		= 
		\begin{cases}
			\frac{ 2\lambda \sigma }{ f + \lambda }
			&
			\textnormal{if } r = 1, 
			\\ 
			\frac{ 2\lambda \sigma_{1} }{ f + \lambda } - 2 \sigma_2
			&
			\textnormal{otherwise}
		\end{cases}	
	\end{equation}
	instead.
	\label{lemma:bound_hessian_case_1}
\end{lemma}

\begin{proof}
Use \refProposition{prop:Mb_is_almost_everywhere_nonsingular} to note immediately that \refAssumption{ass:nonvanishing_maintext} is also met due to us here assuming \refAssumption{ass:eigenvalues_different_maintext} and $\rho = 1$.
	
Recall now from \eqref{eqn:Intermediate__TWMbbar_diffeomorpishm_with_of_0rho_ofmrho} and \eqref{eqn:Intermediate__ofmod0rhoofmrho} that we are able to characterize elements of $\mathrm{T}_{W} \bar{M}_b$ using pairs $(X, E)$ where $X \in \mathrm{Skew}(\R^{\rho \times \rho})$ and $E \in \R^{\rho \times (f - \rho)}$. For the particular case $\rho = 1$, we have that $\mathfrak{o}(1) = \{ 0 \}$ and $E \in \R^{f-1}$. Conclude therefore from \eqref{eqn:Bilinear_form_Mb_implicit_M_Of} that if 
\begin{equation}
	\mathrm{Diag} 
	\bigl(
		S^{\mathrm{T}} 					
		\begin{psmallmatrix}
			0 & \eta^2 E  \\
			0 & 0 \\
		\end{psmallmatrix}
		S 
	\bigr) 
	= 0,
	\label{eqn:case_1_satisfies_ass_comp}
\end{equation}
then $(0,E) \in \ker \mathrm{D}_{W} T$. Here, $\eta^2 = \Sigma^2 = f \sigma / (f + \lambda) \in \R$. Note now that in fact $\eta \neq 0$ under \refAssumption{ass:eigenvalues_different_maintext}: this allows us next to argue that in the present case $\rho = 1$ and $\eta \neq 0$, \eqref{eqn:case_1_satisfies_ass_comp} holds if and only if $E = 0$. This critical observation for the case $\rho = 1$ allows us to extend \refLemma{lemma:bound_hessian}, since we will see that the term in \eqref{eqn:case_1_satisfies_ass_comp} is proportional to $\pnorm{E}{\mathrm{F}}^2$. This allows us to explicitly compute $\zeta_{W}$.

\noindent
\emph{Proof that \eqref{eqn:case_1_satisfies_ass_comp} holds if and only if $E = 0$.}
Let $W = (U \Sigma_2 S, S^{\mathrm{T}} \Sigma_1 V) \in M_{b} \backslash \mathrm{Sing}(M_b)$---by \refLemma{lemma:assumption_holds_if_nonzero} $M_b \backslash \mathrm{Sing}(M_b) \neq \emptyset$---and refer to the rows of $S$ as $S_{1 \cdot}, \ldots, S_{f \cdot}$. By \refLemma{lemma:case_rho_is_one_M_b}, these satisfy
\begin{equation}
	|S_{1j}| 
	= \frac{1}{\sqrt{f}}
	\quad
	\textnormal{for}
	\quad
	j = 1, \ldots, f
	.
	\label{eqn:case_1_satisfies_ass_orth_condition}
\end{equation} 
We therefore have for any $s \neq 0$,
\begin{align}
	\mathrm{Tr} 
	\Bigl[ 
		\mathrm{Diag} 
		\bigl( 
			S^{\mathrm{T}} 								\begin{psmallmatrix}
				0 & s E  \\
				0 & 0 \\
			\end{psmallmatrix}
			S 
		\bigr)^2 
	\Bigr] 
	& 
	\eqcom{\ref{eqn:case_1_satisfies_ass_orth_condition}}= 
	\mathrm{Tr} 
	\Bigl[ 
		\mathrm{Diag} 
		\Bigl(
			\begin{psmallmatrix}
			0 & s E / f \\
			\cdots & \cdots \\
			0 & s E / f \\
		\end{psmallmatrix}
		(S_{\cdot 1}, \ldots, S_{\cdot f}) 
		\Bigr)^2 
	\Bigr]
	\nonumber \\ & 
	=
	\sum_{i=1}^{f} \frac{s^2}{f} 
	\bigl\langle 
		(0, E), 
		S_{\cdot i}^{\mathrm{T}} 
	\bigr\rangle^2	 
	\eqcom{i}=
	\frac{s^2}{f} \pnorm{E}{2}^2
	\label{eqn:case_1_minimization_intercom0}
\end{align}
because (i) the columns of $S$ form an orthonormal basis and we could therefore use Parseval's identity. Consequently,
\begin{equation}
	\mathrm{Diag} 
	\bigl(
		S^{\mathrm{T}} 					
		\begin{psmallmatrix}
			0 & s E  \\
			0 & 0 \\
		\end{psmallmatrix}
		S 
	\bigr) 
	= 0
\end{equation}
if and only if $E = 0$.

\noindent
\emph{Modification of step 4:}
The equality of \eqref{eqn:Diagonals_are_equal_2} and \eqref{eqn:Diagonals_are_equal_1} implies that
\begin{equation}
	\mathrm{Diag} 
	\bigl( 
		S^{\mathrm{T}} 
		\begin{psmallmatrix}
			\eta A_2 & \eta B_2\\
			0 & 0 \\
		\end{psmallmatrix}
		S 
	\bigr)
	=
	\mathrm{Diag} 
	\bigl( 
		S^{\mathrm{T}}
		\begin{psmallmatrix}
			\eta A_1 ^{\mathrm{T}} & \eta C_1 ^{\mathrm{T}} \\
			0 & 0 \\
		\end{psmallmatrix}
		S 
	\bigr)
	.	
	\label{eqn:Diagonals_equal_case_1}
\end{equation}
By taking traces in \eqref{eqn:Diagonals_equal_case_1}, we find that $\eta A_1 = \eta A_2$. Because in the present case we have $A_1, A_2 \in \R$, we can restrict to the solutions of the form $A_1 = A_2 = a$ say as $\eta \neq 0$. Thus, we can optimize over vectors of the type $v = (a, a, B_2, C_1)$ in a similar way as explained in Step 6 of \refLemma{lemma:bound_hessian}'s proof. Next, we conduct the minimization (mimicking Step 7 of \refLemma{lemma:bound_hessian}'s proof).

\noindent
\emph{Modification of step 7: Lower bounding the minimum of $\mathcal{B}_2$.}
In the present setting,
\begin{equation}
	\pnorm{A_1}{\mathrm{F}}^2
	+ \pnorm{A_2}{\mathrm{F}}^2 
	- 2\mathrm{Tr}[ A_2A_1 ] 
	= |a|^2 + |a|^2 - 2 a^2
	= 0
	.
\end{equation}
Furthermore, $\kappa_{\rho} = \sigma$ since $\rho = 1$. The optimization problem in \eqref{eqn:minimization_3} therefore reduces to
\begin{align}
	\min_{
		\substack{
			\norm{v} = 1 \\
			v \in \mathcal{V} \\
		}
	}
	\mathcal{B}_2(v) 
	=
	\min_{
		\substack{
			\norm{v} = 1 \\
			v \in \mathcal{V} \\
		}
	}
	& 	
	8 \eta^2 a^2 
	+ 2 \frac{ \lambda \sigma}{f + \lambda}
	\bigl(
		\pnorm{B_1}{\mathrm{F}}^2 
		+ \pnorm{C_1}{\mathrm{F}}^2
		- 2\mathrm{Tr}(B_2C_1)
	\bigr) 
	\nonumber \\ & 
	+ 2\lambda 
	\mathrm{Tr} 
	\Bigl[ 
		\mathrm{Diag} 
		\Bigl( 
			S^{\mathrm{T}} 								\begin{psmallmatrix}
				2\eta a & \eta (B_2 + C_1^{\mathrm{T}})  \\
				0 & 0 \\
			\end{psmallmatrix}
			S 
		\Bigr)^2
	\Bigr]
	\label{eqn:case_1_minimization1}
\end{align}
in the present setting. We next simplify \eqref{eqn:case_1_minimization1} term by term. 

Observe first that
\begin{equation}
	\pnorm{B_1}{\mathrm{F}}^2 
	+ \pnorm{C_1}{\mathrm{F}}^2 
	- 2\mathrm{Tr}[ B_2 C_1 ] 
	= \pnorm{B_2 - C_1^{\mathrm{T}}}{\mathrm{F}}^2
	\label{eqn:case_1_minimization_intercom4}
\end{equation}
by property of the Frobenius norm. 

Next, let us inspect the trace in \eqref{eqn:case_1_minimization1}. Its argument satisfies
\begin{align}
	\mathrm{Diag} 
	\Bigl( 
		S
		\begin{psmallmatrix}
			2\eta a & \eta (B_2 + C_1^{\mathrm{T}})  \\
			0 & 0 \\
		\end{psmallmatrix}
		S^{\mathrm{T}} 
	\Bigr)^2 
	& 
	= 
	\Bigl( 
		\mathrm{Diag} 
		\Bigl( 
			S^{\mathrm{T}} 								\begin{psmallmatrix}
				2\eta a & 0  \\
				0 & 0 \\
			\end{psmallmatrix}
			S 
		\Bigr) 
		+ \mathrm{Diag} 
		\Bigl( 
			S^{\mathrm{T}} 								\begin{psmallmatrix}
			0 & \eta (B_2 + C_1^{\mathrm{T}})  \\
			0 & 0 \\
			\end{psmallmatrix}
			S 
			\Bigr) 
	\Bigr)^2 
	\nonumber \\ & 
	=
	\mathrm{Diag} 
	\Bigl( 
		S^{\mathrm{T}} 								
		\begin{psmallmatrix}
			2\eta a & 0  \\
			0 & 0 \\
		\end{psmallmatrix}
		S 
	\Bigr)^{2} 
	+ \mathrm{Diag} 
	\Bigl( 
		S^{\mathrm{T}} 								
		\begin{psmallmatrix}
			0 & \eta (B_2 + C_1^{\mathrm{T}})  \\
			0 & 0 \\
		\end{psmallmatrix}
		S 
	\Bigr)^{2} 
	\nonumber \\ & 
	\phantom{=}
	+ 
	2 \mathrm{Diag} 
	\Bigl( 
		S^{\mathrm{T}}
		\begin{psmallmatrix}
			2\eta a & 0  \\
			0 & 0 \\
		\end{psmallmatrix}
		S 
	\Bigr)
	\mathrm{Diag} 
	\Bigl( 
		S^{\mathrm{T}} 
		\begin{psmallmatrix}
			0 & \eta (B_2 + C_1^{\mathrm{T}})  \\
			0 & 0 \\
		\end{psmallmatrix}
		S 
	\Bigr)
	.
	\label{eqn:case_1_minimization_intercom1}
\end{align}
We may thus split the analysis of the trace in \eqref{eqn:case_1_minimization1} by giving attention to the three terms in the right-hand side of \eqref{eqn:case_1_minimization_intercom1}. Since $W \in M_b \backslash \mathrm{Sing}(M_b)$, the trace of the first term in the right-hand side of \eqref{eqn:case_1_minimization_intercom1} satisfies 
\begin{equation}
	\mathrm{Tr}
	\Bigl[
		\mathrm{Diag} 
		\bigl( 
			S^{\mathrm{T}} 								
			\begin{psmallmatrix}
				2\eta a & 0  \\
				0 & 0 \\
			\end{psmallmatrix}
			S 
		\bigr)^2 		
	\Bigr]
	= \frac{ 4 \eta^2 a^2 }{f}
	\quad
	\textnormal{because}
	\quad
	\mathrm{Diag} 
	\bigl( 
		S^{\mathrm{T}} 								
		\begin{psmallmatrix}
			2\eta a & 0  \\
			0 & 0 \\
		\end{psmallmatrix}
		S 
	\bigr)^2 
	\eqcom{\ref{eqn:Alternative_representation_of_Mb}}= 
	\frac{(2\eta a)^2}{f^2} \mathrm{I}_{f \times f}
	.
	\label{eqn:case_1_minimization_intercom2}
\end{equation}
The trace of the second term in the right-hand side of \eqref{eqn:case_1_minimization_intercom1} satisfies 
\begin{equation}
	\mathrm{Tr}
	\Bigl[ 
		\mathrm{Diag} 
		\Bigl( 
			S^{\mathrm{T}} 								\begin{psmallmatrix}
				0 & \eta (B_2 + C_1^{\mathrm{T}})  \\
				0 & 0 \\
			\end{psmallmatrix}
			S 
		\Bigr)^{2} 
	\Bigl] 
	\eqcom{\ref{eqn:case_1_minimization_intercom0}}= 
	\frac{\eta^{2}}{f} \pnorm{B_2 + C_1^{\mathrm{T}}}{\mathrm{F}}
	.						 
	\label{eqn:case_1_minimization_intercom3}
\end{equation}
The trace of the third term in the right-hand side of \eqref{eqn:case_1_minimization_intercom1} satisfies
\begin{equation}
	\mathrm{Tr}
	\Bigl[
		2 \mathrm{Diag} 
		\Bigl( 
			S^{\mathrm{T}}
			\begin{psmallmatrix}
				2\eta a & 0  \\
				0 & 0 \\
			\end{psmallmatrix}
			S 
		\Bigr)
		\mathrm{Diag} 
		\Bigl( 
			S^{\mathrm{T}} 
			\begin{psmallmatrix}
				0 & \eta (B_2 + C_1^{\mathrm{T}})  \\
				0 & 0 \\
			\end{psmallmatrix}
			S 
		\Bigr)	
	\Bigr]
	= 
	0
	.
	\label{eqn:Traceless_product_of_two_diagonal_matrices}
\end{equation}

Substituting \eqref{eqn:case_1_minimization_intercom4}--\eqref{eqn:Traceless_product_of_two_diagonal_matrices} into \eqref{eqn:case_1_minimization1} yields
\begin{equation}
	\min_{
		\substack{
			\norm{v} = 1 \\
			v \in \mathcal{V} \\
		}
	}
	\mathcal{B}_2(v) 
	=	
	\min_{
		\substack{
			\norm{v} = 1 \\
			v \in \mathcal{V} \\
		}
	} 
	8 \eta^2 a^2 
	+ \frac{2 \lambda \sigma}{f + \lambda } \pnorm{B_2 - C_1^{\mathrm{T}}}{\mathrm{F}}^2
	+ \frac{8 \lambda \eta^2 a^2}{f}
	+ \frac{2 \lambda \eta^2}{f} \norm{B_2 + C_1^{\mathrm{T}}}_2^2	
	.
	\label{eqn:case1_intermediate_form_of_minB2}
\end{equation}
Recall now that $\eta^{2} = f/(f + \lambda)\sigma$ and observe that
\begin{equation}
	8 \eta^2 + \frac{8 \lambda \eta^2}{f} 
	= 
	\frac{8 f \sigma}{\lambda + f}
	+ \frac{8 \lambda \sigma}{f + \lambda}
	= 8 \sigma.
	\label{eqn:Small_calculation_on_coefficients}
\end{equation}
Substituting \eqref{eqn:Small_calculation_on_coefficients} into \eqref{eqn:case1_intermediate_form_of_minB2}, we find therefore that
\begin{equation}
	\min_{
		\substack{
			\norm{v} = 1 \\
			v \in \mathcal{V} \\
		}
	}
	\mathcal{B}_2(v) 	
	=
	\min_{
		\substack{
			\norm{v} = 1 \\
			v \in \mathcal{V} \\
	}
	} 
	8 \sigma a^2 
	+ \frac{2 \lambda \sigma}{f + \lambda} \pnorm{ B_2 - C_1^{\mathrm{T}} }{\mathrm{F}}^2 
	+ \frac{2 \lambda \sigma}{f + \lambda} \pnorm{ B_2 + C_1^{\mathrm{T}} }{2}^2
	.
	\label{eqn:case_1_minimization2}
\end{equation}

Finally, note that solutions of the optimization problem in \eqref{eqn:case_1_minimization2} are subject to the constraint $2a^2 + \pnorm{B_2}{2}^2 + \pnorm{C_1}{2}^2 = 1$. By identifying $s_1 = 8\sigma$, $s_2 = s_3 = 2\lambda \sigma / (f + \lambda)$ and applying \refLemma{lemma:minimization_case_rho_1_explicitbound}, see \refAppendixSection{secappendix:Proof_auxiliary_lemmas}, we find that 
\begin{equation}
	\zeta_W
	= \min_{
		\substack{
			\norm{v} = 1 \\
			v \in \mathcal{V} \\
		}
	}
	\mathcal{B}_2(v) 		
	=
	\min 
	\Bigl\{ 
		4 \sigma, 
		\frac{4 \lambda \sigma}{f + \lambda} 
	\Bigr\}.
	\label{eqn:case1_zetaW}
\end{equation}
Replacing $\zeta_{W}$ in \eqref{eqn:minimization_lower_bound_unknown} by \eqref{eqn:case1_zetaW}, we find that
\begin{equation}
	\mathcal{H}^{\mathrm{opt}}_W
	\geq 
	\min 
	\Bigl\{
		\frac{2 \sigma_1 \lambda}{f + \lambda} - 2 \sigma_2, 
		4 \sigma_1, 
		\frac{4 \sigma_1 \lambda}{f + \lambda} 
	\Bigr\} 	
	= \frac{ 2 \sigma_1 \lambda }{f + \lambda} - 2 \sigma_2.
\end{equation}
If $r = 1$, then $\sigma_2 = 0$. This completes the proof.
\end{proof}

\emph{Proof of \refProposition{prop:Lower_bound_to_the_Hessian_restricted_to_directions_normal_to_the_manifold_of_minima}:}
Finally, \refProposition{prop:Lower_bound_to_the_Hessian_restricted_to_directions_normal_to_the_manifold_of_minima} follows directly by combining Lemmas~\ref{lemma:bound_hessian}, \ref{lemma:bound_hessian_case_1}.
\qed

\subsection{Proof \texorpdfstring{of \refProposition{prop:neighborhood_exists}}{} -- Looking at a neighborhood of \texorpdfstring{$W \in M \cap U$}{W in M intersection U}}
\label{sec:Proof_of_proposition_neighborhood_exists}

A consequence of \refProposition{prop:Lower_bound_to_the_Hessian_restricted_to_directions_normal_to_the_manifold_of_minima} is that the manifold $M$ is nondegenerate at $W \in M_b \cap M \backslash \mathrm{Sing}(M)$:

\begin{corollary}
	\label{cor:non-degeneracy_MG=M}	
	Suppose Assumptions \ref{ass:eigenvalues_different_maintext}, \ref{ass:nonvanishing_maintext} hold.
	If $W \in M_b \cap M \backslash \mathrm{Sing}(M)$, then $\ker \nabla^2 \mathcal{I}(W) = \mathrm{T}_{W} M$. Furthermore, the manifold $M$ is locally nondegenerate at $W$.
\end{corollary}

\begin{proof}
	Recall that \refProposition{prop:Lower_bound_to_the_Hessian_restricted_to_directions_normal_to_the_manifold_of_minima} implies that $\nabla^2 \mathcal{I}(W)$ is a positive definite bilinear form when restricted to $\mathrm{T}^{\perp}_{W} M$, and that \refProposition{prop:Lower_bound_to_the_Hessian_restricted_to_directions_normal_to_the_manifold_of_minima} provides a lower bound $\omega$ for it. This implies in particular that $\mathrm{T}_{W}^{\perp} M \subseteq  ( \ker \nabla^{2} \mathcal{I}(W) )^{\perp}$. Because moreover $\mathrm{T}_{W} M \subseteq \ker \nabla^{2} \mathcal{I}(W)$, we find that 
	\begin{equation}
		\mathrm{T}_{W} M 
		=
		\ker \nabla^{2} \mathcal{I}(W)
		. 
		\label{eqn:Tangent_space_of_M_at_a_nonsingular_W_equals_the_kernel_of_the_Hessian}
	\end{equation}		
		
	Now, because $W$ is nonsingular by assumption and $\mathrm{Sing}(M)$ is a closed set (recall \refProposition{prop:Mb_is_almost_everywhere_nonsingular}a), there exists a neighborhood $U_W \subset \mathcal{P}$ of $W$ such that for any $W^{\prime} \in U_{W} \cap M$, $W^{\prime}$ is also nonsingular in $M$. In particular $U_{W} \cap M$ is a submanifold of $\mathcal{P}$. Hence, 
	\begin{equation}
		\dim \mathrm{T}_{W^{\prime}} M 
		= \dim \mathrm{T}_{W} M
		\label{eqn:Constant_dimension_of_the_tangent_space_in_a_neighborhood}
	\end{equation}
	is constant for all $W^{\prime} \in U_{W} \cap M$. 
	
	By continuity of $\nabla^2 \mathcal{I}$, we have furthermore that for any $W^{\prime} \in U_{W} \cap M$, 
	\begin{equation}
		\mathrm{rk}(\nabla^2 \mathcal{I}(W^{\prime})) 
		\geq \mathrm{rk}(\nabla^2 \mathcal{I}(W)) = \dim \mathrm{T}_{W} M.
		\label{eqn:Lower_bound_to_the_rank_of_the_Hessian_in_a_neighborhood_by_continuity}
	\end{equation}

	Now, (i) the \emph{rank--nullity theorem} together with \eqref{eqn:Lower_bound_to_the_rank_of_the_Hessian_in_a_neighborhood_by_continuity} implies that
	\begin{equation}
		\dim \ker \nabla^2 \mathcal{I}(W^{\prime}) 
		\eqcom{i}\leq 
		\dim \ker \nabla^2 \mathcal{I}(W) 
		\eqcom{\ref{eqn:Tangent_space_of_M_at_a_nonsingular_W_equals_the_kernel_of_the_Hessian}}= 
		\dim \mathrm{T}_{W} M 
		\eqcom{\ref{eqn:Constant_dimension_of_the_tangent_space_in_a_neighborhood}}= 
		\dim \mathrm{T}_{W^{\prime}} M
		.
		\label{eqn:Lower_bound_on_the_dimension_of_the_tangent_space_in_a_neighborhood}
	\end{equation}
	Since $\mathrm{T}_{W^{\prime}} M \subseteq \ker \nabla^2 \mathcal{I}(W^{\prime})$ also, \eqref{eqn:Lower_bound_on_the_dimension_of_the_tangent_space_in_a_neighborhood} implies that $\ker \nabla^2 \mathcal{I}(W^{\prime}) = \mathrm{T}_{W^{\prime}} M$ for any $W^{\prime} \in U_{W} \cap M$. Hence, $M$ is locally nondegenerate at $W$ according to \refDefinition{definition:non_degenerate}.
\end{proof}

\noindent
\emph{Proof of \refProposition{prop:neighborhood_exists}.}
Let $W \in M_b \cap M \backslash \mathrm{Sing}(M)$ and let $U_W$ be the open neighborhood from \refCorollary{cor:non-degeneracy_MG=M}, where $M$ is nondegenerate at $W$. 

\refProposition{prop:neighborhood_exists}a follows from the definition of nondegeneracy of \refDefinition{definition:non_degenerate} and the existence of $U_{W}$ in \refCorollary{cor:non-degeneracy_MG=M}. \refProposition{prop:neighborhood_exists}b follows because an immediate of \refProposition{prop:neighborhood_exists}a is that for any $W^{\prime} \in U_{W} \cap M$, $U_{W} \cap M$ is also locally nondegenerate at $W^{\prime}$.

\refCorollary{cor:non-degeneracy_MG=M} yields that for any $W^{\prime} \in U_{W} \cap M$, $\ker \nabla^2 \mathcal{I}(W^{\prime}) = \mathrm{T}_{W^{\prime}} M$. Hence, there exists an $\omega_{W^{\prime}}$ such that
\begin{equation}
	\min_{
	\substack{
	\norm{v} = 1 \\
	v \in T^{\perp}_{W^{\prime}} M
	}
	}
	v^{\mathrm{T}} \nabla^2 \mathcal{I}(W^{\prime}) v
	= \omega_{W^{\prime}}
	> 0.
\end{equation}
This is in fact \refProposition{prop:neighborhood_exists}c, and this completes the proof of \refProposition{prop:neighborhood_exists}.
\qed

\subsection{Proof \texorpdfstring{of \refProposition{prop:nondeg_almost_everywhere}}{} -- Extension in generic sense}
\label{sec:Proof_of_proposition_nondeg_almost_everywhere}

\refProposition{prop:Mb_is_almost_everywhere_nonsingular} implies that $M_b$ is nonsingular for generic points. Together with \refProposition{prop:neighborhood_exists}, this implies that up to a closed algebraic set with lower dimension than that of $M_b$, for every $W \in M_b \cap M$, there exists a neighborhood of $U_W \in \mathcal{P}$ of $W$ such that $U_W \cap M$ is a manifold that is locally nondegenerate at $W$. We will now extend these results to $M$ by using the group action in \eqref{eqn:definition_orthogonalgroup_action_lemma}.

The result of \refProposition{prop:exists_diagonal_reduction_Mb_to_M} implies that the action of $\pi$ extends $M_b$ to $M$ as defined in \eqref{eqn:definition_orthogonalgroup_action_lemma}. We look at the action on the Hessian. We need to prove that:
\begin{itemize}[topsep=2pt,itemsep=2pt,partopsep=2pt,parsep=2pt,leftmargin=0pt]
	\item[(a)] a given point $W^{\prime} \in M$ is nonsingular if the point in $M_b$ corresponding to $W^{\prime} \in M$ under the group action $\pi$ in \eqref{eqn:Action_pi} is also nonsingular; and 
	\item[(b)] $W^{\prime} \in M$ is nondegenerate.
\end{itemize}

\noindent
\emph{Proof of (a).}
Recall that $H = \mathrm{Diag}((\R^*)^{f})$ as a Lie group. \refProposition{prop:exists_diagonal_reduction_Mb_to_M} provides the bijective map $\pi^{-1}: M \to M_b \times H$ given by 
\begin{equation}
	\pi^{-1}(W) = (\pi(C_{W})(W), C_{W})
\end{equation}
where $C_{W} = \mathrm{Diag}(W_2^{T}W_2)^{-1/4} \mathrm{Diag}(W_1W_1^{T})^{1/4}$. From \eqref{eqn:definition_inverse_action_pi}, $\pi^{-1}$ has continuous inverse in the open set $\pi(M_b \times H)$. Recall that by \refProposition{prop:neighborhood_exists}, for each $W_b \in M_b \backslash \mathrm{Sing}(M_b)$ there is a neighborhood $Q_{W_b}$ of $W_{b}$ such that every $W^{\prime}_{b} \in Q_{W_b} \cap M_b$ is nonsingular in $M_b$. In particular, for any $C \in H$, there is a neighborhood $Q_{C} \subset H$ of $C$ such that $\pi^{-1}: Q_{W_b} \times Q_{C} \to M$ is smooth. Hence, $\pi^{-1}$ is a local diffeomorphism at $(W_b, C)$ and so $M$ is smooth at $W = \pi(C)(W_{b})$. Hence, $W$ is nonsingular in $M$ whenever $W_b \in M_b$ is nonsingular.

This implies that if $W \in M_b \cap M$ is generic, then so is $\pi(C)(W) \in M$ for any $C \in H$.

\noindent
\emph{Proof of (b).}
We start by computing the effect of the action $\pi: M \times H \to M$ on the Hessian. For any fixed $C \in H$ and $W \in M$ nonsingular, there is an induced smooth map $\mathrm{D}_{(C, W)} \pi: \mathrm{T}_{W} M \to \mathrm{T}_{\pi(C)(W)} M$ for $V =(V_2, V_1) \in \mathrm{T}_{W} M$ given by
\begin{equation}
	\mathrm{D}_{(C, W)}\pi(V) 
	= ( V_2 C , C^{-1} V_1 ) 
\end{equation}
In vectorization notation for $V$ and denoting $\mathrm{vec}(A, B) = (\mathrm{vec}(A), \mathrm{vec}(B))$ for any $A,B$, the map $\mathrm{D}_{W} \pi$ is given by
\begin{align}
	\mathrm{vec}(\mathrm{D}_{W} \pi (V)) 
	&
	=  
	\mathrm{vec}(( V_2 C , C^{-1} V_1 )) 
	= ( \mathrm{vec}(V_2 C) , \mathrm{vec}(C^{-1} V_1 ))) 
	\nonumber \\ & 
	=
	\begin{psmallmatrix} 
		C \otimes \mathrm{I}_{e \times e} & 0 \\
		0 & C^{-1} \otimes \mathrm{I}_{h \times h} \\
	\end{psmallmatrix} 
	\mathrm{vec}(V_2, V_1)  
	= 
	\mathcal{C} \mathrm{vec}(V_2, V_1)
	\label{eqn:DDWpivecV2V1}
\end{align}
say. 

We next consider the Hessian of the map $\mathcal{I}(\pi(C)(\cdot)): \mathcal{P} \to \R$ and compare it to $\nabla^2 \mathcal{I}$. We let $\nabla( g(W) )(V)$ be the differential of a function $g(W)$ depending on $W$ in the direction $V$; note that we use only Euclidean coordinates in $\mathcal{P}$ and so we can understand the differential as a gradient. First, use the chain rule to conclude that for $V \in \mathrm{T}_{W} \mathcal{P}$ we have
\begin{align}
	\nabla \bigl( \mathcal{I}(\pi(C)(W))\bigr) 
	(V) 
	&
	= 
	\nabla \mathcal{I}(\pi(C)(W)) \bigl(\mathrm{D}_{W} \pi (V)\bigr).
	\label{eqn:extension_all_M_gradient}
\end{align}
For the Hessian $\nabla^2 \bigl( \mathcal{I}( \pi(C)(\cdot) ) \bigr) : \mathrm{T}_{W}\mathcal{P} \times \mathrm{T}_{W} \mathcal{P} \to \R$, we have that similarly that for $V, R \in \mathrm{T}_{W} \mathcal{P}$ and $W \in M$,
\begin{align}
	\nabla^2 
	& 
	\Bigl( \mathcal{I}( \pi(C)(W) ) \Bigr) (V, R) 
	= 
	\nabla \Bigl( \nabla \bigl( \mathcal{I}(\pi(C)(W))\bigr)
	(V) \Bigr) (R) 
	\nonumber \\ & \eqcom{\ref{eqn:extension_all_M_gradient}}= 	 
	\nabla \Bigl( \nabla \mathcal{I}(\pi(C)(W)) \bigl(\mathrm{D}_{W} \pi (V)\bigr) \Bigr) (R) 
	\nonumber \\ & 
	\eqcom{i}= 
	\nabla \Bigl( \nabla \mathcal{I}(\pi(C)(W)) \Bigr) (R) \bigl( \mathrm{D}_{W} \pi (V) \bigr) 
	+ \nabla \mathcal{I}(\pi(C)(W)) \Bigl( \nabla(\mathrm{D}_{W} \pi (V))(R) \Bigr) 
	\nonumber \\ & 
	\eqcom{ii}= 
	\nabla^{2}\mathcal{I}(\pi(C)(W)) \bigl(\mathrm{D}_{W} \pi (V), \mathrm{D}_{W} \pi (R) \bigr) 
	+ \nabla \mathcal{I}(\pi(C)(W)) \bigl( \nabla(\mathrm{D}_{W} \pi (V))(R) \bigr)
	\label{eqn:extension_all_M_computation}
\end{align}
where we have (i) used Leibniz's rule in the one-to-last step and (ii) the chain rule. Since $\nabla \mathcal{I}(\pi(C)(W)) = 0$ at any minimizer $\pi(C)(W)$, we have that \eqref{eqn:extension_all_M_computation} reduces to
\begin{equation}
	\nabla^2 \bigl( \mathcal{I}( \pi(C)(W) ) \bigr)(V, R) 
	= 
	\nabla^2 \mathcal{I}(\pi(C)(W))(\mathrm{D}_{W} \pi (V), \mathrm{D}_{W} \pi (R)).
	\label{eqn:equivalence_Hessian_comp_pre}
\end{equation}
We abuse now the vectorization notation from \eqref{eqn:DDWpivecV2V1} and consider the Hessian as a bilinear form in terms of $\mathrm{vec}(V)$ and $\mathrm{vec}(R)$ in \eqref{eqn:equivalence_Hessian_comp_pre}. This means specifically that \eqref{eqn:equivalence_Hessian_comp_pre} can be written as
\begin{equation}
	\nabla^2 \bigl( \mathcal{I}( \pi(C)(W) ) \bigr) 
	= 
	\mathcal{C}^{\mathrm{T}} \bigl( \nabla^2 \mathcal{I}(\pi(C)(W)) \bigr) \mathcal{C}.
	\label{eqn:equivalence_Hessian_comp}
\end{equation}
Recall now that for any $C \in H$ and $W \in \mathcal{P}$, $\mathcal{I}( \pi(C)(W) ) = \mathcal{I}(W)$. Consequently, as bilinear forms
\begin{equation}
	\nabla^{2} \mathcal{I}(W) 
	=
	\mathcal{C}^{\mathrm{T}} \nabla^2 \mathcal{I}(\pi(C)(W)) \mathcal{C}.
	\label{eqn:equivalence_Hessians_under_action_pi}
\end{equation}

Finally, note that $\mathcal{C}$ in \eqref{eqn:equivalence_Hessians_under_action_pi} is invertible for any $C \in H$. Therefore, we have that the ranks are equal:
\begin{equation}
	\mathrm{rk}(\nabla^2 \mathcal{I}(W))
	=
	\mathrm{rk}\bigl( \nabla^2 \mathcal{I}( \pi(C)(W) ) \bigr) 	
\end{equation}
for any $C \in H$. Now, if $W \in M_b \cap M$ is nondegenerate (so the rank is maximal), we can repeat the arguments of \refProposition{prop:neighborhood_exists} and in particular conclude that $\pi(C)(W)$ is a nondegenerate point in $M$.

Combining (a) and (b) implies that $M$ is nondegenerate at generic points.
\qed

\section{Auxiliary statements}
\label{secappendix:Proof_auxiliary_lemmas}

\subsection{Inequalities pertaining to the Frobenius norm}

\begin{lemma}
	The following inequalities hold:
	\begin{itemize}[topsep=2pt,itemsep=2pt,partopsep=2pt,parsep=2pt,leftmargin=0pt]
		\item[(a)] For any $C \in \R^{a \times b}$, $D \in \R^{b \times a}$, it holds that
		\begin{equation}
			2\mathrm{Tr}[ C D ] 
			\leq \pnorm{C}{\mathrm{F}}^2 + \pnorm{D}{\mathrm{F}}^2.
			\label{eqn:Auxiliary_lemma__Trace_bound}
		\end{equation}
		with equality if and only if $C = D^{\mathrm{T}}$.
		\item[(b)] For any $A \in \R^{h \times f}$, $B \in \R^{e \times f}$, $\Lambda \in \R^{e \times h}$, it holds that
		\begin{equation}
			\mathrm{Tr}[ A^{\mathrm{T}} B^{\mathrm{T}} \Lambda ] 
			\leq 
			\frac{\sigma_1(\Lambda)}{2} 
			\bigl( 
				\mathrm{Tr}[ B^{\mathrm{T}}B ] 
				+ \mathrm{Tr}[ A A^{\mathrm{T}} ]
			\bigr).
		\end{equation}			
		\item[(c)] For any $B\in \R^{e \times f}$, and diagonal matrix $\Lambda \in \R^{e \times e}$ with positive entries and minimal eigenvalue $s = \min_{i=1, \ldots, e} \Lambda_{ii}$, it holds that
		\begin{equation}
			\pnorm{B^{\mathrm{T}} \Lambda}{\mathrm{F}}^{2} \geq 
			s^{2} \pnorm{B}{\mathrm{F}}^{2}.
		\end{equation}
	\end{itemize}
	\label{lemma:inequalities_GM}
\end{lemma}

We prove the inequalities in \refLemma{lemma:inequalities_GM} one by one.

\noindent
\emph{Proof of (a).}
Recall that the Frobenius norm satisfies $\pnorm{A+B}{\mathrm{F}}^2 = \pnorm{A}{\mathrm{F}}^2 + \pnorm{B}{\mathrm{F}}^2 - 2 \langle A,B \rangle_{\mathrm{F}}$ where $\langle A,B \rangle_{\mathrm{F}} = \mathrm{Tr}[ A^T B ]$ (for real matrices) denotes the \emph{Frobenius inner product}. We have in particular that
\begin{equation}
	0 
	\leq
	\pnorm{C - D^{\mathrm{T}}}{\mathrm{F}}^2 
	= \pnorm{C}{\mathrm{F}}^2 + \pnorm{D}{\mathrm{F}}^2 - 2\mathrm{Tr}[CD]
\end{equation}
with equality if and only if $C = D^{\mathrm{T}}$ (by property of a norm).

\noindent
\emph{Proof of (b).}
Consider any square matrix $R \in \realNumbers^{f \times f}$ and let $\sigma_{\max}(R)$ be its spectral norm, i.e., its largest singular value. Recall that
\begin{align}
	\sigma_{\max}(R)
	= 
	\sup 
	\bigl\{ 
		\pnorm{Rx}{2} : x \in \realNumbers^{f}, \pnorm{x}{2} = 1
	\bigr\}
	=
	\sup 
	\bigl\{ 
		\frac{ x^T R^T R x }{ x^T x } : x \in \realNumbers^{f}, x \neq 0
	\bigr\}^{1/2}
	.
\end{align}
Note that 
\begin{align}
	\quad
	\sigma_{\max}
	\begin{psmallmatrix}
		0 & R \\
		R^{\mathrm{T}} & 0 \\
	\end{psmallmatrix} 
	^2
	&
	= \sup 
	\Bigl\{
		\|
		\begin{psmallmatrix}
			0 & R \\
			R^{\mathrm{T}} & 0 \\
		\end{psmallmatrix} 
		\begin{psmallmatrix}
			a \\
			b \\
		\end{psmallmatrix} 
		\|_2^2
		: 
		a, b \in \realNumbers^f, 
		\pnorm{a}{2}^2 + \pnorm{b}{2}^2 = 1
	\Bigr\}
	\nonumber \\ &
	= \sup 
	\Bigl\{
		\|
		\begin{psmallmatrix}
			R b  \\
			R^{\mathrm{T}}a \\
		\end{psmallmatrix} 
		\|_2^2
		: 
		a, b \in \realNumbers^f, 
		\pnorm{a}{2}^2 + \pnorm{b}{2}^2 = 1
	\Bigr\}
	\nonumber \\ &	
	= \sup 
	\Bigl\{
		\pnorm{R^T a}{2}^2 
		+ \pnorm{R b}{2}^2
		: 
		a, b \in \realNumbers^f, 
		\pnorm{a}{2}^2 + \pnorm{b}{2}^2 = 1
	\Bigr\}
	\nonumber \\ &		
	\leq  \sigma_{\max}(R)^2,
	\label{eqn:intermediate_step_lemma}
\end{align}
where the inequality follows because $\sigma_{\max}(R) = \sigma_{\max}(R^T)$ and by definition of $\sigma_{\max}(R)$,  $\pnorm{R x}{2}^2 \leq \sigma_{\max}(R)^{2} \pnorm{x}{2}^2  $ for any $x \in \R^{f}$. 
	
Recall the properties of the vectorization notation in \refAppendixSection{sec:Proof_of_the_bilinear_form_of_the_Hessian}. 
We have then in vectorization notation
\begin{align}
	\mathrm{Tr}[ A^{\mathrm{T}} B^{\mathrm{T}} \Lambda ]
	&
	= 
	\frac{1}{2}
	\bigl(
		\mathrm{Tr}[ A^{\mathrm{T}} (B^{\mathrm{T}} \Lambda) ] 
		+ \mathrm{Tr}[ B(A \Lambda^{T}) ] 
	\bigr)	
	\nonumber \\	
	&
	= 
	\frac{1}{2}
	\bigl(
		\mathrm{vec}(A)^{T} \Lambda^{T} \otimes \mathrm{I}_{f \times f} \mathrm{vec}(B^{T}) 
		+ \mathrm{vec}(B^{T}) \Lambda \otimes \mathrm{I}_{f \times f} \mathrm{vec}(A)
	\bigr)
	\nonumber \\
	&
	=
	\bigl( \mathrm{vec}(A), \mathrm{vec}(B^{\mathrm{T}}) \bigr)^{\mathrm{T}}
	\tfrac{1}{2}
	\begin{psmallmatrix}
		0 & \Lambda^{\mathrm{T}} \otimes \mathrm{I}_{f \times f} \\
		\Lambda \otimes \mathrm{I}_{f \times f} & 0 \\
	\end{psmallmatrix}
	\bigl( \mathrm{vec}(A), \mathrm{vec}(B^{\mathrm{T}}) \bigr)
	\nonumber \\ &
	\leq
	\frac{\sigma_{\max}(\Lambda)}{2}
	\bigl(
		\mathrm{Tr}[ A A^{\mathrm{T}} ]
		+ \mathrm{Tr}[ B^{\mathrm{T}} B ]
	\bigr)
\end{align}
where we have used that $\sigma_{\max}(Y \otimes I) = \sigma_{\max}(Y)$.

\noindent
\emph{Proof of (c).}
Suppose without loss of generality that the diagonal elements of $\Lambda$ are ordered, i.e., $\Lambda_{11} \geq \ldots \geq \Lambda_{ee} > 0$. Denote the columns of $B$ by $B_{\cdot 1}, \ldots, B_{\cdot e}$. Calculating the Frobenius norm directly, we find that 
\begin{equation}
	\pnorm{B^{\mathrm{T}}\Lambda}{\mathrm{F}}^{2} 
	= \sum_{j=1}^e \Lambda_j^2 \pnorm{B_{\cdot j}}{2}^2
	\geq \Lambda_{e}^2 \sum_{j=1}^e \pnorm{B_{\cdot j}}{2}^2 
	= \Lambda_{e}^2 \pnorm{B}{\mathrm{F}}^2.
\end{equation}
This completes the proof of \refLemma{lemma:inequalities_GM}.
\qed

\subsection{Subspace minimization}

\begin{lemma}
	Let $\mathcal{V}_1$, $\mathcal{V}_2$ be two orthogonal subspaces, and let $\{ v_1, \ldots, v_d \}$ be an orthonormal basis of $\mathcal{V}_1 \oplus \mathcal{V}_2$ such that $\mathcal{V}_1 = \mathrm{span} \{ v_1, \ldots, v_s \}$ and $\mathcal{V}_2 = \mathrm{span} \{ v_{s + 1}, \ldots, v_{d} \}$. Assume that $l_1, \ldots, l_s \in (0,\infty)$, and let $\mathcal{B}_2 : \mathcal{V}_2 \to [0,\infty)$ be a function that satisfies $\mathcal{B}_2( \zeta u_2 ) = \zeta^2 \mathcal{B}_2(u_2)$ for $\zeta \in \realNumbers$. Then,
	\begin{equation}
		\min_{
			\substack{ 
				\pnorm{u_1}{\mathrm{F}}^2 + \pnorm{u_2}{\mathrm{F}}^2 = 1 \\
				u_1 \in \mathcal{V}_1, 
				u_2 \in \mathcal{V}_2
			}
		} 
		\sum_{i=1}^s l_i | \langle v_i, u_1 \rangle |^{2} 
		+ \mathcal{B}_2(u_2)
		\geq 
		\min
		\bigl\{ 
			l_1, \ldots, l_s, 
			\min_{
				\substack{ 
					\pnorm{u_2}{\mathrm{F}}^2 = 1 \\
					u_2 \in \mathcal{V}_1}
			}
			\mathcal{B}_2(u_2)
		\bigr\}.
		\label{eqn:auxiliary_minimization_lower_bound}
	\end{equation}
	\label{lemma:decoupling_minimization_subspace}
\end{lemma}

\begin{proof}
	Note that
	\begin{align}
		&
		\{ 
			(u_1, u_2) \in \mathcal{V}_1 \times \mathcal{V}_2
			: 
			\pnorm{u_1}{\mathrm{F}}^2 
			+ \pnorm{u_2}{\mathrm{F}}^2 
			= 1
		\}
		\nonumber \\ &
		=
		\cup_{\zeta \in [0,1]}
		\{ 
			(u_1, u_2) \in \mathcal{V}_1 \times \mathcal{V}_2
			: 
			\pnorm{u_1}{\mathrm{F}}^2 
			= \zeta^2,
			\pnorm{u_2}{\mathrm{F}}^2 
			= 1 - \zeta^2
		\}
		\nonumber \\ &
		=
		\cup_{\zeta \in [0,1]}
		\{ 
			(\zeta w_1, \sqrt{1-\zeta^2} w_2)
			: 
			(w_1,w_2) \in \mathcal{V}_1 \times \mathcal{V}_2,
			\pnorm{w_1}{\mathrm{F}}^2 
			= 1,
			\pnorm{w_2}{\mathrm{F}}^2 
			= 1
		\}	
		.
	\end{align}	
	The left-hand side of \eqref{eqn:auxiliary_minimization_lower_bound} therefore equals
	\begin{equation}
		\min_{
			\substack{ 
				\pnorm{w_1}{\mathrm{F}}^2 = 1, 
				\pnorm{w_2}{\mathrm{F}}^2 = 1,
				\\
				w_1 \in \mathcal{V}_1, 
				w_2 \in \mathcal{V}_2,
				\zeta \in [0,1]
			}
		} 
		\zeta^2 \sum_{i=1}^s l_i | \langle v_i, w_1 \rangle |^2 
		+ (1-\zeta^2) \mathcal{B}_2(w_2).
		\label{eqn:auxiliary_intercom}		
	\end{equation}
	Observe in \eqref{eqn:auxiliary_intercom} a convex combination in terms of $\zeta^{2}$. The minimum of \eqref{eqn:auxiliary_intercom} therefore occurs at either $\zeta=0$ or $\zeta=1$. Note additionally that if $\pnorm{w_1}{\mathrm{F}} = 1$, then
	\begin{equation}
		\sum_{i=1}^s l_i | \langle v_i, w_1 \rangle |^2 
		\eqcom{i}\geq \min \{ l_1, \ldots, l_s \} \sum_{i=1}^s | \langle v_i, w_1 \rangle |^2 
		\eqcom{ii}= \min \{ l_1, \ldots, l_s \}
	\end{equation}	
	by (i) strict positivity of the summands and (ii) an application of Parseval's identity---which is warranted since $\{ v_1, \ldots, v_s \}$ is an orthonormal basis of $\mathcal{V}_1$. Together, this proves the lower bound on the right-hand side in \eqref{eqn:auxiliary_minimization_lower_bound}.
\end{proof}

\subsection{Minimization}

\begin{lemma}
	For $a \in \R$, $B, C \in \R^{f-1}$ and $s_1, s_2 ,s_3 > 0$,
	\begin{equation}
		\min_{
			\substack{
				a,B,C \\
				2a^2 + \pnorm{B}{2}^2 + \pnorm{C}{\mathrm{F}}^2 = 1
			}
		}
		\Bigl\{
			s_1 a^2
			+ s_2 \pnorm{B - C}{\mathrm{F}}^2
			+ s_3 \norm{B + C}_2^2
		\Bigr\}
		= 
		\min 
		\Bigl\{
			\frac{s_1}{2}, 2s_2, 2s_3 
		\Bigr\}.
		\label{eqn:case_1_auxiliary_minimization}
	\end{equation}
	\label{lemma:minimization_case_rho_1_explicitbound}
\end{lemma}

\begin{proof}
We can decouple the minimization over $a$ and over $(B,C)$ in \eqref{eqn:case_1_auxiliary_minimization}, respectively. To see this, suppose that $(a_0, B_0, C_0)$ is a minimizer of \eqref{eqn:case_1_auxiliary_minimization}. If so, then $(B_0,C_0)$ must also be a minimizer of
\begin{equation}
	\min_{
		\substack{
			B, C \\
			\pnorm{B}{2}^2 + \pnorm{C}{2}^2 = 1 - 2a_0^2
		}
	} 
	s_2 \pnorm{B - C}{\mathrm{F}}^2 
	+ s_3 \pnorm{B + C}{\mathrm{F}}^2
	\label{eqn:Intermediate_decoupled_optimization_problem}
\end{equation} 
for otherwise $(a_0, B_0, C_0)$ would not be a minimizer of \eqref{eqn:case_1_auxiliary_minimization} by linearity. 

For fixed $a_0$, the following holds:
\begin{itemize}[topsep=2pt,itemsep=2pt,partopsep=2pt,parsep=2pt,leftmargin=0pt]
	\item[--] if $s_2 > s_3$, then the minimizer $(B_0,C_0)$ of \eqref{eqn:Intermediate_decoupled_optimization_problem} satisfies $B_0 = C_0$ and the minimum is $4 s_3 \pnorm{B_0}{\mathrm{F}}^2 = 2 s_3 (1 - 2 a_0^2)$;
	\item[--] if $s_2 < s_3$, then the minimizer $(B_0,C_0)$ of \eqref{eqn:Intermediate_decoupled_optimization_problem} satisfies $B_0 = - C_0$ and the minimum is $4 s_2 \pnorm{B_0}{\mathrm{F}}^2 = 2 s_2 (1 - 2 a_0^2)$;
	\item[--] if $s_2 = s_3$, then any point $(B_0,C_0)$ that satisfies $\pnorm{B_0}{2}^2 + \pnorm{C_0}{2}^2 = 1 - 2a_0^2$ is a minimizer of \eqref{eqn:Intermediate_decoupled_optimization_problem} by the parallelogram law, and the minimum is in fact $2 s_2 (1 - 2a_0^2)$.
\end{itemize}

Thus, we have that the left-hand side of \eqref{eqn:case_1_auxiliary_minimization} reduces to:
\begin{itemize}[topsep=2pt,itemsep=2pt,partopsep=2pt,parsep=2pt,leftmargin=0pt]
	\item[--] if $s_2 > s_3$, then $\min_{a \in [-1/\sqrt{2},1/\sqrt{2}]} \{ s_1 a^2 + 2 s_3(1-2a^2) \} = \min \{ s_1 / 2, 2s_3 \}$;
	\item[--] if $s_2 < s_3$, then $\min_{a \in [-1/\sqrt{2},1/\sqrt{2}] } \{ s_1 a^2 + 2 s_2(1-2a^2) \} = \min \{ s_1 / 2, 2s_2 \}$;
	\item[--] if $s_2 = s_3$, then $\min_{a \in [-1/\sqrt{2},1/\sqrt{2}]} \{ s_1 a^2 + 2 s_2 (1-2a^2) \} = \min \{ s_1 / 2, 2s_2 \}$.
\end{itemize}
Combining cases, observe that the left-hand side of \eqref{eqn:case_1_auxiliary_minimization} equals $\min\{s_1/2, 2s_2, 2s_3 \}$.
\end{proof}

\end{document}